\documentclass[12pt]{article}

\usepackage{header/PRIMEarxiv}
\usepackage{natbib}
\usepackage{amsthm}
\usepackage{amsfonts,amsmath,amssymb}
\usepackage{physics,mathtools,bm,mathrsfs,mleftright}
\usepackage{xcolor,enumerate,autobreak}
\allowdisplaybreaks[4]
\usepackage{wrapfig,subfigure,multirow}

\usepackage[hidelinks]{hyperref}
\usepackage{prettyref}

\usepackage{tikz}

\theoremstyle{plain}
\newtheorem{theorem}{Theorem}[section]
\newtheorem{lemma}[theorem]{Lemma}
\newtheorem{corollary}[theorem]{Corollary}

\theoremstyle{definition}
\newtheorem{proposition}[theorem]{Proposition}
\newtheorem{definition}[theorem]{Definition}
\newtheorem{remark}[theorem]{Remark}

\newtheorem{assumption}{Assumption}

\newrefformat{eq}{\textup{(\ref{#1})}}
\newrefformat{lem}{Lemma~\textup{\ref{#1}}}
\newrefformat{thm}{Theorem~\textup{\ref{#1}}}
\newrefformat{cor}{Corollary~\textup{\ref{#1}}}
\newrefformat{prop}{Proposition~\textup{\ref{#1}}}
\newrefformat{assu}{Assumption~\textup{\ref{#1}}}
\newrefformat{remark}{Remark~\textup{\ref{#1}}}
\newrefformat{example}{Example~\textup{\ref{#1}}}
\newrefformat{def}{Definition~\textup{\ref{#1}}}
\newrefformat{sec}{Section~\textup{\ref{#1}}}
\newrefformat{subsec}{Subsection~\textup{\ref{#1}}}
\newrefformat{subsubsec}{Subsection~\textup{\ref{#1}}}

\newcommand{\ep}{\varepsilon}
\newcommand{\R}{\mathbb{R}}

\newcommand{\E}{\mathbb{E}}
\newcommand{\Var}{\ensuremath{\mr{Var}}}
\newcommand{\Cov}{\ensuremath{\mr{Cov}}}

\newcommand{\T}{{\top}}

\newcommand{\caA}{\mathcal{A}}

\newcommand{\caE}{\mathcal{E}}

\newcommand{\caG}{\mathcal{G}}
\newcommand{\caH}{\mathcal{H}}
\newcommand{\caI}{\mathcal{I}}

\newcommand{\caL}{\mathcal{L}}

\newcommand{\caN}{\mathcal{N}}

\newcommand{\caR}{\mathcal{R}}

\newcommand{\caX}{\mathcal{X}}

\newcommand{\bbS}{\mathbb{S}}
\newcommand{\bbZ}{\mathbb{Z}}
\newcommand{\bbN}{\mathbb{N}}
\newcommand{\bbP}{\mathbb{P}}

\newcommand{\hf}{\frac{1}{2}}

\DeclareMathOperator*{\argmin}{arg\,min}

\DeclareMathOperator{\spn}{span}

\newcommand{\qimplies}{\ensuremath{\quad\Longrightarrow\quad}}

\newcommand{\mr}{\mathrm}
\providecommand{\ang}[1]{\ensuremath{\left\langle{#1}\right\rangle}}
\newcommand{\xk}[1]{\ensuremath{\left(#1\right)}}
\newcommand{\zk}[1]{\ensuremath{\left[#1\right]}}
\newcommand{\dk}[1]{{\ensuremath{\left\{#1\right\}}}}

\newcommand{\cek}[1]{\ensuremath{\lceil#1\rceil}}

\newcommand{\rz}{{\ensuremath{r_0}}}

\newcommand{\mm}{{\ensuremath{\bm{m}}}} 
\newcommand{\nn}{{\ensuremath{\bm{n}}}}
\newcommand{\rr}{{\ensuremath{\bm{r}}}}

\newcommand{\St}{{\ensuremath{\mathcal{S}}}} 
\newcommand{\proj}{{\ensuremath{\Pi}}} 
\DeclareMathOperator{\Sym}{Sym} 

\DeclareMathOperator{\Diag}{Diag} 

\newcommand{\Rho}{\varPsi}
\newcommand{\polylog}{\mr{polylog}}

\providecommand{\cref}{\prettyref}

\newcommand{\theSupp}{the appendix}
\newcommand{\suppref}[1]{\cref{#1} in \theSupp}

\title{Supporting Evidence for the Adaptive Feature Program across Diverse Models}

\author{
{\bfseries Yicheng Li} \\
{\normalfont Department of Statistics and Data Science} \\
{\normalfont Tsinghua University} \\
{\normalfont Beijing, 100084, China} \\
{\normalfont \texttt{liyc22@mails.tsinghua.edu.cn}}
\and
{\bfseries Qian Lin}\thanks{Corresponding Author.} \\
{\normalfont Department of Statistics and Data Science} \\
{\normalfont Tsinghua University} \\
{\normalfont Beijing, 100084, China} \\
{\normalfont \texttt{qianlin@tsinghua.edu.cn}}
}

\date{}

\begin{document}

\maketitle

\begin{abstract}
Theoretically exploring the advantages of neural networks might be one of the most challenging problems in the AI era.
An adaptive feature program has recently been proposed to analyze feature learning, the characteristic property of neural networks, in a more abstract way.
 Motivated by the celebrated Le Cam equivalence, we advocate the over-parameterized sequence models to further simplify the analysis of the training dynamics of adaptive feature program and present several pieces of supporting evidence for the adaptive feature program.
 More precisely, after having introduced the feature error measure (FEM) to characterize the quality of the learned feature, we show that the FEM is decreasing during the training process of several concrete adaptive feature models including linear regression, single/multiple index models, etc.
 We believe that this hints at the potential successes of the adaptive feature program.


\end{abstract}



\section{Introduction}

The remarkable empirical success of neural networks has transformed modern data analysis, achieving unprecedented performance across diverse domains such as computer vision, natural language processing, and reinforcement learning.
These models consistently generalize well beyond their training data, even in complex, high-dimensional settings, often surpassing traditional statistical techniques.
Despite this practical success, the theoretical understanding of their generalization capabilities remains elusive, posing a significant challenge for researchers~\citep{zhang2017_UnderstandingDeep}.

A pivotal insight into this success lies in \emph{feature learning}, the process by which neural networks dynamically adapt their internal representations to uncover task-relevant patterns \citep{woodworth2020_KernelRich, ba2022_HighdimensionalAsymptotics, yang2022_FeatureLearning,zhang2024_StatisticalUnderstanding}.
Unlike classical non-parametric regression methods, such as kernel regression or spline smoothing~\citep{steinwart2008_SupportVector, wendland2004_ScatteredData}, which rely on static, predefined feature maps, neural networks exhibit a dynamic adaptability that defies traditional analysis.

However, due to the complex nature of neural networks, the theoretical frameworks for understanding feature learning in neural networks remain fragmented.
One tractable approach is the Neural Tangent Kernel (NTK) theory~\citep{jacot2018_NeuralTangent,arora2019_ExactComputation,lee2019_WideNeurala}, which models wide neural networks in the infinite-width limit, where the feature map remains static, behaving like a kernel method with fixed representations.
Hence, one can explain the generalization ability of neural networks via the corresponding kernel regression theory~\citep{steinwart2008_SupportVector,caponnetto2007_OptimalRates}.
While this framework enables the analysis of neural networks via kernel methods, it fails to capture the dynamic feature learning of realistic neural networks, which operate with finite widths and evolve their feature representations during training~\citep{woodworth2020_KernelRich}.

Another line of research~\citep{ba2022_HighdimensionalAsymptotics,moniri2024_TheoryNonLinear,cui2024_AsymptoticsFeature,bordelon2024_HowFeature,lejeune2023_AdaptiveTangent,dandi2024_RandomMatrix,yang2022_FeatureLearning,damian2022_NeuralNetworks,dandi2023_HowTwoLayer} focuses on understanding the feature learning behavior of neural networks through the lens of random matrix theory.
Viewing shallow neural networks as random feature models, these studies consider training the feature weights using only one-step gradient descent with output weights fixed and analyze the resulting feature matrix and its spectral properties, showing the generalization properties of the resulting feature weights.
These researches demonstrate that the feature matrix is adjusted to align with the target function, leading to spikes in its spectrum~\citep{dandi2024_RandomMatrix}.

The over-parameterization nature of neural networks has also been studied under the perspective of implicit regularization.
A key insight is that over-parameterized models, when optimized via gradient-based methods, exhibit implicit biases toward simpler solutions, thus exhibiting better generalization.
Recent studies include linear models~\citep{hoff2017_LassoFractional}, matrix factorization~\citep{gunasekar2017_ImplicitRegularization,arora2019_ImplicitRegularization,li2021_ResolvingImplicit,razin2021_ImplicitRegularization} and other models~\citep{yun2021_UnifyingView,nacson2022_ImplicitBias,fan2021_UnderstandingImplicit}.

These approaches all face the challenging dichotomy: How can we reconcile the dynamic feature learning of neural networks while retaining the tractability of statistical analysis?


\subsection{Adaptive Feature Program}

While adaptive feature learning has been explored in various contexts~\citep{woodworth2020_KernelRich,gatmiry2021_OptimizationAdaptive,lejeune2023_AdaptiveTangent,li2025_DiagonalOverparameterization},
a unified framework capturing its core principles remains elusive.
Building on our prior survey~\citep{zhang2024_StatisticalUnderstanding}, we propose a general \emph{adaptive feature program} that integrates the dynamic learning capabilities of neural networks into a structured statistical framework for non-parametric regression.

Consider the non-parametric regression problem \( y = f^*(x) + \epsilon \), where \( x \sim \mu \) is drawn from a distribution on the input space \(\mathcal{X}\), \( \epsilon \) is independent noise, and \( f^* : \mathcal{X} \to \mathbb{R} \) is the unknown target function.
Given i.i.d.\ samples \(\{(x_i, y_i)\}_{i=1}^n\) from this model, our goal is to estimate \(f^*\).
Classical non-parametric regression methods rely on a fixed feature map \( \Phi : \mathcal{X} \to H \), transforming inputs \( x \) into a feature representation in a Hilbert space \( H \) (e.g., \( \ell^2(\mathbb{N}) \)).
The predictor is then defined as \( f(x) = \ang{\Phi(x), \bm{\beta}}_H \), where \( \bm{\beta} \in H \) is a trainable coefficient.
However, selecting an effective \( \Phi \) is challenging, often leading to suboptimal alignment with \( f^* \) and poor performance.

Beyond fixed feature maps, we propose a parameterized feature map \( \Phi_{\theta} : \mathcal{X} \to H \), where \( \theta \) is another trainable parameter and \( H \) is a fixed Hilbert space.
The predictor becomes \( f(x) = \langle \Phi_{\theta}(x), \bm{\beta} \rangle_H \).
We employ gradient descent to jointly optimize \( \theta \) and \( \bm{\beta} \).
Defining the empirical loss as $\caL_n = \frac{1}{2n} \sum_{i=1}^n \left( y_i - f(x_i) \right)^2$,
the adaptive feature program trains both \(\theta\) and \(\bm{\beta}\) simultaneously via gradient descent (flow):
\begin{equation}
  \label{eq:AdaptiveFeatureProgram}
  \left\{
    \begin{aligned}
      \dot{\theta}_t &= - \nabla_{\theta} \caL_n, \\
      \dot{\bm{\beta}}_t &= - \nabla_{\bm{\beta}} \caL_n,
    \end{aligned}
  \right.
\end{equation}
where \(\bm{\beta}\) is typically initialized to zero, and \(\theta\)'s initialization depends on the parameterization of $\Phi_{\theta}$.

The adaptive feature model allows the feature map \(\Phi_{\theta}\) to evolve during training, discovering a representation that aligns more closely with \(f^*\).
Moreover, this process is achieved automatically via gradient descent rather than requiring problem-specific estimates.
This dynamics mirrors the behavior of neural networks, where feature representations formed by the network's weights are learned implicitly through training.

The flexibility of the parameterization \(\Phi_{\theta}\) enables the model to integrate a variety of models.
If $\Phi_{\theta}$ is fixed, then it degenerates to the standard kernel gradient descent method~\citep{yao2007_EarlyStopping}.
Other instances of $\Phi_{\theta}$ include over-parameterized linear regression, diagonal adaptive kernel methods and directional adaptive feature methods
that will be introduced later in the paper.
With various parameterization of $\Phi_{\theta}$, adaptive feature models adapt to different types of data structures in a unified manner,
For instance, in over-parameterized linear regression \cref{eq:OpGDLinear}, adapting the feature map \(\Phi_{\theta}\) helps identify sparse signal components.
Similarly, in diagonal adaptive kernel methods \cref{eq:DiagAdaK_Seq}, the model adjusts the kernel's spectral weights, improving the alignment between the feature map and the underlying function \(f^*\).
Under single-index models, where \(f^*\) depends on a subspace projection of the input, adapting \(\Phi_{\theta}\) in \cref{eq:Def_AdaK_SIM} enables the model to learn the relevant projection direction.

While adaptivity offers clear, intuitive benefits, it also introduces new challenges for theoretical analysis.
Unlike fixed-feature methods, where the feature map's static nature often permits closed-form solutions for the optimal coefficient, the adaptive scheme's simultaneous evolution of via gradient descent typically lacks an analytic solution.
Also, this joint optimization leads to non-linear dynamics, as the simultaneous updates couple the feature map and coefficient in complex, data-dependent ways.
Furthermore, the gradient descent can overfit the noisy training data if run indefinitely,
so a refined analysis on the early stopping time is often necessary.

\begin{figure}[ht]
  \centering
  \includegraphics[width=1\textwidth]{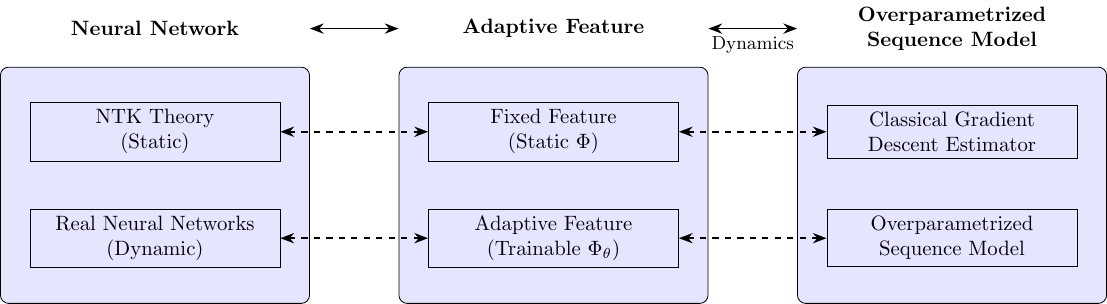}
  \caption{The program of this paper.
  We propose to model complex neural networks with adaptive feature program, capturing its dynamic feature learning.
  Moreover, we propose to analyze the adaptive features under the sequence model observation,
    which allows us to focus on the training dynamics while preserving the essence of non-parametric regression.
  }
  \label{fig:Diagram}
\end{figure}

\subsection{Feature Error Measure}

With a family of feature maps \(\Phi_{\theta}\) at hand, one crucial question arises: How to evaluate the effectiveness of the learned feature map?
To address this, we introduce the \emph{feature error measure}, an oracle metric designed to quantify how well the learned feature map \(\Phi_{\theta}\) aligns with the target function \(f^*\).
Let us consider feature maps of the form
\begin{equation}
  \Phi(x) = \xk{\lambda_j^{\hf} e_j(x)}_{j \in N} \in \ell^2(N),
\end{equation}
where \(N \subseteq \bbN\) is an index set, $\ell^2(N)$ is the space of square-summable sequences, \(\dk{e_j(\cdot)}_{j \in N}\) is an orthonormal system (not necessarily a basis) in \(L^2(\caX, \dd \mu)\), and \(\lambda_j \geq 0\) are weights.
Let \(L^2_{\Phi}\) be the subspace of \(L^2(\caX, \dd \mu)\) spanned by \(\dk{e_j}_{j \in N}\),
and denote the projection of \(f^*\) onto this subspace by \(P_{L^2_{\Phi}} f^* = \sum_{j \in N} f^*_j e_j\),
where \(f^*_j = \ang{f^*, e_j}_{L^2}\) are the coefficients of \(f^*\) in the orthonormal system.
We introduce the following definition of the feature error measure.

\begin{definition}
  The feature error measure, a function of $\delta, \epsilon^2 > 0$, is defined as
  \begin{equation}
    \label{eq:FeaturedError}
    \caE(\delta, \epsilon^2; \Phi, f^*) = \caE_{\mr{Proj}}(\Phi, f^*) + \caE_{\mr{Stat}}(\delta, \epsilon^2; \Phi, f^*),
  \end{equation}
  where the projection error $\caE_{\mr{Proj}}(\Phi, f^*)$ and statistical error $\caE_{\mr{Stat}}(\delta, \epsilon^2; \Phi, f^*)$ are given by
  \begin{align*}
    \caE_{\mr{Proj}}(\Phi, f^*) & \coloneqq \norm{f^* - P_{L^2_{\Phi}} f^*}_{L^2(\caX, \dd \mu)}^2, \\
    \caE_{\mr{Stat}}(\delta, \epsilon^2; \Phi, f^*) &\coloneqq \caE_{\text{V}} + \caE_{\text{B}}
    \coloneqq \abs{\dk{j \in N : \lambda_j \geq \delta}}  \cdot \epsilon^2 + \sum_{j \in N} (f_j^*)^2 \bm{1} \dk{\lambda_j < \delta}.
  \end{align*}
\end{definition}

The feature error measure \(\mathcal{E}(\delta, \epsilon^2; \Phi, f^*)\) quantifies the oracle error of the feature map \(\Phi\) in approximating the target function \(f^*\),
where the parameter \(\epsilon^2 > 0\) defines the effective noise level and \(\delta > 0\) acts as a truncation level.
It is composed of three components:
\begin{itemize}
  \item {Projection error} \(\caE_{\mr{Proj}}(\Phi, f^*)\): Measures the limit of the feature map \(\Phi\) in representing \(f^*\);
  \item {Variance term} \(\caE_{\text{V}}\): Reflects the model complexity via the number of significant components (i.e., \(\lambda_j \geq \delta\));
  \item {Bias term} \(\caE_{\text{B}}\): Captures the approximation error from features with small weights (i.e., \(\lambda_j < \delta\)).
\end{itemize}

By introducing $\epsilon^2$ instead of measuring the error with $n$ samples, this oracle quantity allows us to understand and analyze the feature map's performance in a more general sense.
Also, it enables us to separate the training of the feature map and the evaluation of the feature map, leading to a clearer understanding of the feature map's performance.
For \(n\) samples, $\epsilon^2$ typically scales as \(\epsilon^2 \asymp \frac{1}{n}\).

On the other hand, \(\delta\) sets a threshold: features with weights \(\lambda_j \geq \delta\) contribute to the model’s complexity, while those below are effectively ignored.
In the context of early-stopped gradient descent, we often have \(\delta \asymp t^{-1}\), where \(t\) is the training time—early stopping corresponds to a larger \(\delta\), limiting the number of active features to prevent overfitting.
Moreover, we can optimize \(\delta\) to minimize the error, defining optimally tuned error and optimal truncation level:
\begin{equation}
  \caE^*(\epsilon^2; \Phi, f^*) = \inf_{\delta \geq 0} \caE(\delta, \epsilon^2; \Phi, f^*), \quad
  \delta^*(\epsilon^2; \Phi, f^*) = \argmin_{\delta \geq 0} \caE(\delta, \epsilon^2; \Phi, f^*).
\end{equation}
However, we emphasize here that the optimally tuned error still depends on the feature map \(\Phi\) and it can differ substantially for different feature maps.

\begin{remark}
  We can always reformulate the projection error into the bias term of statistical error by extending the feature map with the orthogonal complement and zero weights.
  However, this approach can impact interpretability of the original form of the feature map,
  so we would like to keep the projection error in the feature error measure.
\end{remark}


The feature error measure quantifies the generalization error associated with a given feature map \(\Phi\).
Notably, considering the sequence model and assuming that $(e_j)_{j \geq 1}$ is a basis,
it captures the generalization error of linear estimators~\citep{johnstone2017_GaussianEstimation} associated with the feature map.
Particularly, the error measure \(\mathcal{E}_{\mr{Stat}}(\delta, \epsilon^2; \Phi, f^*)\) corresponds exactly to the generalization error of the estimator \(\hat{f}_j = \mathbf{1}_{\{\lambda_j \geq \delta\}} z_j\).
Moreover, let us consider the gradient descent estimator $\hat{f}^{\text{Seq}}_t$ in \cref{eq:MainText_Seq_Equiv},
which is also a linear estimator with closed form $\hat{f}_j  = (1 - e^{-\lambda_j t}) z_j$.
Its generalization error is given by
\begin{equation}
  \E \norm{f^* - \hat{f}^{\text{Seq}}_t }_{L^2}^2 = \mathcal{E}_{\text{V}}^{\text{GD}}(t) + \mathcal{E}_{\text{B}}^{\text{GD}}(t),\quad
  \mathcal{E}_{\text{V}}^{\text{GD}}(t) = \frac{1}{n} \sum_{j \geq 1} (1 - e^{-\lambda_j t})^2, \quad \mathcal{E}_{\text{B}}^{\text{GD}}(t) = \sum_{j \geq 1} e^{-2 \lambda_j t} (f_j^*)^2.
\end{equation}
By setting \(t = \delta^{-1}\), the terms \(\mathcal{E}_{\text{V}}^{\text{GD}}(t)\) and \(\mathcal{E}_{\text{B}}^{\text{GD}}(t)\) closely approximate \(\mathcal{E}_{\text{V}}\) and \(\mathcal{E}_{\text{B}}\), respectively.
These connections demonstrate that the feature error measure serves as a representative oracle proxy for measuring the quality of the feature map \(\Phi\) in learning the target function \(f^*\).

\subsection{Overparametrized Sequence Models}\label{subsec:SeqModel}

To focus on the dynamics of the adaptive feature while preserving the essence of non-parametric regression,
we further propose to consider the Gaussian sequence model~\citep{johnstone2017_GaussianEstimation}.
Suppose that we are given a fixed orthonormal basis \( (\phi_\ell)_{\ell \in \caI}\) in \(L^2(\caX, \dd \mu)\), where $\caI$ is an index set.
It has been observed in the literature~\citep{johnstone2017_GaussianEstimation,zhang2024_StatisticalUnderstanding,reiss2008_AsymptoticEquivalence,brown2002_AsymptoticEquivalence} that
observing $n$ samples in the non-parametric regression problem is effectively equivalent to observing the collection $(z_\ell)_{\ell \in \caI}$ in the sequence model
\begin{equation}
  \label{eq:SeqModel}
  z_\ell = f^*_\ell + \ep_{\ell}, \quad f^*_\ell = \ang{f^*, \phi_\ell}_{L^2(\caX, \dd \mu)}, \quad \ep_{\ell} \stackrel{i.i.d.}{\sim} \caN(0,\sigma^2/n),\quad \ell \in \caI.
\end{equation}
Here, $f^*_\ell$ represents the true coefficient of the target function \(f^*\) in the orthonormal basis, \(\ep_{\ell}\) is the noise term, being independent across different indices, and $\sigma^2$ is a variance parameter.
We point out that the variance $\sigma^2/n$ of the noise scales with the number of samples, reflecting the effect of averaging over $n$ samples.

Let $f$ be a candidate function.
Since \( (\phi_\ell)_{\ell \in \caI}\) is an orthonormal basis, the population loss (excess risk) can be written as
$\norm{f - f^*}_2^2 = \sum_{\ell \in \caI} (f_\ell - f^*_\ell)^2$, where $f_{\ell} = \ang{f, \phi_\ell}_{L^2}$.
Therefore, we define similarly the sequence loss in the sequence model as
\begin{equation}
  \label{eq:SeqLoss}
  \bar{\caL}_n(f) = \hf \sum_{\ell \in \caI} \xk{f_\ell - z_\ell}^2, \quad f_\ell = \ang{f, \phi_\ell}_{L^2},
\end{equation}
which corresponds to the empirical loss $\caL_n$ in the non-parametric regression problem.
Then, the adaptive feature program under sequence models is almost the same as it in the sample version,
as we only need to replace $\caL_n$ in \cref{eq:AdaptiveFeatureProgram} with $\bar{\caL}_n$.

The shift from finite samples to the sequence model observation is not only validated by the so-called ``Le Cam equivalence''~\citep{brown2002_AsymptoticEquivalence,reiss2008_AsymptoticEquivalence},
but also justified by recent works~\citep{li2024_GeneralizationError} on the generalization error of gradient descent with fixed feature map.
Let us consider the training process \cref{eq:AdaptiveFeatureProgram} with $\Phi = (\lambda_{\ell}^{\hf} \phi_{\ell})_{\ell \in \caI}$ fixed and denote by
\(\hat{f}^{\text{GD}}_t\) and $\hat{f}^{\text{Seq}}_t$ the resulting predictor at time \(t\) under the loss $\caL_n$ and $\bar{\caL}_n$ respectively.
Then, it has been established in \citet{li2024_GeneralizationError} that, under suitable conditions,
\begin{equation}
  \label{eq:MainText_Seq_Equiv}
  \norm{f^* - \hat{f}^{\text{GD}}_t }_{L^2}^2
  =  (1 + o_{\mathbb{P}}(1)) \E \norm{f^* - \hat{f}^{\text{Seq}}_t }_{L^2}^2,\qq{as} n \to \infty,
\end{equation}
where $o_{\mathbb{P}}(1)$ denotes a term that converges to zero in probability.

Furthermore, for other instances of the adaptive feature models, we can also observe empirically the closeness between the predictor under the empirical loss and sequence loss.
As shown in \cref{fig:SeqDiff}, the gap between the two predictors vanishes as the number of samples increases, as long as the training time is bounded in a certain range.
This similarity of the two dynamics allows us to consider the adaptive feature model under the sequence loss as an effective approximation.
We hypothesize that this strong ``path equivalence'' between the two training processes generally holds true for a broad class of adaptive feature models,
which is beyond the focus of our current work but will be an interesting future direction.

\subsection{Goal of the Paper}

In this paper, we will investigate various instances of the adaptive feature program across various statistical settings,
focusing on the dynamics of the feature map \(\Phi_{\theta}\) and its impact on the feature error measure.
We demonstrate that the adaptive feature models \emph{consistently reduces the feature error measure—sometimes monotonically, sometimes in distinct phases—often achieving near-optimal feature error rates}.
Focusing on the sequence model in \cref{sec:OpSeq}, our instances include high-dimensional linear regression, kernel regression, and single- and multi-index models,
each with its own unique feature map parameterization and training dynamics.
For linear and kernel regression, we explore diagonal adaptive methods with fixed feature bases, showing monotonic error reduction by aligning feature weights with the target function (e.g., \cref{thm:DiagSparseSeq_Monotonic}, \cref{thm:DiagonalKernelDecay}).
For single- and multi-index models, we investigate models that learn directional structures, revealing phased error reduction and near-optimal rates (e.g., \cref{thm:SIM_Sequence}, \cref{thm:Multi_Seq}).
Moreover, returning to the sample version in \cref{sec:RealAdaF}, we also demonstrate similar behavior for the adaptive feature program for diagonal adaptive features.
Numerical experiments also support our theoretical findings on adaptive features.
Our analysis highlights the adaptive feature program's ability to learn effective representations,
showing the potential of this framework in understanding the feature learning dynamics of neural networks and its implications for generalization.

\subsection{Notation}

We will use $C,c,C_1,C_2,\ldots$ to denote positive generic constants that may change from line to line, the dependence of which depends on the context.
We write $a \lesssim b$ if there exists a constant $C>0$ such that $a \leq C b$ and similarly for $\gtrsim$.
We use $a \asymp b$ if $a \lesssim b$ and $b \lesssim a$.
For an integer $n$, we denote by $[n] = \{1,2,\ldots,n\}$.
We denote by $\abs{X}$ the cardinality of a set $X$.
We use $L^2(\caX, \dd \mu)$ or simply $L^2$ for the Hilbert space of square-integrable functions with respect to the measure $\mu$
and $\ang{\cdot,\cdot}_{L^2}, \norm{\cdot}_{L^2}$ for its inner product and norm, respectively.


\section{Overparametrized Sequence Models}\label{sec:OpSeq}

In this section, we will investigate adaptive feature models in the context of overparametrized sequence models.

\subsection{Diagonal Adaptive Feature Models}\label{subsec:DiagAdaKSeq}

Let us consider a special setting of the adaptive feature model where the feature basis is fixed but the feature weights are trainable.
Although this setting seems to be simplistic,
recent studies\citep{vaskevicius2019_ImplicitRegularization,zhao2022_HighdimensionalLinear,li2024_ImprovingAdaptivity} have shown that certain adaptive feature methods can substantially improve the generalization performance compared to the fixed feature method.
In the following, we will further investigate the learning process of the features by means of the feature error measure in \cref{eq:FeaturedError}.

Let us consider a fixed feature basis $\dk{e_j}_{j \in N}$, where $N$ is an index set (e.g., $N = \bbN$).
Since the feature basis is fixed, the projection error $\caE_{\text{Proj}}(\Phi,f^*)$ is a fixed constant,
so we can assume without loss of generality that $f^*$ is contained in the span of $\dk{e_j}_{j \in N}$.
Then, the feature error measure in \cref{eq:FeaturedError} simplifies to
\begin{equation}
  \label{eq:FeaturedErrorFixed}
  \caE(\delta;\epsilon^2;\Phi,f^*) = \caE_{\text{Stat}}(\delta,\epsilon^2;\Phi,f^*)
  = \# \dk{j \in N :  \lambda_j \geq \delta} \epsilon^2 + \sum_{j \in N} (f_j^*)^2 \bm{1} \dk{ \lambda_j < \delta},
\end{equation}
where $f^* = \sum_{j \in N} f_j^* e_j$ is the true function expressed in the basis,
and $\lambda_j$ is the weight sequence associated with the feature map $\Phi$.
Under the fixed feature basis, the feature map effectively designates an indices' order of learning via the weight sequence.
The feature error measure is minimized when the order given by the feature map coincides with that of the truth function's coefficients.
Hence, the feature error measure can be interpreted as a measure of the ``misalignment'' between the truth function and the feature map.

\subsubsection{High Dimensional Sparse Mean}\label{subsubsec:DiagSparseSeq}

To warm up, let us consider the high-dimensional linear regression under the sequence model observations $z_j = w^*_j + \ep_j$ for $j \in [d]$, where $d$ represents the dimension.
Let us assume that $(w^*_j)_{j \in [d]}$ is a sparse vector with $s^*$ non-zero entries.
We consider the asymptotics when the dimension $d \geq n$ goes to infinite with $n$, while the sparsity $s^*$ is fixed.

Recent literature has proposed an over-parameterized gradient descent\citep{vaskevicius2019_ImplicitRegularization,zhao2022_HighdimensionalLinear} under this setting,
which is a special case of the adaptive feature program.
For $\bm\theta \in \R^d$, we take the parameterized feature map $\Phi_{\bm{\theta}}(x) = (\theta_j x_j e_j)_{j \in [d]} : \R^d \to \R^d$,
where $e_j$ is the $j$-th standard unit vector in $\R^d$.
Let $\bm\beta \in \R^d$ be the coefficient vector.
The predictor is defined by $f(x) = \ang{\bm\beta,\Phi_{\bm\theta}(x)}_{\R^d}$.
Recalling the adaptive feature program in \cref{eq:AdaptiveFeatureProgram},
we consider the following gradient descent dynamics:
\begin{equation}
  \label{eq:OpGDLinear}
  \left\{
    \begin{aligned}
      \dot{\bm{\beta}}(t) &= - \nabla_{\bm\beta} \bar\caL_n,\quad \beta_j(0) = 0; \\
      \dot{\bm\theta}(t) &= - \nabla_{\bm\theta} \bar\caL_n, \quad \theta_j(0) = \alpha,
    \end{aligned}
  \right.
\end{equation}
where $\alpha > 0$ is a common initialization that will be chosen later.
Here, we remark that while $\bm{\beta}$ and $\bm\theta$ seems to be symmetric, but their initializations are different.
More importantly, they have different interpretations: $\bm\beta$ is the coefficient of the output function, while $\bm\theta$ is the parameter of the feature map.

While the recent literature~\citep{vaskevicius2019_ImplicitRegularization,zhao2022_HighdimensionalLinear} view the over-parameterized gradient descent as ``implicit regularization'' and establish the generalization performance of the method,
we investigate this method under the adaptive feature perspective and study how the feature evolves during the training process,
which requires a refined analysis.
We have the following theorem, whose proof is contained in \suppref{subsec:OpSparseSeq}.

\begin{theorem}
  \label{thm:DiagSparseSeq_Monotonic}
  Consider the adaptive feature model \cref{eq:OpGDLinear}.
  With $t_* = t_*(n) \asymp \log n$ and $\alpha \asymp d^{-1/2}$,
  it holds with probability at least $1 - C d^{-2}$ that
  \begin{equation*}
    \caE^*(n^{-1};\Phi_{\bm{\theta}(t)},w^*) \quad \text{is monotonically decreasing in $t$ for $t \in [0,t_*]$}.
  \end{equation*}
  Furthermore,
  \begin{equation*}
    \caE^*(n^{-1};\Phi_{\bm{\theta}(0)},w^*) = \min\xk{\frac{d}{n}, \norm{w^*}_2^2} \gg \caE^*(n^{-1};\Phi_{\bm{\theta}(t_*)},w^*) = \frac{s^*}{n}.
  \end{equation*}
\end{theorem}

As an instance of the adaptive feature program, \cref{thm:DiagSparseSeq_Monotonic} demonstrates that
the over-parameterized high-dimensional linear regression improves the feature error measure during the training process.
The result shows that as soon as the training starts, the feature error measure decreases monotonically until the early stopping time \( t_* \).
Moreover, the initial feature map is agnostic to the true signal and has an error measure of $d/n$.
In contrast, by identifying the relevant features during training, the final feature map reduces the error to the optimal rate \( s^*/n \).


\subsubsection{Non-parametric Regression}\label{subsubsec:DiagNonparam}

We now turn to non-parametric regression under the sequence model observation \cref{eq:SeqModel} with the fixed feature basis $\dk{e_j}_{j \geq 1}$.
We consider the parameterized feature map in the form
\begin{equation}
  \label{eq:DiagAdaK_FeatureMap}
  \Phi_{\bm{\theta}}(x) = (\theta_j e_j(x))_{j \geq 1},\quad \bm{\theta} = (\theta_j)_{j \geq 1} \in \ell^2(\bbN),
\end{equation}
which is similar to the one in \cref{eq:OpGDLinear}.
With a coefficient vector $\beta \in \ell^2(\bbN)$, we define the predictor as
$f(x) = \ang{\beta,\Phi_{\bm\theta}(x)}_{\ell^2(\bbN)} = \sum_{j \geq 1} \beta_j \theta_j e_j(x)$,
and $f_j = \beta_j \theta_j$ being the corresponding coefficient.
Using \cref{eq:SeqLoss}, we consider the following adaptive feature model:
\begin{equation}
  \label{eq:DiagAdaK_Seq}
  \left\{
    \begin{aligned}
      \dot{\bm{\beta}}(t) &= - \nabla_{\bm\beta} \bar\caL_n,\quad \beta_j(0) = 0; \\
      \dot{\bm\theta}(t) &= - \nabla_{\bm\theta} \bar\caL_n, \quad \theta_j(0) = \lambda_j^{\hf},
    \end{aligned}
  \right.
\end{equation}
where $(\lambda_j)_{j \geq 1} \in \ell^1(\bbN)$ is a fixed weight sequence giving the initialization of the feature map.
While the generalization performance of the method \cref{eq:DiagAdaK_Seq} was studied in \citet{li2024_ImprovingAdaptivity},
we further investigate the evolution of the feature map using the feature error measure in \cref{eq:FeaturedErrorFixed}.

We make the following assumption on the weight sequence $\lambda_j$ and the truth coefficients $f_j^*$ as considered in \citet{li2024_ImprovingAdaptivity}.
\begin{assumption}
  \label{assu:DiagonalKernel_pq}
  Assume that $\lambda_j \asymp j^{-\gamma}$ for $\gamma > 1$.
  Furthermore, there exist $p > 0$ and $q > 1$ such that $f_{j(\ell)}^* \asymp \ell^{-\frac{p+1}{2}}$ for an index sequence $j(\ell) \asymp \ell^{q}$,
  and $f_{k}^* = 0$ for all other indices.
\end{assumption}
\cref{assu:DiagonalKernel_pq} quantifies the ``smoothness'' of the truth function as well as the ``misalignment'' between the truth coefficients and the initial weight sequence:
the former is characterized by the decay rate parameter $p$, while the latter is quantified by the parameter $q$.
Larger $q$ indicates a larger misalignment between the truth coefficients and the initial weight sequence.
This assumption holds, for example, if $f^*$ is a low-dimensional function expressed in a high-dimensional basis,
where $q$ often scales as the ambient dimension $d$.
We can establish the following theorem, which is proven in \suppref{subsec:OpSeq}.

\begin{theorem}
  \label{thm:DiagonalSeq_Monotonic}
  Consider the model defined in \cref{eq:DiagAdaK_Seq} under \cref{assu:DiagonalKernel_pq}.
  With $t_* = t_*(n) \asymp \sqrt{n / (\log n)}$, it holds with probability at least $1 - C n^{-2}$ that
  \begin{equation}
    \caE^*(n^{-1};\Phi_{\bm{\theta}(t)},f^*) \quad \text{is monotonically decreasing in $t$ for $t \in [0,t_*]$}.
  \end{equation}
  Furthermore, if $\gamma > \frac{1}{2}(1 + \frac{p}{q})$, then
  \begin{equation*}
    \caE^*(n^{-1};\Phi_{\bm{\theta}(0)},f^*) \asymp n^{-\frac{p}{p+q}} \gg
    n^{-(1-\frac{1}{2\gamma})} + n^{-\frac{p}{p+1}} (\log n)^{\frac{2p}{p+1}}
    \gtrsim
    \caE^*(n^{-1};\Phi_{\bm{\theta}(t_*)},f^*).
  \end{equation*}
\end{theorem}

Following \citet{li2024_ImprovingAdaptivity}, we can consider a deeper parameterization of the feature map.
Let $D \in \bbN^*$ be a fixed depth.
We consider the parameterized feature map
\begin{equation}
  \Phi_{\bm{\theta},\bm{b}}(x) = \xk{\theta_j b_j^D e_j(x) }_{j \geq 1},\quad \bm{\theta} = \xk{\theta_j}_{j \geq 1},~\bm{b} = \xk{b_j}_{j \geq 1}.
\end{equation}
Then, the predictor is given by
$f(x) = \ang{\beta,\Phi_{\bm\theta,\bm{b}}(x)}_{\ell^2(\bbN)} = \sum_{j \geq 1} \beta_j b_j^D \theta_j e_j(x)$, where $\beta_j, b_j, \theta_j$ are all trainable parameters.
The corresponding adaptive feature model writes
\begin{align}
  \label{eq:DiagAdaK_Seq_Multilayer}
  \left\{
    \begin{aligned}
      \dot{\bm{\beta}}(t) &= - \nabla_{\bm\beta} \bar\caL_n,\quad \bm{\beta}(0) = \bm{0}; \\
      \dot{\bm{\theta}}(t) &= - \nabla_{\bm\theta} \bar\caL_n, \quad \theta_j(0) = \lambda_j^{1/2}; \\
      \dot{\bm{b}}(t) &= - \nabla_{\bm{b}} \bar\caL_n, \quad b_j(0) = b_0,
    \end{aligned}
  \right.
\end{align}
where $b_0 > 0$ is a common initialization of the trainable weights $b_j$ which can be chosen according to $n$.
Regarding the deeper parameterization, we have the following theorem.

\begin{theorem}
  \label{thm:DiagonalKernelMonotonic_Multilayer}
  Consider the model defined in \cref{eq:DiagAdaK_Seq_Multilayer} under \cref{assu:DiagonalKernel_pq}.
  With $t_* = t_*(n) \asymp n^{\frac{D+1}{D+2}} / \sqrt{\log n}$ and $b_0 \asymp n^{-\frac{1}{2(D+2)}}$, it holds with probability at least $1 - C n^{-2}$ that
  \begin{equation}
    \caE^*(n^{-1};\Phi_{\bm{\theta}(t)},f^*) \quad \text{is monotonically decreasing in $t$ for $t \in [0,t_*]$}.
  \end{equation}
  Furthermore, if $\gamma > \frac{1}{D+2}(1 + \frac{p}{q})$, then
  \begin{equation*}
    \caE^*(n^{-1};\Phi_{\bm{\theta}(0)},f^*) \asymp n^{-\frac{p}{p+q}} \gg
    n^{-(1 - \frac{1}{(D+2)\gamma})}
    + n^{-\frac{p}{p+1}} (\log n)^{\frac{2p}{p+1}}
    \gtrsim
    \caE^*(n^{-1};\Phi_{\bm{\theta}(t_*)},f^*).
  \end{equation*}
\end{theorem}
\cref{thm:DiagonalSeq_Monotonic} and \cref{thm:DiagonalKernelMonotonic_Multilayer} show
the improvement of the feature error measure during the training process,
being similar to the regression case in \cref{thm:DiagSparseSeq_Monotonic} but more complicated.

The feature map \( \Phi_{\bm{\theta}(0)} \) has a feature error measure \( \caE^*(n^{-1}; \Phi_{\bm{\theta}(0)}, f^*) \asymp n^{-\frac{p}{p+q}} \),
which is largely impacted by the misalignment \( q > 1 \).
During the training process, the adaptive feature gradually adjusts to the truth function,
resulting in a feature error measure \( \caE^*(n^{-1}; \Phi_{\bm{\theta}(t)}, f^*) \) that is monotonically decreasing in \( t \).
The term \( n^{-(1 - \frac{1}{(D+2)\gamma})} \) in the final feature error measure comes from the initial misalignment that cannot be fully corrected.
Nevertheless, as long as the initial decay is fast enough that $\gamma > \frac{p+1}{D+2}$, the final feature error measure recovers the non-parametric optimal rate \( n^{-\frac{p}{p+1}} \) with a logarithmic factor.

Moreover, as observed in \citet{li2024_ImprovingAdaptivity}, the introduction of the depth \( D \) in \cref{thm:DiagonalKernelMonotonic_Multilayer} can potentially improve the feature error measure by relaxing the need for initial fast decay.
The benefits of depth appear on the extra error term \( n^{-(1 - \frac{1}{(D+2)\gamma})} \) caused by the misalignment,
which decreases as \( D \) increases.
This improvement stems from the deeper parameterization's enhanced flexibility to adjust feature weights during the training process.
This depth-enhanced adaptivity mirrors the behavior of deep neural networks, where multiple layers refine feature representations progressively.
However, this comes at the cost of increased computational complexity and a longer optimal stopping time \( t_* \asymp n^{\frac{D+1}{D+2}} / \sqrt{\log n} \), showing a trade-off between performance gains and training efficiency.

In summary, we have demonstrated that the adaptive feature models with a fixed feature basis consistently enhances the feature error measure across diverse statistical models.
The results show that the adaptive models can effectively learn an ``optimal'' feature map that aligns with the true function, achieving a feature error measure that approaches the non-parametric optimal rate.
These findings underscore the adaptive feature program's flexibility and robustness, bridging classical statistical methods with modern learning paradigms.

\subsection{Directional Adaptive Feature Models}

In this subsection, we shift our focus to the adaptive feature program with a learnable basis.
We investigate this approach within the context of Gaussian single-index and multi-index models, where the adaptive feature is designed to capture the underlying directional structure of the data.
By allowing the basis itself to evolve during training, this framework offers flexible mechanism to uncover latent directional information critical to these models.

Single-index and multi-index models have been studied in extensive prior literature~\citep{bietti2023_LearningGaussian,dudeja2018_LearningSingleindex,damian2024_ComputationalStatisticalGaps,bietti2022_LearningSingleindex,kuchibhotla2019_EfficientEstimation,arous2021_OnlineStochastic,fan2021_UnderstandingImplicit}.
Early works~\citep{kalai2009_IsotronAlgorithm,kakade2011_EfficientLearning} leveraged properties like invertibility or monotonicity of the link function under mild data distribution assumptions to enable learning, while \citet{dudeja2018_LearningSingleindex,arous2021_OnlineStochastic,arous2023_HighdimensionalLimit} developed harmonic analysis frameworks for (stochastic) gradient descent on Gaussian data, providing quantitative guarantees for single-index models.
Extensions to multi-index models~\citep{abbe2024_MergedstaircaseProperty,abbe2023_SGDLearning} address semi-parametric learning and sample complexity, often focusing on specific link function structures like the staircase property.
While these studies employ various estimation techniques and provide theoretical guarantees,
few have explored these models through the perspective of adaptive features.
Leveraging our unified adaptive feature framework, we analyze single-index and multi-index models to highlight the potential of a learnable basis.
Our goal is to showcase the potential of this program in learning directional information, offering a fresh perspective on these classical problems.




Let $d$ be the dimension and the covariate $x$ follows the $d$-dimensional standard Gaussian measure $\gamma_d = N(0,I_d)$.
We denote by $\ang{\cdot,\cdot}_{\gamma_d}$ the inner product in $\R^d$ with respect to $\gamma_d$.
A fundamental component for the Gaussian index models is the Hermite polynomials, which are orthogonal with respect to the Gaussian measure.
Let $H_m$, $m \geq 0$, denote the normalized (probabilistic) Hermite polynomials in one dimension,
which are orthonormal with respect to the Gaussian measure $N(0,1)$,
namely $\ang{H_m, H_n}_{\gamma_1} = \delta_{mn}$, where $\delta_{mn}$ is the Kronecker delta.
For higher dimensions, let  $\mm = (m_1,\ldots,m_d)$ be a multi-index.
We denote its degree by $\abs{\mm} = m_1 + \dots + m_d$.
The tensorized Hermite polynomial is defined as $H_{\mm}(x)= \prod_{j=1}^d H_{m_j}(x_j)$ for $x = (x_1,\dots,x_d) \in \R^d$,
a multivariate polynomial of total degree $\abs{\mm}$.
The set of tensorized Hermite polynomials $\{H_{\mm} : \mm \in \bbN^d\}$ forms an orthonormal basis of $L^2(\R^d,\gamma_d)$,
so any function $f \in L^2(\R^d,\gamma_d)$ can be expanded as
\( f = \sum_{\mm \in \bbN^d} f_{\mm} H_{\mm}, \) where the coefficients $f_{\mm} = \ang{f, H_{\mm}}_{\gamma_d}$.

Throughout this section, leveraging the orthonormal basis $\{H_{\mm} : \mm \in \bbN^d\}$, we consider following the Gaussian sequence model as in \cref{eq:SeqModel}:
\begin{equation}
  z_{\mm} = f^*_{\mm} + \ep_{\mm}, \quad f^*_{\mm} = \ang{f^*, H_{\mm}}_{\gamma_d}, \quad \ep_{\mm} \stackrel{i.i.d.}{\sim} N(0,1/n),\quad
  \mm \in \bbN^d,
\end{equation}
where $f^*_{\mm}$ represents the true coefficient of the target function and $\ep_{\mm}$ is the noise term.
The collection $(z_{\mm})_{\mm \in \bbN^d}$ constitutes the observed data.
Moreover, as in \cref{eq:SeqLoss}, for a candidate function $f$ on $\R^d$, we introduce the sequence loss
\begin{equation}
  \bar{\caL}_n(f) = \hf \sum_{\mm \in \bbN^d} \xk{f_{\mm} - z_{\mm}}^2,\quad f_{\mm} = \ang{f, H_{\mm}}_{\gamma_d}.
\end{equation}

\subsubsection{Single-Index Model}

Let us first consider the single-index model where the truth function is given by
\begin{equation}
  f^*(x) = g^*(\ang{w_*,x}),
\end{equation}
where the unit vector $w_* \in \bbS^{d-1}$ is an unknown direction, and $g^* \in L^2(\gamma_1)$ is an unknown link function.
Let $g^* = \sum_{r \geq 0} g^*_r H_r$ be the expansion of $g^*$ with respect to the Hermite polynomials.

As we aim to learn the unknown direction $w_*$, we consider the parameterized feature map given by
\begin{equation}
  \Phi_{w}(x) = \xk{\lambda_r^{\hf} H_r(\ang{w,x})}_{r \geq 0},\quad w \in \bbS^{d-1},
\end{equation}
where $w$ is a trainable vector that aims to learn the true direction, and $(\lambda_r)_{r \geq 0}$ is a fixed sequence of summable positive weights and
Corresponding to the Gaussian kernel~\citep{rasmussen2006_GaussianProcesses} where the Hermite polynomials serve as its eigen-basis and the eigenvalues exhibit an exponential decay, we take $\lambda_r = \exp(-\gamma r)$ for some fixed $\gamma > 0$.
Let $\bm\beta \in \ell^2(\bbN)$ be the functional coefficient parameter.
The predictor is given by
\begin{equation*}
  f(x) = \ang{\bm\beta, \Phi_{w}(x)}_{\ell^2(\bbN)} = \sum_{r \geq 0} \beta_r \lambda_r^{\hf} H_r(\ang{w,x}).
\end{equation*}
The training process of the adaptive feature model is then given by
\begin{equation}
  \label{eq:Def_AdaK_SIM}
  \left\{
    \begin{aligned}
      \dot{\bm\beta}(t) &= - \nabla_{\bm\beta} \bar{\caL}_n, \quad \bm{\beta}(0) = \bm{0}, \\
      \dot{w}(t) &= - \nabla_{w}^{\bbS^{d-1}} \bar{\caL}_n, \quad w(0) \sim \mr{Unif}(\bbS^{d-1}),
    \end{aligned}
  \right.
\end{equation}
where $\nabla_{w}^{\bbS^{d-1}}$ denotes the gradient on the sphere $\bbS^{d-1}$ and $\mr{Unif}(\bbS^{d-1})$ is the uniform distribution on the sphere.
Specifically, suppose $\nabla_w F$ is the classical gradient for a function $F$ on $\R^d$,
then the gradient on the sphere is given by $ \nabla_{w}^{\bbS^{d-1}} F = P_{w}^\perp \nabla_w F$,
where $P_{w}^{\perp} x = x - \ang{w,x} w$ is the orthogonal projection onto the tangent space of $\bbS^{d-1}$ at $w$.

Regarding the single index model, previous literature
~\citep{arous2021_OnlineStochastic,arous2023_HighdimensionalLimit}
has observed that the first non-zero coefficient of the expansion of $g^*$, which is referred to as the information exponent of $g^*$,
greatly influences the sample complexity for the single index model.
We formalize it by the following assumption.

\begin{assumption}
  \label{assu:SIM_InformationIndex}
  The function $g^*$ is fixed and its information exponent is $\rz \coloneqq \min\{r \geq 1 : g^*_r \neq 0\}$.
\end{assumption}

Moreover, we also introduce the following assumption on the decay of the coefficients of $g^*$.

\begin{assumption}
  \label{assu:SIM_g_decay}
  There exists $\alpha > 0$ such that the coefficients of $g^*$ satisfy $\abs{g^*_r} \lesssim r^{-\frac{\alpha+1}{2}}$.
\end{assumption}

To understand the training process of the adaptive feature model, our first theorem consider the population dynamics where we replace the loss function \( \bar{\caL}_n \)
by its population version \( \caL(f) = \hf \int_{\R^d} \xk{f(x) - f^*(x)}^2 \dd \gamma_d(x) \).

\begin{theorem}[SIM Population Dynamics]
  \label{thm:SIM_Population}
  Consider the population dynamics version of \cref{eq:Def_AdaK_SIM} under \cref{assu:SIM_InformationIndex}.
  Then, with probability one with respect to the random initialization,
  \begin{equation}
    \caE(\delta,\epsilon^2;\Phi_{w(t)},f^*) \text{ is monotonically decreasing in $t$},
  \end{equation}
  as is $\caE^*(\epsilon^2;\Phi_{w(t)},f^*)$.
  Moreover, under \cref{assu:SIM_g_decay}, with probability at least $0.99 - C \exp(-c d)$,
  it holds that
  \begin{equation}
    \caE^*(\epsilon^2;\Phi_{w(0)},f^*) - \caE^*(\epsilon^2;\Phi_{w_*},f^*) = \Theta(1),
  \end{equation}
  and there exists $T_1 \lesssim \log d + d^{\rz - 1}$ such that
  \begin{equation*}
    \caE^*(\epsilon^2;\Phi_{w(T_1 + s)},f^*)
    - \caE^*(\epsilon^2;\Phi_{w_*},f^*)
    \lesssim \exp(-C s), \quad \forall s \geq 0.
  \end{equation*}
\end{theorem}

The next result shows that the adaptive feature model can reduce the feature error under noisy observations.

\begin{theorem}
  \label{thm:SIM_Sequence}
  Consider the model defined in \cref{eq:Def_AdaK_SIM} under \cref{assu:SIM_InformationIndex}.
  Assume further that $n \gtrsim d^{2\rz + s}$ for some $s > 0$.
  Then, with probability at least $0.99 - C\exp(-c d)$ over the initialization and the randomness of the noise,
  there exist times $T_0 = \Theta(1) \leq T_1 \leq T_2 \lesssim \log d + \log n + d^{\rz - 1}$ such that
  \begin{equation}
    \caE(\delta,\epsilon^2;\Phi_{w(t)},f^*) \text{ is monotonically decreasing for $t \in [T_0, T_2]$},
  \end{equation}
  and it holds under \cref{assu:SIM_g_decay} that
  \begin{equation}
    \begin{aligned}
      & \caE^*(\epsilon^2;\Phi_{w(T_1)},f^*) - \caE^*(\epsilon^2;\Phi_{w_*},f^*) = \Theta(1), \\
      & \caE^*(\epsilon^2;\Phi_{w(T_1 + s)},f^*) - \caE^*(\epsilon^2;\Phi_{w_*},f^*) \lesssim \exp(-C s), \quad \forall s \in [T_2 - T_1], \\
      & \caE^*(\epsilon^2;\Phi_{w(T_2)},f^*) - \caE^*(\epsilon^2;\Phi_{w_*},f^*) \lesssim \xk{\frac{d}{n}}^{\min(\alpha,1)} \polylog(n,d).
    \end{aligned}
  \end{equation}
\end{theorem}

\cref{thm:SIM_Population} and \cref{thm:SIM_Sequence} analyze the performance of the adaptive feature model in single-index models, considering both the population (noiseless) dynamics and the noisy setting.
These results shed light on how the adaptive feature map $\Phi_{w(t)}$ is learned over time,
approximating the optimal feature map $\Phi_{w_*}$ to represent the target function $f^*$.
We discuss the key aspects below.

\paragraph{Measuring the Alignment}

There is an identifiability issue in the single-index model, as the alignment between $w$ and $w_*$ cannot be directly measured by their distance due to the inherent symmetry of the model, that is, flipping the sign of $w_*$ and adjusting $g^*$ accordingly does not change the function $f^*$.
This issue is naturally addressed by considering the excess feature error \( \caE^*(\epsilon^2; \Phi_{w}, f^*) - \caE^*(\epsilon^2; \Phi_{w_*}, f^*) \) as the alignment metric under our framework, sidestepping identifiability issues by directly assessing how well \( \Phi_{w} \) represents \( f^* \).
Nevertheless, we remark that our result can imply other alignment measure (such as $1 - \abs{\ang{w,w_*}}$) that is commonly used in the literature.

\paragraph{Improving the Feature Error Measure}
Due to the random initialization of $w(0)$, the initial excess feature error is at a constant level $\Theta(1)$.
As training progresses, the adaptive feature model effectively reduces the error.
In the population setting, \cref{thm:SIM_Population} establishes that the feature error measure \( \caE(\delta, \epsilon^2; \Phi_{w(t)}, f^*) \) and thus \( \caE^*(\epsilon^2; \Phi_{w(t)}, f^*) \) decreases monotonically as training time \( t \) increases.
Moreover, after a certain time $T_1$, the excess error decays exponentially fast, showing the improvement of the feature map \( \Phi_{w(t)} \) towards the optimal ones.

In the presence of noise, \cref{thm:SIM_Sequence} also shows that \( \caE(\delta, \epsilon^2; \Phi_{w(t)}, f^*) \) also exhibits a decreasing trend after an initial phase.
In addition, a similar exponential decay is observed until time \( T_2 \), where the excess error approaches the rate \( (d/n)^{\min(\alpha, 1)} \) up to logarithmic factors.
If the link function $g^*$ is smooth enough, namely $\alpha \geq 1$,
it achieves the parametric rate \( d/n \) up to logarithmic factors;
when \( \alpha < 1 \), the rate suffers from the limited smoothness of the link function, leading to a slower convergence rate.

\paragraph{Alignment and Approximation Error}
Since $\caE^*(\epsilon^2;\Phi_{w_*},f^*)$ is also determined by the smoothness of the link function,
we can further obtain full final feature error measure as the following corollary.
\begin{corollary}
  \label{cor:SIM_Sequence}
  Under the same conditions as \cref{thm:SIM_Sequence}, it additionally holds that
  \begin{equation*}
    \caE^*(n^{-1};\Phi_{w(T_2)},f^*)
    \lesssim \xk{\frac{d}{n}}^{\min(\alpha,1)} \polylog(n,d) + n^{-\frac{\alpha}{\alpha+1}}.
  \end{equation*}
\end{corollary}
As shown in \cref{cor:SIM_Sequence}, the final feature error is composed of two terms:
the first term represents the alignment error of the direction, while the second term captures the approximation error of the link function.
We can observe an interesting phase transition phenomenon.
Omitting the logarithmic factors, if $\alpha \geq 1$, then the approximation error dominates the alignment error iff $n \geq d^{1+\alpha}$,
while if $\alpha \leq 1$, the approximation error dominates iff $n \geq d^{1+\alpha^{-1}}$,
so the critical exponent is $1+\max(\alpha, \alpha^{-1})$.
This demonstrates an interesting phase transition phenomenon.
When $\alpha$ is large or small, learning the alignment is essential,
while when $\alpha$ is moderate, the main error source comes from learning the link function.

\paragraph{Phases of Learning}
In comparison to the training dynamics in \cref{subsec:DiagAdaKSeq} with basis fixed, where the feature error measure generally decreases smoothly,
the directional adaptive feature exhibits a more complex behavior, which can be divided into three phases.
At the initialization phase when $t \in [0,T_0]$, the model identifies the signal component of the link function at the information exponent by learning the corresponding coefficient of $g^*$, while the direction $w$ remains almost unchanged.
In the second phase, the small but identifiable signal allows the model to learn the direction from scratch, which in turn further amplify the signal.
This phase will take the time $T_1 \lesssim \log d + d^{\rz - 1}$, so larger information exponent $\rz$ leads to a longer time.
Finally, when the direction is basically learned,
we enter the final convergence phase, where the feature error measure decreases exponentially fast by refining the direction.
The three phases demonstrate how the adaptive feature with gradient descent can learn both the feature map and the link function simultaneously.

\paragraph{Impact of the Information Exponent \( \rz \)}

As observed in the previous literature~\citep{arous2021_OnlineStochastic,arous2023_HighdimensionalLimit}, the information exponent \( \rz \) plays a crucial role in both the training dynamics and the sample complexity.
On one hand, it determines the time required for the adaptive feature model to learn the direction as in \( T_1 \lesssim \log d + d^{\rz - 1} \).
On the other hand, the sample complexity $n \gtrsim d^{2\rz + s}$ also depends on the information exponent.
Intuitively, the information exponent determines the hardness of identifying the signal component of the link function \( g^* \).
In our result, although the dependency on the information exponent \( \rz \) is not optimal compared to previous works~\citep{bietti2022_LearningSingleindex,arous2021_OnlineStochastic} focusing on the single-index model,
we believe that it is sufficient to demonstrate the potential of the adaptive feature program.
We would like to leave the refinement as future work.

\subsubsection{Multi-Index Model}

The results of adaptive features for the single-index model can be extended to the multi-index model.
Let us define the Stiefel manifold $\mr{St}(d,p) = \dk{W \in \R^{d\times p} : W^\T W = I_p}$ as the set of $d \times p$ matrices with orthonormal columns.
The multi-index model is given by
\begin{equation}
  \label{eq:MultiIndexModel}
  f^*(x) = g^*(W_*^\T x), \quad W_* \in \mr{St}(d,p^*),
\end{equation}
where $W_*$ is the unknown direction and $g^* \in L^2(\gamma_{p^*})$ is an unknown low dimensional link function.
Moreover, let $g^* = \sum_{\mm \in \bbN^{p^*}} g^*_{\mm} H_{\mm}$ be the expansion of $g^*$ with respect to the $p^*$-dimensional Hermite polynomials.

For the multi-index model, we consider similarly the parameterized feature map given by
\begin{equation}
  \Phi_{W}(x) = \xk{\lambda_{\mm}^{\hf} H_{\mm}(W^\T x)}_{\mm \in \bbN^p}, \quad W \in \mr{St}(d,p),
\end{equation}
where $W$ is a trainable matrix representing the direction and $(\lambda_{\mm})_{\mm \in \bbN^p}$ is a fixed sequence of summable positive weights.
Particularly, we take $\lambda_{\mm} = \exp(-\gamma \abs{\mm})$ for some fixed $\gamma > 0$,
which corresponds to the tensorized version of the feature map in the single-index model.
Let $\bm\beta \in \ell^2(\bbN^p)$ be the functional coefficient parameter.
Then, the predictor is given by
\begin{equation*}
  f(x) = \ang{\bm\beta, \Phi_{W}(x)}_{\ell^2(\bbN^p)} = \sum_{\mm \in \bbN^p} \beta_{\mm} \lambda_{\mm}^{\hf} H_{\mm}(W^\T x).
\end{equation*}

Being substantially different from the single-index model, the multi-index model has a more complex structure due to its higher-dimensional directional component.
Unlike the single-index model, where the direction \( w \) is identifiable up to a sign, the multi-index model involves a matrix \( W \in \St(d, p) \), representing a subspace spanned by its columns via the orthogonal projection $W W^\T$, which is only unique up to orthogonal transformations.
Specifically, for any orthogonal matrix \( Q \in O(p) \), $W$ and $WQ$ span the same subspace, and thus the function $f = g(W^\T x)$ remains the same if $g$ is adjusted accordingly.
This rotational ambiguity poses extra technical challenge for the analysis.
To address this complexity and focus on the essential statistical properties, let us introduce the following assumption on the rotation invariance of the function \( g^* \).

\begin{assumption}
  \label{assu:Multi_RotationInvariance}
  We assume that $p = p^*$ is fixed and $g^* \in L^2(\gamma_{p})$ is a fixed rotationally invariant function.
\end{assumption}

While $g^*$ is assumed to be rotationally invariant, the complexity of the multi-index model remains, which lies in estimating the subspace spanned by $W_*$.
Therefore, there is still substantial difference between the single-index and multi-index models even with this assumption.
\cref{assu:Multi_RotationInvariance} allows us to partially simplify the analysis by focusing on the subspace rather than its specific orientation,
allowing us to study the model's core behavior more effectively.

For the gradient training process, we will also maintain the rotational invariance of the function \( g^* \) by restricting the coefficients \( \bm\beta \).
Let us introduce subspace of coefficients representing rotationally invariant functions as
\begin{equation*}
  \caG_{\bm{\lambda}}(p) = \dk{\bm\beta \in \ell^2(\bbN^p) : f = \sum_{ \mm \in \bbN^p}\lambda_{\mm}^{\hf} \beta_{\mm} H_{\mm} \text{ is rotationally invariant}}.
\end{equation*}
Let us denote by $\nabla_{\bm\beta}^{\caG_{\bm{\lambda}}(p)}$ the gradient in the subspace $\caG_{\bm{\lambda}}(p)$
and by $\nabla_{W}^{\mr{St}(d,p)}$ the gradient on the Stiefel manifold.
We consider the following adaptive feature model
\begin{equation}
  \label{eq:Def_AdaK_Multi}
  \left\{
    \begin{aligned}
      \dot{\bm\beta}(t) &= - \nabla_{\bm\beta}^{\caG_{\bm{\lambda}}(p)} \bar{\caL}_n, \quad \bm{\beta}(0) = \bm{0}, \\
      \dot{W}(t) &= - \nabla_{W}^{\mr{St}(d,p)} \bar{\caL}_n, \quad W(0) \sim \mr{Unif}(\mr{St}(d,p)),
    \end{aligned}
  \right.
\end{equation}
where the initialization $W(0) \sim \mr{Unif}(\mr{St}(d,p))$ is uniformly distributed over the Stiefel manifold.

Similar to the single-index model, we also introduce the information exponent of the function \( g^* \) in the multi-index model,
which is the minimum degree of the non-zero coefficients in the expansion of \( g^* \).

\begin{assumption}
  \label{assu:Multi_InformationIndex}
  The information exponent of $g^*$ is $m_0 \coloneqq \min\{ \abs{\mm} : g^*_{\mm} \neq 0\}$.
\end{assumption}

Moreover, we make the following assumption on the decay of the coefficients of \( g^* \),
where the term $p$ in the decay rate ensures the squared summability of the coefficients.

\begin{assumption}
  \label{assu:Multi_g_decay}
  The coefficients of $g^*$ satisfy $\abs{g^*_{\mm}} \lesssim \abs{\mm}^{-\frac{\alpha+p}{2}}$ for some $\alpha > 0$.
\end{assumption}

Our first result shows the convergence of the population dynamics.

\begin{theorem}[Population Dynamics]
  \label{thm:Multi_Population}
  Consider the population version of \cref{eq:Def_AdaK_Multi} under \cref{assu:Multi_RotationInvariance} and \cref{assu:Multi_InformationIndex}.
  Then, with probability one with respect to the random initialization,
  \begin{equation}
    \caE(\delta,\epsilon^2;\Phi_{W(t)},f^*) \text{ is monotonically decreasing in $t$},
  \end{equation}
  as is $\caE^*(\epsilon^2;\Phi_{W(t)},f^*)$.
  Moreover, with probability at least $0.99 - C \exp(-c d)$,
  it holds that
  \begin{equation}
    \caE^*(\epsilon^2;\Phi_{W(0)},f^*) - \caE^*(\epsilon^2;\Phi_{W_*},f^*) = \Theta(1),
  \end{equation}
  and under \cref{assu:Multi_g_decay}, there exists $T_0 \lesssim \log d + d^{m_0 - 1}$ such that
  \begin{equation}
    \caE^*(\epsilon^2;\Phi_{W(T_0 + s)},f^*)- \caE^*(\epsilon^2;\Phi_{W_*},f^*) \lesssim \exp(-C s), \quad \forall s \geq 0.
  \end{equation}
\end{theorem}

For the sequence model, we have the following result.

\begin{theorem}
  \label{thm:Multi_Seq}
  Consider the model defined in \cref{eq:Def_AdaK_Multi} under \cref{assu:Multi_RotationInvariance}, \cref{assu:Multi_InformationIndex} and \cref{assu:Multi_g_decay}.
  Assume further that $n \gtrsim d^{2m_0 + 1 + s}$ for some $s > 0$.
  Then, with probability at least $0.99 - C\exp(-c d)$ over the initialization and the randomness of the noise,
  there exist times $T_1 \leq T_2 \lesssim \log d + \log n + d^{m_0 - 1}$ such that
  \begin{align*}
    &\caE^*(\epsilon^2;\Phi_{W(0)},f^*) \geq \caE^*(\epsilon^2;\Phi_{W(T_1)},f^*) = \Theta(1), \\
    & \caE^*(\epsilon^2;\Phi_{W(T_1 + s)},f^*) - \caE^*(\epsilon^2;\Phi_{W_*},f^*) \lesssim \exp(-C s),\quad \forall s \in [T_2 - T_1], \\
    & \caE^*(\epsilon^2;\Phi_{W(T_2)},f^*) - \caE^*(\epsilon^2;\Phi_{W_*},f^*)
    \lesssim  p \xk{\frac{dp}{n}}^{\min(\alpha,1)} \polylog(n,d,p).
  \end{align*}
\end{theorem}

The proof of \cref{thm:Multi_Population} and \cref{thm:Multi_Seq} are deferred to \theSupp.
Let us discuss them in the following.

\paragraph{Improving the Feature Error Measure}
\cref{thm:Multi_Population} and \cref{thm:Multi_Seq} show that the adaptive feature method in the multi-index model has similar behaviors as in the single-index model.
The feature error measure exhibits multiple phases of learning, with the initial phase being constant and the subsequent phases showing exponential decay.
The final excess feature error scales as \(p (dp/n)^{\min(\alpha, 1)} \) up to logarithmic factors, with extra $p$ factors corresponding to the dimension of the direction.
However, we note that the sample complexity over $d$ is slightly larger by one than that in the single-index model,
which is due to technical reasons in the proof.
Overall, under the multi-index model, the adaptive feature model is also able to learn the direction and the link function simultaneously,
which is yet another illustrative example of the potential of adaptive features.

\paragraph{Proof Idea}
Let us briefly discuss the proof idea, while the detailed proof is highly technical and is deferred to \suppref{sec:Multi_Proof}.
The challenges lie in analyzing the matrix valued dynamics of $W$, its interaction with the functional coefficient $\bm\beta$ and the noise terms.
First, we introduce the matrix angle $\Rho = W^\T W_*$ and consider the singular value decomposition (SVD) $\Rho = U \Sigma V^\T$.
The alignment between $W$ and $W_*$ can then be measured by the closeness of $\Sigma$ to the identity matrix.
Focusing on $\Sigma$, we can simplify the complex matrix valued dynamics into entry-wise scalar dynamics.
However, due to the non-uniqueness of the SVD, these entry-wise dynamics depend on the choice of the orthogonal matrices $U,V$
and thus lead to noise terms that can not be controlled uniformly.
To resolve this, we introduce symmetric quantities (such as $\Tr \Sigma^2$) that are independent of the SVD\@.
One particular quantity is $\omega = -\log(\exp(-K\Sigma^2))/K$ for some $K > 0$, which is a smooth proxy of the minimum squared singular value.
Using this quantity, we can apply a multiple phase analysis to show the increase $\Sigma$ while providing a uniform bound on the noise terms.
Finally, the feature error measure can be controlled also in terms of $\Sigma$.
We believe that our proof technique can be applied to other matrix-valued models under noisy observations,
which can be of independent interest.

\paragraph{Comparison with the Literature}
Let us compare the results with the most relevant literature~\citep{bietti2023_LearningGaussian}, which also considers gradient training for the multi-index model.
One of the main differences is that we consider the noisy setting under the sequence model,
while \citet{bietti2023_LearningGaussian} only considers the population dynamics.
Another main differences is that we learn the functional coefficient $\bm\beta$ using simultaneous gradient descent \cref{eq:Def_AdaK_Multi},
while $\bm\beta$ is directly set to the interpolator at each time step in \citet{bietti2023_LearningGaussian}.
Their way of updating $\bm\beta$ is not suitable for the noisy setting as it leads to overfitting the noise.
Nevertheless, the training time-complexity $d^{m_0-1}$ in our results, though under a different training scheme, coincides with the time-complexity in \citet{bietti2023_LearningGaussian}.
This shows the intrinsic nature of the multi-index model and suggests that this adaptive feature model is able to learn the direction efficiently while prevent overfitting the noise.


\section{Connecting Sequence Model to Adaptive Features}\label{sec:RealAdaF}

In this section, we would like to show the similarities between the adaptive feature model under the sequence loss and the empirical loss via both theoretical and numerical studies,
justifying the focus on the sequence model in the previous section.

\subsection{Diagonal Adaptive Feature under Empirical Loss}

For the diagonal adaptive feature model, we can establish similar theoretical counterparts of the results in \cref{subsec:DiagAdaKSeq} under the empirical loss in the following.

\subsubsection{High Dimensional Linear Regression}

The sequence model in \cref{subsubsec:DiagSparseSeq} corresponding to the high dimensional linear regression.
Let us consider the high-dimensional linear regression model $y = \ang{w_*, x} + \ep$,
where $x \in \R^d$ is the $d$-dimensional input, $w_* \in \R^d$ is the true weight vector, and $\ep$ is an independent $\sigma^2$-sub-Gaussian noise.
We assume further that $\E x x^\T = I_d$ and each component of $x$ is sub-Gaussian with parameter $\sigma_x$.
Being the same as in \cref{subsubsec:DiagSparseSeq}, the true parameter $w_*$ is assumed to be a sparse vector with $s^*$ non-zero entries.
Let us be given i.i.d.\ samples $\dk{(x_i,y_i)}_{i=1}^n$.
The following result is a sample version of \cref{thm:DiagSparseSeq_Monotonic}.

\begin{theorem}
  \label{thm:DiagLinear_Monotonic}
  Under the assumptions of \cref{thm:DiagSparseSeq_Monotonic},
  consider \cref{eq:OpGDLinear} with the empirical loss $\caL_n$.
  With $t_* = t_*(n) \asymp \log n$ and $\alpha \asymp d^{-1/2}$,
  it holds with probability at least $1 - C d^{-2}$ that
  \begin{equation*}
    \caE^*(n^{-1};\Phi_{\bm{\theta}(t)},w^*) \quad \text{is monotonically decreasing in $t$ for $t \in [0,t_*]$}.
  \end{equation*}
  Furthermore,
  \begin{equation*}
    \caE^*(n^{-1};\Phi_{\bm{\theta}(0)},w^*) = \min\xk{\frac{d}{n}, \norm{w^*}_2^2} \gg \caE^*(n^{-1};\Phi_{\bm{\theta}(t_*)},w^*) = \frac{s^*}{n}.
  \end{equation*}
\end{theorem}

\subsubsection{Non-parametric Regression}

Let us now investigate the non-parametric regression problem corresponding to \cref{subsubsec:DiagNonparam} under the empirical loss.
Let the truth function admits the expansion $f^*(x) = \sum_{j=1}^\infty f_j^* e_j(x)$, where $\dk{e_j(x)}_{j \geq 1}$ is the orthonormal basis of $L^2$.
The samples are generated from $y = f^*(x) + \ep$, where $\ep$ is an independent sub-Gaussian noise.

Considering the empirical loss, we need the following assumption on the uniform boundedness of the eigenfunctions,
which is also introduced in \citet{li2025_DiagonalOverparameterization}.

\begin{assumption}
  \label{assu:EigenSystem}
  We assume that $\sup_{j \geq 1} \norm{e_j(x)}_{\infty} \leq C_{\mr{eigf}}$ for some constant $C_{\mr{eigf}} > 0$.
\end{assumption}


We have the following theorems, which are proven in \suppref{subsec:DiagonalKernel}.

\begin{theorem}
  \label{thm:DiagonalKernelDecay}
  Assume \cref{assu:DiagonalKernel_pq} and \cref{assu:EigenSystem} hold.
  Consider the model defined in \cref{eq:DiagAdaK_Seq} or \cref{eq:DiagAdaK_Seq_Multilayer} under the empirical loss $\caL_n$,
  with $b_0 \asymp n^{-\frac{1}{2(D+2)}}$ (if $D \neq 0$).
  Let $s > 0$ be an arbitrarily small constant and define $q = 2^{\frac{2(D+1)}{D+2}}$.
  Then, there exist $L \asymp (-\hf+s) \log n$,
  a decreasing sequence $\delta_l = C q^{-l}$ for $l \leq L$ satisfying $\delta_L \leq n^{-\hf+s}$,
  and times $t_0 = 0 < t_1 < \dots < t_{L} = t_* \lesssim n^{\frac{D+1}{D+2}}$ satisfying $t_l \lesssim \delta_l^{-l} \log n$, such that,
  with probability at least $1 - C n^{-2}$,
  \begin{equation}
    \caE^*(n^{-1};\Phi_{\bm{\theta}(t),\bm{b}(t)},f^*)
    \lesssim  \delta_l^{p} + n^{-\frac{p}{p+1}} +  n^{-(1 - \frac{1+s}{(D+2)\gamma})}
    \quad     \forall t \in [t_l, t_*],~ \forall l = 0,\dots,L.
  \end{equation}
  In particular,
  \begin{equation*}
    \caE^*(n^{-1};\Phi_{\bm{\theta}(t_*),\bm{b}(t_*)},f^*) \lesssim n^{-\frac{p}{p+1}+s} +  n^{-(1 - \frac{1+s}{(D+2)\gamma})}.
  \end{equation*}
\end{theorem}

Similar to  \cref{thm:DiagonalSeq_Monotonic} and \cref{thm:DiagonalKernelMonotonic_Multilayer}
\cref{thm:DiagonalKernelDecay} shows the that the diagonal adaptive feature methods also improve the feature error measure during the training process progressively under the empirical loss.
In addition, \cref{thm:DiagonalKernelDecay} exhibits a progressive staircase decrease pattern rather than monotonic decrease, which is due to the interaction across different coefficients under the empirical loss.
Nevertheless, the same final feature error measure can be obtained in \cref{thm:DiagonalKernelDecay} as in the sequence model.


\subsection{Numerical Studies}


We provide numerical simulation results in this subsection to further support our theoretical findings.
First, we present the evolution of the feature error measure (FEM) during the training process in \cref{fig:Decay}.
We can see that the feature error measure decreases as the training progresses.
For the diagonal adaptive feature, while the initial FEM decreases at $n$ increases, the final FEM more rapidly.
For the directional adaptive feature, the initial FEM remains a constant as $n$ increases, but the final FEM shows a clear decrease.
Both two settings show the improved performance via the adaptive feature program.

\begin{figure}[ht]
  \centering

  \subfigure{
    \begin{minipage}[t]{0.4\linewidth}
      \centering
      \includegraphics[width=1.\linewidth]{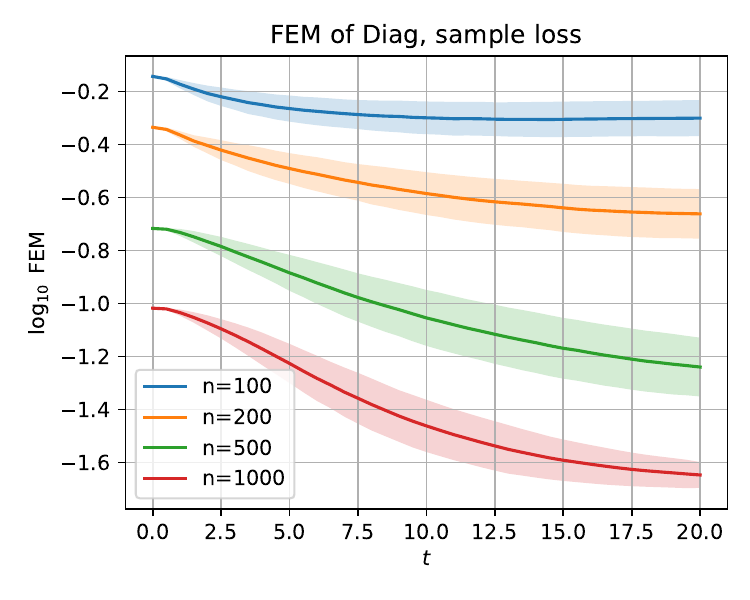}
    \end{minipage}%
  }%
  \subfigure{
    \begin{minipage}[t]{0.4\linewidth}
      \centering
      \includegraphics[width=1.\linewidth]{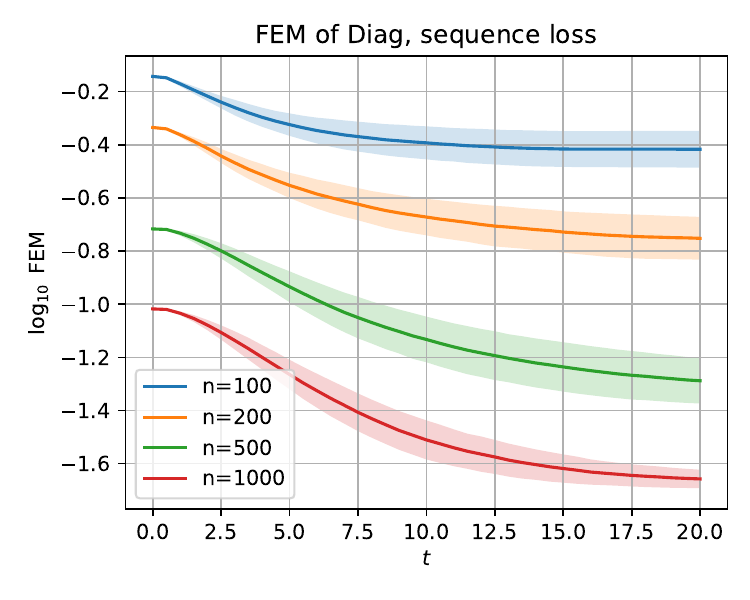}
    \end{minipage}%
  }%

  \subfigure{
    \begin{minipage}[t]{0.4\linewidth}
      \centering
      \includegraphics[width=1.\linewidth]{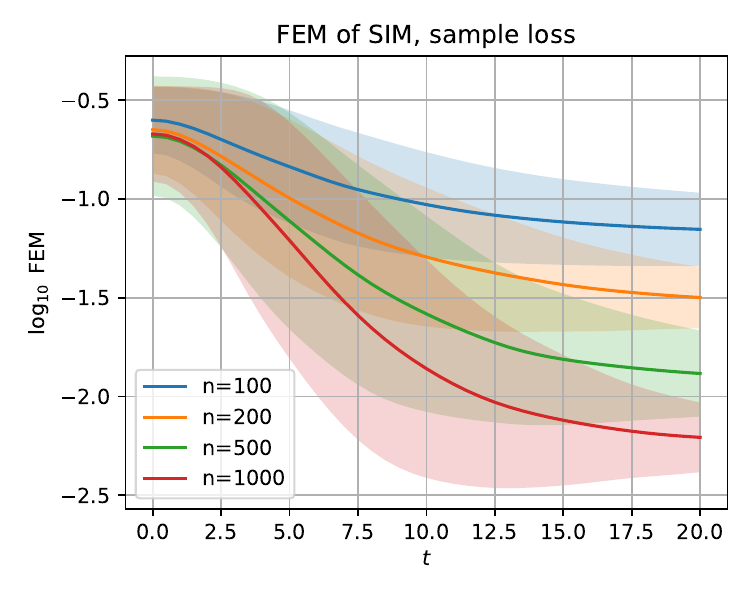}
    \end{minipage}%
  }%
  \subfigure{
    \begin{minipage}[t]{0.4\linewidth}
      \centering
      \includegraphics[width=1.\linewidth]{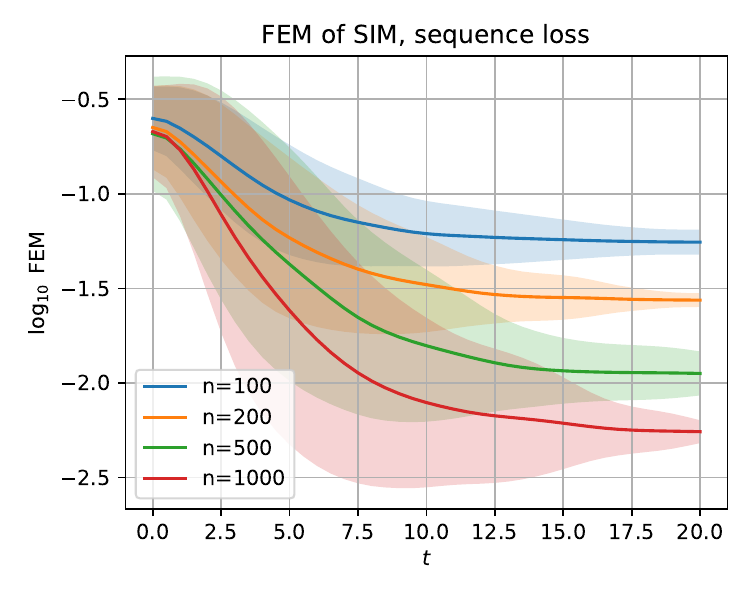}
    \end{minipage}%
  }%

  \caption{Decay of feature error measure $\caE^*$ (FEM) during the training process.
  Upper row: diagonal adaptive feature (Diag); lower row: directional adaptive feature for single-index model (SIM).
  Left column: empirical loss; right column: sequence loss.
  The shaded regions represent the standard deviation computed by 200 runs.
  }
  \label{fig:Decay}
\end{figure}

The similarity of the FEM curves in \cref{fig:Decay} between the sample loss and the sequence loss also validates the effectiveness of focusing on the sequence model.
Motivated by this similarity,
we would like to propose a strong path equivalence between the adaptive feature model under the two losses.

Formally, denoting by $\hat{f}^{\text{GD}}_t$ and $\hat{f}^{\text{Seq}}_t$ the predictor at time $t$ under the empirical loss $\caL_n$ and the sequence loss $\bar{\caL}_n$ respectively,
we hypothesize that the distributions of $\hat{f}^{\text{GD}}_t$ and $\hat{f}^{\text{Seq}}_t$ with respect to the random samples converge as $n \to \infty$.
As a result, the generalization errors and the feature error measures are also asymptotically equivalent.
This hypothesis is supported by the numerical results in \cref{fig:SeqDiff},
where we measure the distance between two distributions of functions via the energy distance with respect to the $L^2$ norm.
Furthermore, the FEMs under the two losses are also shown to converge in \cref{fig:FEM_Diff}.
However, proving this hypothesis in general can be very challenging and would require more involved analysis,
which is beyond the scope of this paper.
We would like to leave this as an open problem for future work.

\begin{figure}[ht]
  \centering
  \subfigure{
    \begin{minipage}[t]{0.32\linewidth}
      \centering
      \includegraphics[width=1.\linewidth]{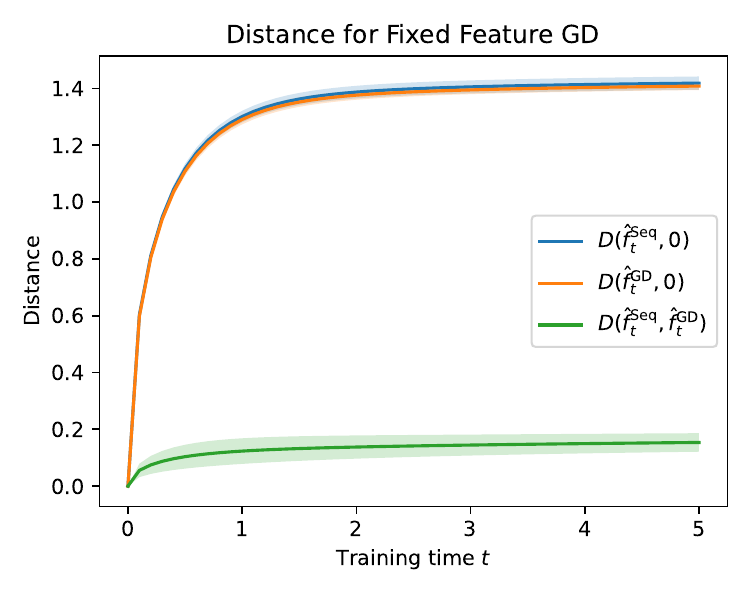}
    \end{minipage}%
  }%
  \subfigure{
    \begin{minipage}[t]{0.32\linewidth}
      \centering
      \includegraphics[width=1.\linewidth]{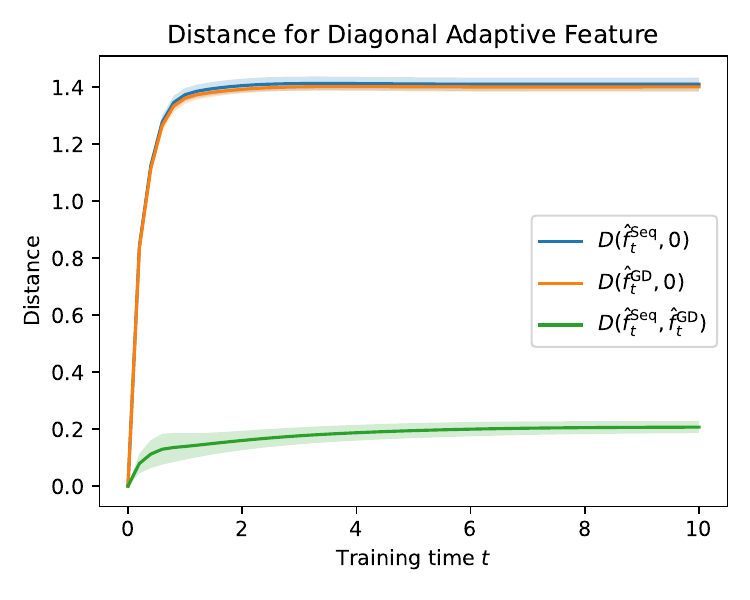}
    \end{minipage}%
  }%
  \subfigure{
    \begin{minipage}[t]{0.32\linewidth}
      \centering
      \includegraphics[width=1.\linewidth]{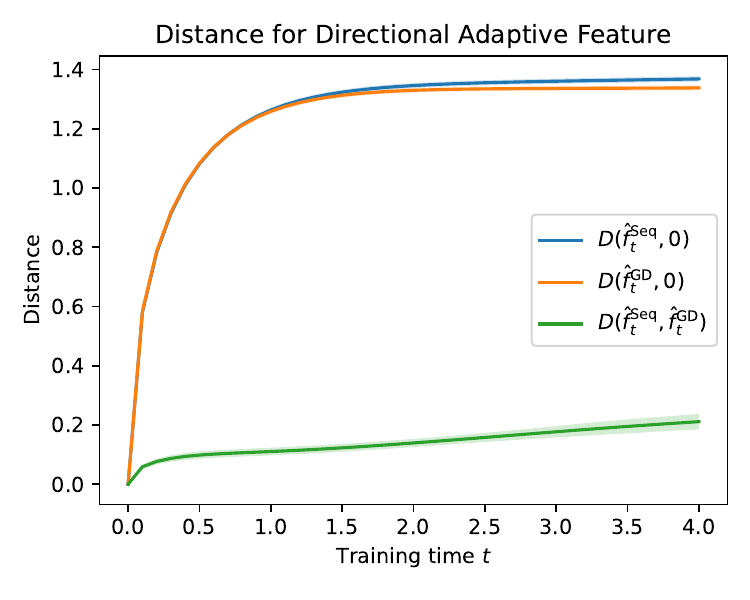}
    \end{minipage}%
  }%

  \subfigure{
    \begin{minipage}[t]{0.32\linewidth}
      \centering
      \includegraphics[width=1.\linewidth]{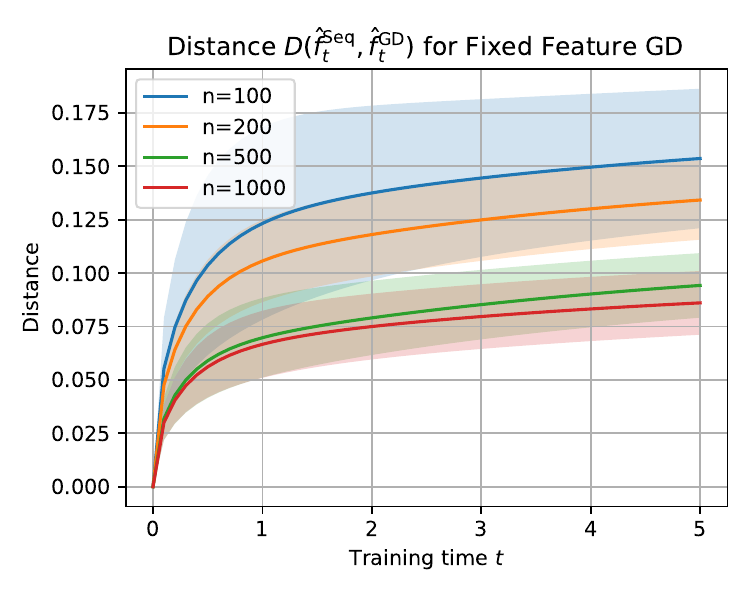}
    \end{minipage}%
  }%
  \subfigure{
    \begin{minipage}[t]{0.32\linewidth}
      \centering
      \includegraphics[width=1.\linewidth]{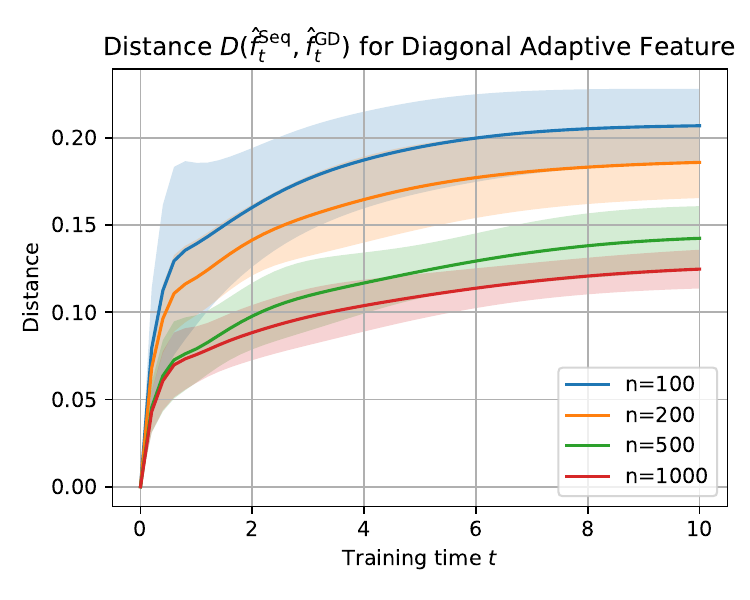}
    \end{minipage}%
  }%
  \subfigure{
    \begin{minipage}[t]{0.32\linewidth}
      \centering
      \includegraphics[width=1.\linewidth]{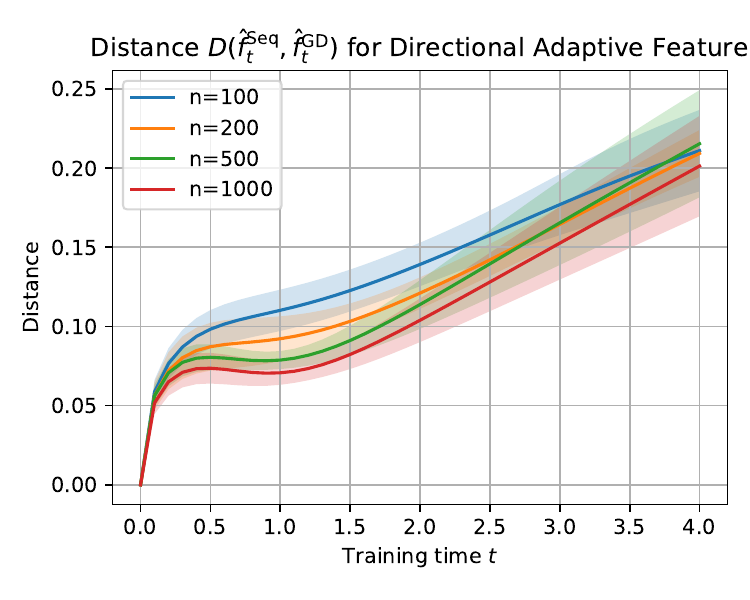}
    \end{minipage}%
  }%

  \caption{
    Similarity between the training curves under the empirical loss $\caL_n$ and sequence loss $\bar{\caL}_n$.
  We plot the energy distances estimated from 200 independent runs, and also shaded regions represent the standard deviation  estimated by bootstrapping.
  Upper row: $D(\hat{f}^{\text{Seq}}_t,\hat{f}^{\text{GD}}_t)$ is much smaller than that of $D(\hat{f}^{\text{Seq}}_t,0)$, $D(\hat{f}^{\text{GD}}_t,0)$ along the training path.
  Lower row: The difference between $\hat{f}^{\text{GD}}_t$ and $\hat{f}^{\text{Seq}}_t$ decreases as $n$ increases.
  The methods in three columns are fixed feature method,  diagonal adaptive kernel method and directional adaptive feature method respectively.
  }
  \label{fig:SeqDiff}
\end{figure}

\begin{figure}[ht]
  \centering

  \subfigure{
    \begin{minipage}[t]{0.4\linewidth}
      \centering
      \includegraphics[width=1.\linewidth]{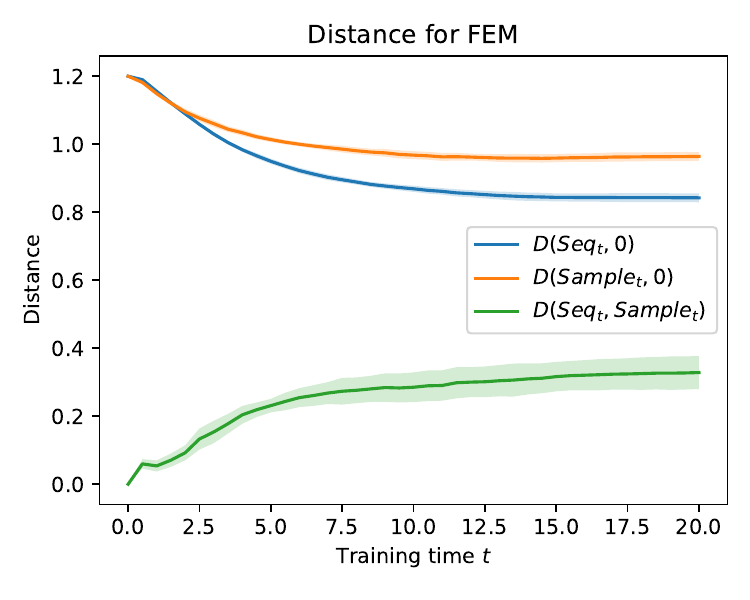}
    \end{minipage}%
  }%
  \subfigure{
    \begin{minipage}[t]{0.4\linewidth}
      \centering
      \includegraphics[width=1.\linewidth]{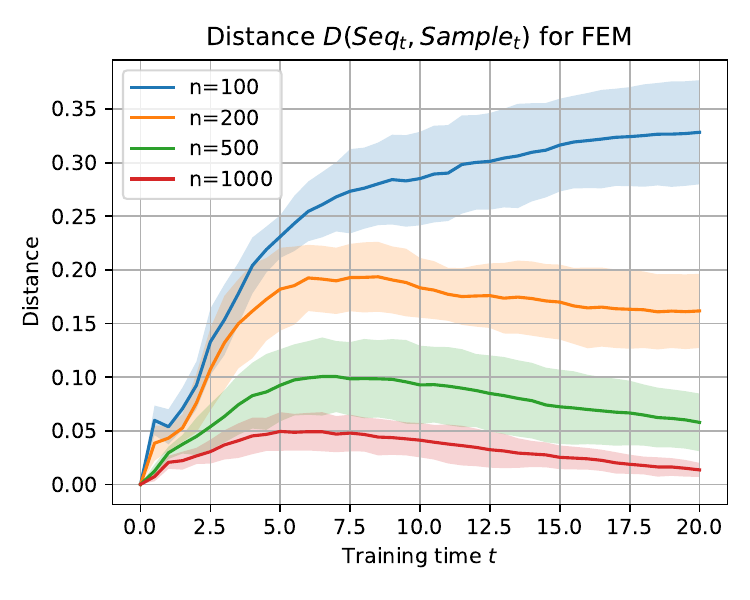}
    \end{minipage}%
  }%


  \caption{Energy distances between the feature error measure $\caE^*$ (FEM)  under the empirical loss $\caL_n$ and sequence loss $\bar{\caL}_n$.
  }
  \label{fig:FEM_Diff}
\end{figure}


\section{Conclusion}

In this paper, we consider the adaptive feature program, a unified framework that allows us to mirror the training dynamics of complex neural networks,
and propose the feature error measure, a metric that quantifies the quality of the feature map in learning the target function.
We investigate various instances of the adaptive feature scheme, including those with trainable feature weights and trainable feature basis,
and demonstrate its effectiveness in improving the feature error measure.
The adaptive feature scheme not only connects classical statistical techniques with modern machine learning methods,
but also provides new insights into the feature learning of neural networks.

\paragraph{Future Directions}
The adaptive feature scheme opens up several avenues for future research.
Besides the models considered in this paper, we can explore other models that can be expressed in the adaptive feature scheme,
such as random feature~\citep{rudi2016_GeneralizationProperties} and matrix factorization~\citep{gunasekar2017_ImplicitRegularization,arora2019_ImplicitRegularization} models.
Moreover, we can investigate the  parameterization form of the feature map $\Phi_{\theta}$ corresponding to different neural network architectures such as convolutional neural networks and transformers.
Another technical direction is to study the strong ``path equivalence'' (see \cref{subsec:SeqModel} and \cref{sec:RealAdaF}) between the empirical loss and sequence loss for general adaptive feature models, which will deeply enhance the understanding of non-parametric regression.
We believe that these explorations will lead to a deeper understanding of the feature learning process in neural networks and its implications for generalization.

\paragraph{Acknowledgements.}
Qian Lin's research was supported in part by the National Natural Science Foundation of China (Grant 92370122, Grant 11971257).

\clearpage
\bibliographystyle{plainnat}
\bibliography{main}

\clearpage
\appendix

\paragraph{Additional Notations}

Let us introduce some additional notations that will be used in the proofs.
We denote by $\log^+(x) = \max(\log x, 0)$.
For a function $f(z)$, we denote by $[z^r] f(z)$ the coefficient of $z^r$ in the Taylor expansion of $f(z)$ around $0$ (provided that it is well-defined).


\section{Proof for Diagonal Overparametrization}

In the following, let us fix the feature basis $\dk{e_j}_{j \geq 1}$ as well as the truth function $f^*$ and thus the coefficients $\dk{f_j^*}_{j \geq 1}$.
Now, the feature error measure is only related to the weights $\bm{\lambda}= \xk{\lambda_j}_{j \geq 1}$.
To simplify the notation, we denote
\begin{equation*}
  \caE(\delta,\epsilon^2;\bm{\lambda}) =  \caE(\delta,\epsilon^2;\Phi,f^*)  = \# \dk{j \in N :  \lambda_j \geq \delta} \epsilon^2 + \sum_{j \in N} (f_j^*)^2 \bm{1} \dk{ \lambda_j < \delta},
\end{equation*}
and
\begin{equation*}
  \caE^*(\epsilon^2;\bm{\lambda}) = \inf_{\delta \geq 0} \caE(\delta;\epsilon^2;\bm{\lambda}),\qquad
  \delta^*(\epsilon^2;\bm{\lambda}) \in \argmin_{\delta \geq 0} \caE(\delta;\epsilon^2;\bm{\lambda}).
\end{equation*}
From the expression of the feature error measure, it is clear that $\caE^*(\epsilon^2;\bm{\lambda})$
only depends on the order of the indices induced by the weights $\bm{\lambda}$.
In addition, one can choose $\delta^*(\epsilon^2;\bm{\lambda}) = \lambda_j$ for some $j \geq 1$.
Furthermore, we have the local condition:
\begin{equation}
  \label{eq:FEM_LocalCondition}
  \sum_{\lambda_j = \delta^*(\epsilon^2;\bm{\lambda})} (f_j^*)^2
  \geq  \# \dk{j \in N :  \lambda_j = \delta} \epsilon^2,
\end{equation}
since otherwise we can increase $\delta$ to obtain a smaller error.

\paragraph{Further notations.}
For index sets $I,J$,
we use $v_I$ to denote the vector with indices in $I$ and $A_{IJ}$ to denote the submatrix with rows in $I$ and columns in $J$.

\subsection{Basic properties on the feature error measure}
Let us define the index sets of signals and noises as
\begin{equation}
  \caI_{\mr{s}}(\epsilon^2) = \dk{j : (f_j^*)^2 \geq \epsilon^2},\qquad
  \caI_{\mr{n}}(\epsilon^2) = \dk{j : (f_j^*)^2 < \epsilon^2}.
\end{equation}
The following proposition characterizes sufficient conditions for the feature error measure to be non-increasing.

\begin{proposition}
  \label{prop:FEM_Monotonic}
  Let $\bm{\lambda}$ be a sequence of weights and $\bm{\lambda}'$ be the result of modifying $\bm{\lambda}$ by changing only $\lambda_j$ to $\lambda_j'$.
  Suppose that $\lambda_r = \delta^*(\epsilon^2;\bm{\lambda})$.
  Then, $\caE^*(\epsilon^2;\bm{\lambda}') > \caE^*(\epsilon^2;\bm{\lambda})$ is only possible if
  \begin{enumerate}[(a)]
    \item $\lambda_j < \lambda_r$, $\lambda_j' \geq \lambda_r$ and $(f_j^*)^2 < \epsilon^2$;
    \item $\lambda_j \geq \lambda_r$, $\lambda_j' < \lambda_r$, $(f_j^*)^2 > \epsilon^2$ and there is some $\lambda_l$ such that
    $(f_l^*)^2 < \epsilon^2$, $\lambda_r > \lambda_l \geq \lambda_j$.
  \end{enumerate}
\end{proposition}
\begin{proof}
  We enumerate the following cases and consider them one by one:
  \begin{enumerate}[(1)]
    \item $\lambda_j \geq \lambda_r$ and $\lambda_j' \geq \lambda_r$, or $\lambda_j < \lambda_r$ and $\lambda_j' < \lambda_r$;
    \item $\lambda_j < \lambda_r$, $\lambda_j' \geq \lambda_r$;
    \item $\lambda_j \geq \lambda_{k}$, $\lambda_j' < \lambda_r$.
  \end{enumerate}
  For case (1), we always have
  \begin{equation*}
    \caE^*(\epsilon^2;\bm{\lambda}') \leq \caE(\lambda_r,\epsilon^2;\bm{\lambda}') = \caE(\lambda_r,\epsilon^2;\bm{\lambda}) = \caE^*(\epsilon^2;\bm{\lambda}).
  \end{equation*}

  \noindent For case (2), if $(f_j^*)^2 \geq \epsilon^2$, we find that
  \begin{equation*}
    \caE^*(\epsilon^2;\bm{\lambda}) - \caE(\lambda_r,\epsilon^2;\bm{\lambda}')
    = \caE(\lambda_r,\epsilon^2;\bm{\lambda}) - \caE(\lambda_r,\epsilon^2;\bm{\lambda}')
    = (f_j^*)^2 - \epsilon^2 \geq 0,
  \end{equation*}
  so $\caE^*(\epsilon^2;\bm{\lambda}') > \caE^*(\epsilon^2;\bm{\lambda})$ only if $(f_j^*)^2 < \epsilon^2$, which is case (a).

  \noindent For case (3), if $(f_j^*)^2 \leq \epsilon^2$, similar to the previous case, we have
  \begin{equation*}
    \caE^*(\epsilon^2;\bm{\lambda}) - \caE(\lambda_r,\epsilon^2;\bm{\lambda}') = \epsilon^2 - (f_j^*)^2 \geq 0.
  \end{equation*}
  Now, if $(f_j^*)^2 > \epsilon^2$, but there is no $\lambda_l$ as specified in (b), we have
  \begin{equation*}
    \caE^*(\epsilon^2;\bm{\lambda}) - \caE(\lambda_j',\epsilon^2;\bm{\lambda}')
    = \caE(\lambda_r,\epsilon^2;\bm{\lambda}) - \caE(\lambda_j',\epsilon^2;\bm{\lambda}')
    = \sum_{l : \lambda_r >  \lambda_l \geq \lambda_j'} \zk{(f_l^*)^2 - \epsilon^2} \geq 0.
  \end{equation*}
\end{proof}

\begin{corollary}
  \label{cor:FEM_Monotonic}
  Under the same setting as in \cref{prop:FEM_Monotonic},
  $\caE^*(\epsilon^2;\bm{\lambda}') > \caE^*(\epsilon^2;\bm{\lambda})$ is only possible if there is an ``up-crossing''.
  Namely, there is some indices $j,k$ such that
  (1) $\lambda_j \geq \lambda_r > \lambda_k$; (2) $(f_j^*)^2 \geq \epsilon^2$ and $(f_k^*)^2 < \epsilon^2$;
  (3) $\lambda_j' \leq \lambda_k'$.
\end{corollary}
\begin{proof}
  For the case (b), the condition already holds for the pair $(j,l)$.
  For the case (a), using the local condition \cref{eq:FEM_LocalCondition}, we can find there is some $k$ with $\lambda_k = \lambda_r$ and $(f_k^*)^2 \geq \epsilon^2$.
  Then, the pair $(k,j)$ satisfies the conditions.
\end{proof}

From \cref{prop:FEM_Monotonic}, we find that $\caE^*(\epsilon^2;\bm{\lambda})$
is non-increasing after the change of $\bm{\lambda}$ if there is

\begin{lemma}
  \label{lem:FEM_Monotonic}
  Let $\lambda(t), t \in [0,T]$ be a continuous flow of weights.
  Let $N = N_1 \sqcup N_2$ be a partition of the index set $N$.
  Assume further that
  \begin{enumerate}[(1)]
    \item For each $j \in N_1 \cap \caI_{\mr{s}}(\epsilon^2)$ and $k \in N_1 \cap \caI_{\mr{n}}(\epsilon^2)$,
    if there is some $t_0$ such that $\lambda_j(t_0) \geq \lambda_k(t_0)$, then $\lambda_j(t) \geq \lambda_k(t)$ for all $t \geq t_0$.
    \item For each $j \in N_2$, $\lambda_j(t) < \delta^*(\epsilon^2;\bm{\lambda}(t))$.
  \end{enumerate}
  Then, $\caE^*(\epsilon^2;\bm{\lambda}(t))$ is non-increasing in $t$.
\end{lemma}
\begin{proof}
  Using the continuity of the weights and that $\caE^*(\epsilon^2;\bm{\lambda})$ only depends on the order of the indices induced by the weights,
  we can reduce the continuous dynamics of $\bm{\lambda}(t)$ to discrete steps that change only one weight at a time (if $N$ is infinite, we can take a finite but large subset).
  Then, the result follows from \cref{cor:FEM_Monotonic}:
  for $j \in N_2 \cap \caI_{\mr{s}}(\epsilon^2)$,  up-crossing can not happen we always have $\lambda_j(t) < \delta^*(\epsilon^2;\bm{\lambda}(t))$;
  for $j \in N_1 \cap \caI_{\mr{s}}(\epsilon^2)$, the condition (1) also ensures that the up-crossing can not happen for $k \in N_1 \cap \caI_{\mr{n}}(\epsilon^2)$,
  while the condition (2) also ensures that the up-crossing can not happen for $k \in N_2  \cap \caI_{\mr{n}}(\epsilon^2)$.
\end{proof}

\subsection{Results on One-dimensional Dynamics}

In this subsection, we will collect some results regarding the one-dimensional dynamics encountered in both the over-parameterized linear regression and the diagonal adaptive kernel.
Let us consider the one-dimensional gradient flow equation
\begin{align}
  \label{eq:TwoLayerEq}
  \left\{
    \begin{aligned}
      \dot{\theta}(t) &= \beta(t) (z(t) - w(t)), \quad \theta(0) = \lambda^{\hf} > 0, \\
      \dot{\beta}(t) &=  \theta(t) (z(t) - w(t)), \quad \beta(0) = 0,
    \end{aligned}
  \right.
\end{align}
where $z(t)$ is a continuous function $\lambda > 0$ is a constant and $w(t) = \theta(t) \beta(t)$.
Then, we can also compute the dynamics of $w(t)$ that
\begin{equation*}
  \dot{w}(t) = (\theta^2(t) + \beta^2(t)) (z(t) - w(t)),\quad f(0) = 0.
\end{equation*}
In the following, we denote $\tilde\lambda(t) = \theta(t)^2$ and $\tilde\lambda = \lambda(0)$.

Following the analysis in the literature\citep{li2024_ImprovingAdaptivity},
we can compute that
\begin{align*}
  \frac{1}{2}\dv{t} \theta^2 = \frac{1}{2}\dv{t} \beta^2 = \theta \beta (z - w),
\end{align*}
so we have
\begin{align}
  \label{eq:TwoLayerEq_Conservation}
  \theta^2(t) - \beta^2(t) = \theta(0)^2 - \beta(0)^2 = \lambda.
\end{align}
This also shows that $\theta(t) \geq \lambda^{\hf}$.

In addition, if $z(t)$ does not change sign, $\beta$ will have the same sign as $z$.
Moreover, if $z(t) \equiv z$ is a constant, we know that $\theta(t)$ and $\abs{\beta(t)}$ are monotonically increasing.

\begin{lemma}
  \label{lem:TwoLayerEq_Comparison}
  Consider two instances of \cref{eq:TwoLayerEq} with $\theta,\beta,z, \lambda$ and $\theta',\beta',z', \lambda'$ respectively,
  Suppose that $\min_t \abs{z(t)} \geq \max_t \abs{z'(t)}$.
  Then, if $\theta(t_0) \geq \theta'(t_0)$ for some $t_0 \geq 0$, we have $\theta(t) \geq \theta'(t)$ for all $t \geq t_0$.
\end{lemma}
\begin{proof}
  Without loss of generality, we can assume $\min_t z(t) \geq \max_t \abs{z'(t)} \geq 0$.
  First, for the case that $\lambda \geq \lambda'$, it is easy to see that $\theta(t) \geq \theta'(t)$ and $\beta(t) \geq \abs{\beta'(t)}$ for all $t \geq 0$ using the comparison principle.
  Now, if $\lambda < \lambda'$, using \cref{eq:TwoLayerEq_Conservation}, we find that
  \begin{equation*}
    \beta(t_0)^2 = \theta^2(t_0) - \lambda \geq \xk{\theta'(t_0)}^2 - \lambda' = \xk{\beta'(t_0)}^2,
  \end{equation*}
  so using the comparison principle again, we find that $\beta(t) \geq \abs{\beta'(t)}$ and $\theta(t) \geq \theta'(t)$ for all $t \geq t_0$.
\end{proof}

\begin{lemma}
  \label{lem:TwoLayerEq_NoiseUpperBound}
  Denote $M = \max_t \abs{z}$.
  We have
  \begin{equation*}
    \theta(t) \leq \sqrt{2} \lambda^{1/2},\quad \abs{w(t)} \leq \sqrt{2}\lambda \quad \forall t \leq \frac{1}{\sqrt{2}M}
  \end{equation*}
  and
  \begin{equation*}
    \theta(t)    \leq \lambda^{1/2} \zk{1+ \exp(\sqrt{2} t M)}
  \end{equation*}
\end{lemma}
\begin{proof}
  We can use the same proof as Lemma 16 in \citet{li2025_DiagonalOverparameterization}.
  For the bound on $\theta$, we use $ \theta(t) \leq \sqrt{\lambda + \beta^2(t)}$.
\end{proof}

\begin{lemma}
  \label{lem:TwoLayerEq_SignalLowerBound}
  Suppose $m = \min_t \abs{z} > 0$.
  We have
  \begin{equation*}
    \theta(t)^2 \geq \frac{1}{2}m, \qq{for}
    t \geq m^{-1} \xk{2 + \log^+ \frac{m}{2\lambda}}.
  \end{equation*}
\end{lemma}
\begin{proof}
  Let us remove the subscript $j$ for ease of notation.
  We define
  \begin{equation*}
    T^{\mr{esc}} = \inf\dk{t \geq 0: \abs{\beta(t)} \geq \lambda^{1/2}},\quad
    T^{\mr{sig}} = \inf\dk{t \geq 0: \abs{w(t)} \geq m/2}.
  \end{equation*}
  We note that if $\abs{w(t)} = \abs{\theta(t) \beta(t)} \geq m/2$, then
  \begin{equation*}
    \theta(t)^2 \geq \abs{\theta(t) \beta(t)} = \abs{w(t)} \geq m/2.
  \end{equation*}
  Hence, if suffices to consider the case $T^{\mr{sig}} > 0$ and bound $T^{\mr{sig}}$.
  Without loss of generality, we assume $z(t) > 0$.
  When $t \leq T^{\mr{esc}} \wedge T^{\mr{sig}}$, we have
  \begin{equation*}
    \dot{\beta(t)} \geq \frac{1}{2} \lambda^{\hf} m,\quad t \leq T^{\mr{esc}} \wedge T^{\mr{sig}}
  \end{equation*}
  so $T^{\mr{esc}} \wedge T^{\mr{sig}} \leq 2/m$.
  If $T^{\mr{sig}} = T^{\mr{esc}} \wedge T^{\mr{sig}} \leq 2/m$, we already proved the result.
  For the other case, we have
  \begin{equation*}
    \dot{w} = (\theta^2(t) + \beta^2(t)) (z(t) - w(t)) \geq 2 (\theta(t) \beta(t)) \cdot \frac{1}{2} m = w(t) m,\quad
    t \in [T^{\mr{esc}}, T^{\mr{sig}}].
  \end{equation*}
  Combining with $w(T^{\mr{esc}}) \geq \lambda$, we conclude that
  \begin{equation*}
    T^{\mr{sig}} - T^{\mr{esc}} \leq \frac{1}{m} \log \frac{m}{2\lambda}.
  \end{equation*}
\end{proof}

\subsection{Proof of \cref{thm:DiagSparseSeq_Monotonic}}\label{subsec:Proof_DiagonalSparseSeqMonotonic}
\label{subsec:OpSparseSeq}
Since $\ep_j \sim N(0,1/n)$, with probability at least $1- C n^{-2}$, we have
\begin{equation*}
  \abs{\ep_j} \lesssim \sqrt{\frac{\log d}{n}},\quad \forall j \geq 1.
\end{equation*}
Let us denote by $S$ the signal components.

\paragraph{Monotonicity.}
Let us apply \cref{lem:FEM_Monotonic} to prove the monotonicity of the feature error measure.
We set $N = N_1 = [d]$ so it suffices to prove condition (1).
Since signal components are lower bounded, so $\caI_{\mr{s}}(n^{-1}) = S$ and $\caI_{\mr{n}}(n^{-1}) = R$.
For $j \in S$,
\begin{equation}
  \label{eq:Proof_HighDimLinearRegression1}
  \abs{z_j} \geq \abs{w^*_j} - \abs{\ep_j} \geq c,
\end{equation}
while $\abs{z_j} = \abs{\ep_j} \lesssim \sqrt{(\log d)/n}$ for $j \notin S$.
By \cref{lem:TwoLayerEq_Comparison}, the condition (1) in \cref{lem:FEM_Monotonic} are satisfied.



\paragraph{Final feature error measure.}
For the initial feature error measure, we have
\begin{equation*}
  \caE(\delta, n^{-1};\bm{\lambda}(0)) = \# \dk{j \in [d] :  \alpha^2 \geq \delta}  /n + \sum_{j \in [d]} (w_j^*)^2 \bm{1} \dk{ \alpha^2 < \delta}
  = \frac{d}{n} \bm{1} \dk{\alpha^2 < \delta} + \norm{w^*}^2_2 \bm{1} \dk{\alpha^2 \geq \delta}.
\end{equation*}
Let us consider the feature error measure at time $t = t_* \asymp \log n$.
For $j \in S$, using \cref{lem:TwoLayerEq_SignalLowerBound} with \cref{eq:Proof_HighDimLinearRegression1}, we have
\begin{equation*}
  \theta_j(t_*)^2 \geq \frac{1}{4} \abs{w^*_j} \geq c.
\end{equation*}
For $j \in R$, using \cref{lem:TwoLayerEq_NoiseUpperBound}, we have
\begin{equation*}
  \theta_j(t_*) \leq \sqrt{2} \alpha \lesssim d^{-1/2}.
\end{equation*}
Consequently, taking $\delta^*$ such that $\delta^* \leq c$ and $\delta^* \gtrsim d^{-1/2}$, we have
\begin{align*}
  \caE^*(n^{-1};\bm{\lambda}(t_*))
  & \leq  \caE^*(\delta^*, n^{-1};\bm{\lambda}(t_*)) \\
  & = \# \dk{j \in [d] :  \theta_j(t_*)^2 \geq \delta^*} n^{-1} + \sum_{j \in [d]} (w_j^*)^2 \bm{1} \dk{ \theta_j(t_*)^2 < \delta^*} \\
  & =   \# \dk{j \in [d] :  \theta_j(t_*)^2 \geq \delta^*} n^{-1}  \\
  & = \frac{s^*}{n}.
\end{align*}

\subsection{Over-parameterization under Sequence Model}\label{subsec:OpSeq}

Let us consider the two-layer diagonal adaptive kernel method under sequence model that $z_j = f^*_j + \ep_j$, where $\ep_j \stackrel{i.i.d.}{\sim} N(0,1/n)$.
The explicit form of the gradient flow equation is given by
\begin{equation*}
  \left\{
    \begin{aligned}
      \dot{\theta}_j &= -\nabla_{\theta_j} \caL_n = \beta_j(t) \zk{z_j -f_j(t)},\quad \theta_j(0) = \lambda_j^{\hf}, \\
      \dot{\beta}_j &=-\nabla_{\beta_j} \caL_n = \theta_j(t) \zk{z_j - f_j(t)},\quad \beta_j(0) = 0,
    \end{aligned}
  \right.
\end{equation*}
where $f_j(t) = \theta_j(t) \beta_j(t)$.
This equation aligns with the one-dimensional gradient flow equation \cref{eq:TwoLayerEq} with $z(t) \equiv z = f^*_j + \ep_j$ and $w(t) = f_j(t)$.

\subsubsection{Proof of \cref{thm:DiagonalSeq_Monotonic}}\label{subsubsec:Proof_DiagonalKernelMonotonic}

At the beginning, since $\ep_j \sim N(0,1/n)$, with probability at least $1- C n^{-2}$, we have
\begin{equation*}
  \abs{\ep_j} \lesssim \sqrt{\frac{\log(jn)}{n}},\quad \forall j \geq 1.
\end{equation*}

\noindent \textbf{The monotonicity.}
First, we prove that if $\caE(n^{-1}; \bm{\lambda}(t)) \leq \caE(n^{-1}; \bm{\lambda}(0))$, then
\begin{equation}
  \label{eq:Proof_DiagonalKernelMonotonic1}
  \delta^*(n^{-1};\bm{\lambda}(t)) \geq \delta_0 = C n^{-\frac{q \gamma}{p+q}}.
\end{equation}
We start with computing the initial error measure as
\begin{equation*}
  \caE(\delta,n^{-1}; \bm{\lambda})
  = \# \dk{j \in N :  \lambda_j \geq \delta} n^{-1} + \sum_{j \in N} (f_j^*)^2 \bm{1} \dk{ \lambda_j < \delta}
\end{equation*}
For the first term, using $\lambda_j \asymp j^{-\gamma}$, we have
\begin{equation*}
  \# \dk{j \in N :  \lambda_j \geq \delta} n^{-1} \asymp \delta^{-1/\gamma} n^{-1}.
\end{equation*}
For the second term, we use $f_{j(\ell)}^* \asymp \ell^{-(p+1)/2}$, $j(\ell) \asymp \ell^q$ to find
\begin{equation*}
  \sum_{j \in N} (f_j^*)^2 \bm{1} \dk{ \lambda_j < \delta} \asymp \delta^{p/(q \gamma)}.
\end{equation*}
Balancing the two terms, we find that
\begin{equation*}
  \delta^*(n^{-1};\bm{\lambda}) \asymp n^{-\frac{q \gamma}{p+q}},\qquad
  \caE^*(n^{-1};\bm{\lambda}) \asymp n^{-\frac{p}{p+q}}.
\end{equation*}
Now, since $\lambda_j(t)$ is monotonically increasing, we have
\begin{equation*}
  \caE(\delta,n^{-1}; \bm{\lambda}(t))
  \geq \# \dk{j \in N :  \theta_j(t)^2 \geq \delta} n^{-1}
  \geq \# \dk{j \in N :  \lambda_j \geq \delta} n^{-1} \geq C \delta^{-1/\gamma} n^{-1}.
\end{equation*}
Consequently, if $\caE(n^{-1}; \bm{\lambda}(t)) \leq \caE(n^{-1}; \bm{\lambda}(0))$, we have
$\caE^*(\delta,n^{-1}; \bm{\lambda}(t)) \geq c n^{-\frac{p}{p+q}}$ and thus
\begin{equation*}
  \delta^*(n^{-1};\bm{\lambda}(t)) \geq C n^{-\frac{q \gamma}{p+q}}.
\end{equation*}

Now, for a time interval $[0,T]$ such that \cref{eq:Proof_DiagonalKernelMonotonic1} holds
we will apply \cref{lem:FEM_Monotonic} and verify the conditions (1) and (2).
Let us take $L \asymp (n/\log n)^{\frac{q}{p+1}}$ and let $N_1 = \dk{j < L}$ and $N_2 = \dk{j \geq L}$.
For each $j \geq L$, we have
\begin{equation*}
  \abs{z_j} \leq \abs{f_j^*} + \abs{\ep_j} \lesssim  L^{-\frac{p+1}{2q}} + \sqrt{\frac{\log(jn)}{n}}
  \lesssim \sqrt{\frac{\log(jn)}{n}}
\end{equation*}
Using \cref{lem:TwoLayerEq_NoiseUpperBound}, we find that for $t \leq c \sqrt{n/(\log n)}$,
\begin{equation*}
  \lambda_j(t) = \theta_j(t)^2 \leq \lambda_j \zk{1+ \exp(\sqrt{2} t \abs{z})}^2
  \leq  C \lambda_j \exp(c \sqrt{\log j}) < \delta_0 = C n^{-\frac{q \gamma}{p+q}},\quad j \geq L,
\end{equation*}
since $L \asymp (n/\log n)^{\frac{q}{p+1}} \gg n^{\frac{q}{p+q}}$ by $q > 1$.
Therefore, we have verified condition (2) in \cref{lem:FEM_Monotonic}.
On the other hand, since $L \lesssim n^{\frac{q}{p+1}}$, we have
\begin{equation*}
  \abs{\ep_j} \lesssim \sqrt{\frac{C \log(jn)}{n}} \leq C \sqrt{\frac{\log n}{n}},\quad j \leq L.
\end{equation*}
Now, for $j < L$ such that $j \in \caI_{\mr{s}}(n^{-1})$, we have
\begin{equation*}
  \abs{f_j^*} \gtrsim L^{-\frac{p+1}{2q}} \gtrsim \sqrt{\frac{\log n}{n}}.
\end{equation*}
Therefore, taking the constant factor in $L$ small enough, we can find that
\begin{equation*}
  \abs{z_j} \geq \abs{f_j^*} - \abs{\ep_j} \geq \frac{1}{2} \abs{f_j^*} > C \sqrt{\frac{\log n}{n}}.
\end{equation*}
In the meantime, for $k < L$ and $k \in \caI_{\mr{n}}(n^{-1})$, we have
$\abs{z_k} = \abs{\ep_k} \leq C \sqrt{\frac{\log n}{n}}$.
Consequently, using \cref{lem:TwoLayerEq_Comparison}, we prove condition (1) in \cref{lem:FEM_Monotonic}.

Finally, let us show that the time interval $[0,T]$ can actually cover $T = c \sqrt{n/(\log n)}$ using a continuity argument.
Suppose that $\caE^*(n^{-1};\bm\lambda(t))$ has a jump at time $t_0$, then, it can only increase by at most $\epsilon^2$
(by the continuity of $\bm{\lambda}(t)$ and that $\lambda_j(t)$, $\lambda_{j'}(t)$ do not coincide with probability one).
Therefore, we still have $\caE^*(\delta,n^{-1}; \bm{\lambda}(t)) \lesssim n^{-\frac{p}{p+q}} $ and that $\delta^*(n^{-1};\bm{\lambda}(t)) \geq C n^{-\frac{q \gamma}{p+q}}$,
where the constant factor may increase.
Nevertheless, the second part of the argument still holds with the modified constant so that no up-crossing can happen at time $t_0$,
so $\caE^*(n^{-1};\bm\lambda(t))$ can not increase at time $t_0$.

\paragraph{The feature error measure}
We have already shown in the previous part that $\caE^*(n^{-1};\bm{\lambda}) \asymp n^{-\frac{p}{p+q}}$.
For $t = t_*\asymp \sqrt{n/(\log n)}$, let us take $\delta_* = C n^{-\hf}$.
We first consider $\caE_{\text{V}}(\delta^*,n^{-1};\bm{\lambda})$.
Ss discussed before, we have $\lambda_j(t_*) < \delta_*$ for $j \geq L$.
For $j < L$ such that $\abs{f^*_j} \lesssim n^{-\hf} (\log n)^{\hf}$, we apply \cref{lem:TwoLayerEq_NoiseUpperBound} to find
\begin{equation*}
  \lambda_j(t) = \theta_j(t)^2 \leq 2 \lambda_j.
\end{equation*}
Consequently,
\begin{align*}
  \# \dk{j \in N :  \theta_j(t_*)^2 \geq \delta_*} n^{-1}
  & = \# \dk{j < L :  \theta_j(t_*)^2 \geq \delta_*} n^{-1} \\
  &\leq n^{-1} \zk{ \#\dk{j < L : \lambda_j \geq \delta_*/2}
  +\# \dk{j < L : \abs{f^*_j} \gtrsim n^{-\hf} (\log n)^{\hf} }  } \\
  &\lesssim \delta_*^{-1/\gamma} n^{-1} +  (n \log n)^{\frac{1}{p+1}} n^{-1} \\
  &\lesssim  n^{-(1-1/(2\gamma))}  + n^{-\frac{p}{p+1}} (\log n)^{-\frac{1}{p+1}},
\end{align*}
Now, we consider $\caE_{\text{B}}(\delta^*,n^{-1};\bm{\lambda})$.
For $j$ such that $\abs{f^*_j} \gtrsim n^{-\hf} (\log n)^{3/2}$ (which implies that $j < L$), we apply \cref{lem:TwoLayerEq_SignalLowerBound} to find
\begin{equation*}
  \theta_j(t)^2 \geq \frac{1}{2} \abs{z_j} \geq \frac{1}{4} \abs{f_j^*} \gtrsim \sqrt {\frac{\log n}{n}}.
\end{equation*}
Consequently,
\begin{equation*}
  \sum_{j \in N} (f_j^*)^2 \bm{1} \dk{ \lambda_j < \delta_*}
  \leq \sum_{j \in N} (f_j^*)^2 \bm{1} \dk{ \abs{f^*_j} \gtrsim n^{\hf} (\log n)^{3/2} }
  \lesssim n^{-\frac{p}{p+1}} (\log n)^{\frac{3}{2} \frac{p}{p+1}}
  \lesssim n^{-\frac{p}{p+1}} (\log n)^{\frac{2p}{p+1}}.
\end{equation*}
Therefore, we have
\begin{equation*}
  \caE^*(n^{-1};\bm{\lambda}(t_*)) \leq \caE(\delta_*,n^{-1};\bm{\lambda}(t_*))
  \lesssim n^{-(1-1/(2\gamma))} + n^{-\frac{p}{p+1}} (\log n)^{\frac{2p}{p+1}}.
\end{equation*}

\subsection{Deeper Over-parameterization under Sequence Model}

In this subsection, let us consider deeper over-parameterization under sequence model.
We recall that the gradient flow dynamics are given by
\begin{align}
  \label{eq:MultiLayerEq}
  \left\{
    \begin{aligned}
      \dot{\beta}_j(t) &= - \nabla_{\beta_j} \caL_n =  \theta_j b_j^D (z_j - f_j),\quad \beta_j(0) = 0; \\
      \dot{\theta}_j(t) &= - \nabla_{\theta_j} \caL_n =  b_j^D \beta_j (z_j - f_j),\quad \theta_j(0) = \lambda_j^{\hf}; \\
      \dot{b}_j(t) &= - \nabla_{b_j} \caL_n =  D \theta_j b_j^{D-1} \beta_j  (z_j - f_j),\quad b_j(0) = b_0,
    \end{aligned}
  \right.
\end{align}
where $f_j = \theta_j b_j^D \beta_j$.
In this case, let us denote
\begin{equation*}
  \tilde{\bm{\lambda}} = (\tilde{\lambda}_j)_{j \geq 1},\quad
  \tilde\lambda_j(t) = \xk{\theta_j(t)b_j^D(t)}^2,\quad \tilde\lambda_j(0) = \lambda_j b_0^{2D}
\end{equation*}

This gradient flow dynamics has been studied in the literature~\citep{li2024_ImprovingAdaptivity,li2025_DiagonalOverparameterization},
and we will collect some results here.
First, we always have $\theta_j(t) \geq 0$ and $b_j(t) \geq 0$, while the sign of $\beta_j(t)$ is the same as that of $z_j$.
Moreover, the flow is symmetric in the sense that for the solution for $z_j < 0$ can be obtained by flipping the sign of $\beta_j(t)$.
Second, we can compute that
\begin{equation*}
  \frac{1}{2}\dv{t} \theta_j^2 = \frac{1}{2 D}\dv{t} b_j^2 = \frac{1}{2}\dv{t} \beta_j^2 = f_j(t) (z - f_j(t)),
\end{equation*}
showing that
\begin{equation}
  \label{eq:MultiLayerEq_Conservation}
  \theta_j^2(t) - \beta_j^2(t) = \theta_j(0)^2 - \beta_j(0)^2 = \lambda_j,\qquad
  b_j^2(t) - D \beta_j^2(t) = b_j(0)^2 - D \beta_j(0)^2 = b_0^2.
\end{equation}

Now, let us introduce some lemmas.

\begin{proposition}
  \label{prop:MultiLayerEq_UltBound}
  Consider \cref{eq:MultiLayerEq} and suppose that $b_0 /\sqrt {D} \leq \lambda_j^{\hf} \leq 1$,
  we have
  \begin{equation*}
    \tilde{\lambda}_j^{\hf} = \theta_j(t) b_j^D(t) \leq
    C_D \max(\lambda_j^{\hf} b_0^D,  \abs{z}^{\frac{D+1}{D+2}}, b_0^{-1} \abs{z}).
  \end{equation*}
\end{proposition}
\begin{proof}
  Let us omit the subscript $j$ for ease of notation.
  Following the conservation quantity \cref{eq:MultiLayerEq_Conservation},
  \begin{equation*}
    \min(\lambda^{\hf},\abs{\beta}) \leq \theta \leq \sqrt {2} \max(\lambda^{\hf},\abs{\beta}),\quad
    \min(b_0,\sqrt{D} \abs{\beta}) \leq b \leq \sqrt {2} \max(b_0,\sqrt{D} \abs{\beta}).
  \end{equation*}
  If $\beta \leq \min(\lambda^{\hf}, b_0 /\sqrt {D})$, we immediately have
  \begin{equation*}
    \tilde{\lambda} = \theta b^D \leq 2^{\frac{D+1}{2}} \lambda^{\hf} b_0^D.
  \end{equation*}
  If $\beta \geq \max(\lambda^{\hf}, b_0 /\sqrt {D})$, we have
  \begin{equation*}
    \abs{z} \geq \abs{f} \geq D^{\frac{D}{2}} \abs{\beta}^{D+2} \qimplies \abs{\beta} \leq D^{-\frac{D}{2(D+2)}} \abs{z}^{\frac{1}{D+2}},
  \end{equation*}
  and thus
  \begin{equation*}
    \tilde{\lambda} = \theta b^D \leq 2^{\frac{D+1}{2}} D^{\frac{D}{2}} \abs{\beta}^{D+1}
    \leq 2^{\frac{D+1}{2}} D^{\frac{D}{2(D+2)}} \abs{z}^{\frac{D+1}{D+2}}.
  \end{equation*}
  Otherwise, if $b_0 /\sqrt {D} \leq \beta \leq \lambda^{\hf}$, we use
  \begin{equation*}
    \abs{z} \geq \abs{f} \geq \lambda^{\hf}\cdot (\sqrt {D}\abs{\beta})^D \cdot \abs{\beta} \qimplies
    \abs{\beta} \leq (\lambda^{\hf} D^{\frac{D}{2}})^{-\frac{1}{D+1}} \abs{z}^{\frac{1}{D+1}},
  \end{equation*}
  so
  \begin{equation*}
    \tilde{\lambda} \leq 2^{\frac{D+1}{2}} \lambda^{\frac{1}{2}}  D^{\frac{D}{2}} \abs{\beta}^{D}
    \leq  2^{\frac{D+1}{2}} (\lambda^{\hf} D^{\frac{D}{2}})^{\frac{1}{D+1}} \abs{z}^{\frac{D}{D+1}}
    \leq 2^{\frac{D+1}{2}} D^{\frac{D}{2(D+1)}} \lambda^{\frac{1}{2(D+1)}} \abs{z}^{\frac{D}{D+1}}.
  \end{equation*}
  Furthermore, we also have
  \begin{equation*}
    \abs{z} \geq \lambda^{\hf}\cdot (\sqrt {D}\abs{\beta})^D \cdot \abs{\beta}
    \geq  \lambda^{\hf}b_0^{D+1} D^{-\hf} \qimplies \lambda^{\hf} \leq D^{\hf} b_0^{-(D+1)}\abs{z}.
  \end{equation*}
  Plugging this into the previous inequality, we find that
  \begin{equation*}
    \tilde{\lambda} \leq 2^{\frac{D+1}{2}} D^{\hf}  b_0^{-1} \abs{z}.
  \end{equation*}
\end{proof}

\begin{lemma}[Comparison]
  \label{lem:MultiLayerEq_Comparison}
  Consider \cref{eq:MultiLayerEq} and suppose that  $\abs{z_j} \geq \abs{z_k}$.
  Then, if $\theta_j(t_0) \geq \theta_k(t_0)$ for some $t_0 \geq 0$, we have $\theta_j(t) \geq \theta_k(t)$ for all $t \geq t_0$.
\end{lemma}
\begin{proof}
  Similar to the proof of \cref{lem:TwoLayerEq_Comparison} and we note that the initialization of $b_j(0) = b_0$ is the same for all $j$.
\end{proof}

\begin{lemma}[Noise case]
  \label{lem:MultiLayerEq_NoiseCase}
  For the gradient flow \cref{eq:MultiLayerEq},
  suppose that $\lambda_j^{\hf} \leq b_0 / \sqrt{D}$.
  Denoting ${T}_j^{(1)} = \xk{2^{\frac{D+1}{2}} b_0^D \abs{z_j}}^{-1}$,
  we have
  \begin{equation}
    \label{eq:MultiLayerEq_UpperBound_Linear}
    \theta_j(t) b_j^D(t) \leq 2^{\frac{D+1}{2}} \lambda_j^{\hf} b_0^D, \qq{for} t \leq {T}_j^{(1)}.
  \end{equation}
  and
  \begin{equation}
    \label{eq:MultiLayerEq_UpperBound_Exp}
    \theta_j(t) b_j^D(t) \leq
    2^{\frac{D+1}{2}} \lambda_j^{\hf} b_0^D \exp( 2^{\frac{D+1}{2}} b_0^D \abs{z_j} (t- {T}_j^{(1)})^+),
    \qq{for} t \leq \left( 1 + \log \frac{b_0}{\sqrt {D} \lambda^{\hf}} \right) {T}_j^{(1)}.
  \end{equation}
\end{lemma}
\begin{proof}
  This is a direct consequence of the proof of Lemma D.2 in \citet{li2024_ImprovingAdaptivity}.
\end{proof}

\begin{lemma}[Signal case]
  \label{lem:MultiLayerEq_SignalCase}
  For the gradient flow \cref{eq:MultiLayerEq}, denote
  \begin{equation*}
    T^{\mr{eig}}_j = \inf\dk{t \geq 0: \theta_j(t) b_j^D(t) \geq \abs{z_j}^{\frac{D+1}{D+2}}}.
  \end{equation*}
  We have
  \begin{itemize}
    \item If $\lambda_j^{\hf} \leq b_0/\sqrt {D}$, then
    \begin{align}
      \label{eq:MultiLayerEq_SignalTime_Case1}
      T^{\mr{eig}}_j \leq 2 (b_0^D \abs{z_j})^{-1} \zk{1+ \log^+ \frac{(D^{-\frac{D}{2}} \abs{z_j}/2 )^{\frac{1}{D+2}}}{\lambda_j^{\hf}} },
    \end{align}
    \item If $\lambda_j^{\hf} \geq b_0/\sqrt {D}$, then
    \begin{align}
      \label{eq:MultiLayerEq_SignalTime_Case2}
      T^{\mr{eig}}_j \leq 2 \xk{\sqrt{D} \lambda_j^{\hf} b_0^{D-1} \abs{z_j}}^{-1} \left( 1 + R_j \right),
    \end{align}
    where
    \begin{align*}
      R_j=
      \begin{cases}
        \log^+ \frac{(D\abs{z_j}/2)^{\frac{1}{D+2}}}{b_0}, & D = 1, \\
        \frac{1}{D-1}, & D > 1.
      \end{cases}
    \end{align*}
  \end{itemize}
\end{lemma}
\begin{proof}
  Let us define $T^{\mr{sig}}_j = \inf\dk{t \geq 0: \abs{f_j(t)} \geq \abs{z_j}/2}.$
  Using the conservation quantity, we find that $\theta_j(t) \geq \abs{\beta_j(t)}$ and $b_j(t) \geq \sqrt{D} \abs{\beta_j(t)}$,
  so
  \begin{equation*}
    \abs{f_j(t)} = \abs{\theta_j(t) b_j^D(t) \beta_j(t)} \leq \abs{\theta_j(t) b_j^D(t)}^{\frac{D+2}{D+1}}
  \end{equation*}
  and thus $T^{\mr{eig}}_j  \leq T^{\mr{sig}}_j$.
  Now the rest follows from Lemma D.3 in \citet{li2024_ImprovingAdaptivity}.
\end{proof}

\subsubsection{Proof of \cref{thm:DiagonalKernelMonotonic_Multilayer}}
The proof is similar to the proof for the two-layer case in \cref{subsubsec:Proof_DiagonalKernelMonotonic},
but we have to consider in addition the interaction of the $b_j(t)$ terms.
We recall that $b_0 = b_0(n) \asymp n^{-\frac{1}{2(D+2)}}$ and $t \leq t_* \asymp b_0^{-D} \sqrt {n / (\log n)} \asymp n^{\frac{D+1}{D+2}} / \sqrt {\log n}$.
Let us define
\begin{equation*}
  J = \min \dk{j \geq 1: \lambda_j^{\hf} \leq b_0/\sqrt{D}} \asymp b_0^{-2/\gamma} \asymp n^{\frac{1}{(D+2)\gamma}}.
\end{equation*}
Following the proof of \cref{subsubsec:Proof_DiagonalKernelMonotonic}, we can compute that
\begin{equation*}
  \delta^*(n^{-1};\bm{\tilde\lambda}) \geq \delta_0 \asymp b_0^{2D} n^{-\frac{q \gamma}{p+q}} \asymp n^{-\frac{D}{D+2}} n^{-\frac{q \gamma}{p+q}}
\end{equation*}
Being the same as the proof in \cref{subsubsec:Proof_DiagonalKernelMonotonic},
we consider $j < L$ and $j \geq L$ separately.
For $j < L$, we can still apply the comparison principle \cref{lem:MultiLayerEq_Comparison} to show that the up-crossing can not happen.

Let us now consider the case $j \geq L$ and prove that $\tilde{\lambda}_j(t) < \delta_0$.
For $j \geq \max(L,J)$, using $\abs{z_j} \lesssim \sqrt{\log(jn)/n}$ and $t \leq t_*$, we find that
\begin{equation*}
  b_0^{D} \abs{z_j} t \lesssim \sqrt {\frac{\log (jn)}{\log n}} \lesssim 1 + \sqrt{\log j/\log n}.
\end{equation*}
Then, we apply \cref{eq:MultiLayerEq_UpperBound_Exp} to get
\begin{equation*}
  \tilde{\lambda}_j(t) \lesssim  b_0^{2D} \lambda_j \exp( C b_0^D \abs{z_j} t)
  \leq b_0^{2D} \lambda_j \exp( C \sqrt {\log j})
  < \delta_0 = b_0^{2D} n^{-\frac{q \gamma}{p+q}},
\end{equation*}
but we have to verify the condition in \cref{eq:MultiLayerEq_UpperBound_Exp}.
Let us be more detailed here.
Since $j \geq J$, we can write
\begin{equation*}
  \lambda_j^{\hf} = \kappa b_0/\sqrt {D}, \quad \kappa \leq 1,\quad
  \log j \lesssim \log C (\kappa b_0/\sqrt {D})^{-2/\gamma} \lesssim 1 + \log \kappa^{-1} + \log n.
\end{equation*}
On one hand, we have
\begin{equation*}
  t / {{T}_j^{(1)}} = C b_0^{D} \abs{z_j} t \lesssim 1 + \sqrt{\log j / \log n} \lesssim 1 + \sqrt {\log \kappa^{-1} / \log n},
\end{equation*}
while on the other hand, we have
\begin{equation*}
  1 + \log \frac{b_0}{\sqrt {D} \lambda_j^{\hf}} = 1 + \log \kappa^{-1}
  \gtrsim 1 + \sqrt {\log \kappa^{-1} / \log n}.
\end{equation*}
Therefore, taking the constant factor in $t$ small, this condition is satisfied.

Now, it remains to consider the case that $J > L$ and $j \in [L,J]$.
Applying \cref{prop:MultiLayerEq_UltBound}, we find
\begin{equation*}
  \tilde{\lambda}_j(t)
  \lesssim \max(\lambda_j^{\hf} b_0^D,  \abs{z_j}^{\frac{D+1}{D+2}}, b_0^{-1} \abs{z_j})^2.
\end{equation*}
The first term already satisfies $\lambda_j^{\hf} b_0^D \ll \delta_0$ as considered before.
To control the second and the third term, we use $J > L$ to get
\begin{equation*}
  n^{\frac{1}{(D+2)\gamma}} \gtrsim (n/\log n)^{\frac{q}{p+1}},\qimplies
  n^{\frac{q \gamma}{p+1}} \lesssim (\log n)^{\frac{q\gamma}{p+1}} n^{\frac{1}{D+2}},
\end{equation*}
so using $q > 1$, there is some $s > 0$ such that
\begin{equation*}
  \delta_0 = b_0^{2D} n^{-\frac{q \gamma}{p+q}}
  \gtrsim b_0^{2D} n^{-\frac{1}{D+2}} n^s
  \asymp n^{-\frac{D+1}{D+2}} n^s.
\end{equation*}
Returning to the quantities, since $J$ is still polynomial in $n$, we have $\abs{z_j} \lesssim \sqrt{\log (jn)/n} \lesssim \sqrt{(\log n) / n}$,
so
\begin{equation*}
  \abs{z_j}^{\frac{2(D+1)}{D+2}} \lesssim n^{-\frac{D+1}{D+2}} (\log n)^{\frac{D+1}{D+2}} \ll n^{-\frac{D+1}{D+2}} n^s,
\end{equation*}
and
\begin{equation*}
(b_0^{-1} \abs{z_j})
  ^2 \lesssim n^{\frac{1}{D+2}} n^{-1} \log n = n^{-\frac{D+1}{D+2}} \log n \ll n^{-\frac{D+1}{D+2}} n^s.
\end{equation*}
Therefore, we conclude in this case that $\tilde{\lambda}_j(t) < \delta_0$.

Finally, for the feature error measure, we can follow the same argument as in \cref{subsubsec:Proof_DiagonalKernelMonotonic} except that we apply \cref{lem:MultiLayerEq_NoiseCase} and \cref{lem:MultiLayerEq_SignalCase}.

\paragraph{The feature error measure}
Now, let us consider the feature error measure at time $t = t_*$.
Let us choose $\delta_* = C n^{-\frac{D+1}{2(D+2)}}$.
We first consider $\caE_{\text{V}}(\delta^*,n^{-1};\bm{\tilde\lambda})$.
For $j \geq L$, as discussed before, we have $\tilde{\lambda}_j <  \delta_0 \leq \delta_*$.
Moreover, for $j < L$ such that $\abs{f^*_j} \lesssim \sqrt {(\log n)/n}$ and $\lambda_j^{\hf} \leq b_0/\sqrt{D}$, we apply \cref{lem:MultiLayerEq_NoiseCase} to find that $\tilde{\lambda}_j(t_*) \leq C \tilde{\lambda}_j(0)$.
Consequently,
\begin{align*}
  \# \dk{j \in N :  \tilde{\lambda}_j(t_*) \geq \delta_*}
  &\leq \# \dk{j < L :  \abs{f^*_j} \gtrsim (n \log n)^{\hf}} \\
  &\quad + \# \dk{j < L : \lambda_j \geq b_0/\sqrt{D}} \\
  &\quad + \# \dk{j < L : C \tilde{\lambda}_j(0) \geq \delta_*} \\
  & \lesssim n^{\frac{1}{p+1}} + n^{\frac{1}{(D+2)\gamma}}
\end{align*}
Now, for $\caE_{\text{B}}(\delta^*,n^{-1};\bm{\tilde\lambda})$, we use \cref{lem:MultiLayerEq_SignalCase} to find that if $\abs{f^*_j} \gtrsim n^{-\hf} (\log n)^{3/2}$, then
\begin{equation*}
  \tilde{\lambda}_j(t) \geq \abs{f^*_j}^{\frac{D+1}{D+2}} \gtrsim n^{-\frac{D+1}{2(D+2)}} (\log n)^{\frac{3}{2}\frac{D+1}{D+2}} \geq \delta_*,
\end{equation*}
so
\begin{equation*}
  \sum_{j \in N} (f_j^*)^2 \bm{1} \dk{ \tilde{\lambda}_j(t_*) \geq \delta_*}
  \leq \sum_{j \in N} (f_j^*)^2 \bm{1} \dk{ \abs{f^*_j} \gtrsim n^{\hf} (\log n)^{3/2} }
  \lesssim n^{-\frac{p}{p+1}} (\log n)^{\frac{2p}{p+1}}.
\end{equation*}
Combining the two parts, we prove the feature error measure bound.

\section{Proof for Diagonal Adaptive Kernel under Empirical Loss}

\subsection{Over-parameterized linear regression}
\label{subsec:OpLinear}

Let us consider the over-parameterized high-dimensional linear regression.
Let us denote by $S$ the indices of the signals and $R = [d] \backslash S$ the rest of the indices.
It is easy to see that the gradient flow equation can be given explicitly as
\begin{equation*}
  \left\{
  \begin{aligned}
    \dot{\theta} &= \beta(t) \odot  (w^* - w(t) + r(t)),\quad \theta(0) = \alpha \bm{1}, \\
    \dot{\beta} &= \theta(t) \odot (w^* - w(t) + r(t)),\quad \beta(0) = \bm{0},
  \end{aligned}
  \right.
\end{equation*}
where $w(t) = \theta(t) \odot \beta(t)$,
\begin{equation*}
  r(t) = (\hat{\Sigma} - I_d)(w^* - w(t)) + h,\quad
  \hat{\Sigma} = \frac{1}{n} \sum_{i=1}^n x_i x_i^{\T}, \quad
  h = \frac{1}{n} \sum_{i=1}^n \ep_i x_i.
\end{equation*}

\subsubsection{Proof of \cref{thm:DiagLinear_Monotonic}}

First, we bound the perturbation term $r(t)$.
We decompose $r_j$ as
\begin{align*}
  r_j &= (\hat{\Sigma}_{j\cdot } - I_{j \cdot})(w^* - w(t)) + h
  = (\hat{\Sigma}_{j S} - I_{j S}) (w^*_S - w_S(t)) + (\hat{\Sigma}_{j R } - I_{j R}) (w^*_R - w_R(t)) + h_j \\
  &\eqqcolon r_{j,1} + r_{j,2} + h_j.
\end{align*}
For $h_j$, standard sub-Gaussian concentration inequality gives that with probability at least $1 - C d^{-2}$,
\begin{equation*}
  \abs{h_j} \lesssim \sqrt{\frac{\log(d)}{n}},\quad \forall j \in [d].
\end{equation*}
For $r_{j,1}$, sub-exponential concentration also gives that with probability at least $1 - C d^{-2}$,
\begin{equation*}
  \norm{\hat{\Sigma}_{j S} - I_{j S}}_2
  \lesssim \sqrt {\frac{s^* \log d}{n}},\quad \forall j \in [d].
\end{equation*}
Therefore, as $s^*$ is a constant, we have
\begin{equation*}
  \abs{r_{j,1}} \leq \norm{\hat{\Sigma}_{j S} - I_{j S}}_2 \norm{w^*_S - w_S(t)}_2
  \lesssim \sqrt {\frac{s^* \log d}{n}} \norm{w^*_S - w_S(t)}_2.
\end{equation*}
For $r_{j,2}$, we notice that $w^*_R = 0$, so
\begin{equation*}
  \abs{r_{j,2}} = \abs{(\hat{\Sigma}_{j R } - I_{j R}) w_R(t)} \leq \norm{\hat{\Sigma}_{j R } - I_{j R}}_{2} \norm{w_R(t)}_{2}
  \leq \sqrt {d} \norm{\hat{\Sigma}_{j R } - I_{j R}}_{2} \norm{w_R(t)}_{\infty}.
\end{equation*}
Standard concentration inequality also gives that with probability at least $1 - C d^{-2}$,
\begin{equation*}
  \norm{\hat{\Sigma}_{j R} - I_{j R}}_2
  \lesssim \sqrt{\frac{d\log d}{n}},\quad \forall j \in [d].
\end{equation*}
Consequently,
\begin{equation*}
  \abs{r_{j,2}} \lesssim \sqrt{\frac{d^2 \log d}{n}} \norm{w_R(t)}_{\infty}.
\end{equation*}
In the following, we claim that there is a constant $C_0 > 0$ such that for $t \in [0,t_*]$,
\begin{equation}
  \label{eq:Proof_HighDimLinearRegression0}
  \norm{w^*_S - w_S(t)}_2 \leq C_0,\qquad \norm{w_R(t)}_{\infty} \leq C_0 d^{-1}.
\end{equation}
With this claim, the bound of the three terms in $r_j$ sum up to
\begin{equation}
  \abs{r_j(t)} \leq C \sqrt{\frac{\log d}{n}},\quad \forall j \in [d].
\end{equation}

Using the bound of $r_{j}(t)$, we can follow the same lines of proof in \cref{subsec:Proof_DiagonalSparseSeqMonotonic} to show
the monotonicity of the feature error measure and its final value.
Particularly, we use the following fact: for $j \in S$, we have
\begin{equation*}
  \abs{w^*_j + r_j(t)} \geq \abs{w^*_j} - \abs{r_j(t)} \geq \abs{w^*_j} - C \sqrt{\frac{\log(d)}{n}}
  \geq \frac{1}{2} \abs{w^*_j} \geq c
\end{equation*}
while for $j \in R$, we have
\begin{equation*}
  \abs{w^*_j + r_j(t)} \leq \abs{r_j(t)} \leq C \sqrt{\frac{\log(d)}{n}}.
\end{equation*}

\paragraph{Proof of the claim.}
Now, let us prove the claim in \cref{eq:Proof_HighDimLinearRegression0}.
Since $(\log d)(\log n) = 0(n)$, the perturbation term can be written as
\begin{equation*}
  \abs{r_j} \leq C C_0 \sqrt{\frac{\log d}{n}} = \eta C_0,
\end{equation*}
where $\eta > 0$ can be taken such that $\eta \log n$ is arbitrary small.
Now, by taking $C_0 = 2 + \norm{w^*_S}$ being a constant, \cref{eq:Proof_HighDimLinearRegression0} holds with strict inequality when $t=0$,
so we can let
\begin{equation*}
  t_0 = \inf \dk{t \in [0,t_*]: \norm{w^*_S - w_S(t)}_2 = C_0 \qq{or} \norm{w_R(t)}_{\infty} = C_0 d^{-1}} > 0.
\end{equation*}
We will prove by contradiction that $t_0 = t_*$.

On one hand, for $j \in S$,
\begin{equation*}
  \dot{w}_j(t) = (\theta_j^2(t) + \beta_j^2(t)) (w^*_j - w_j(t) + r_j(t)).
\end{equation*}
As long as $\eta \leq \min_{j \in S} \abs{w^*_j} / (2C_0)$, we have
\begin{equation*}
  \abs{r_j(t)} \leq \eta C_0 \leq \frac{1}{2} \abs{w^*_j},
\end{equation*}
so if $\abs{w_j(t_0) - w^*_j} \geq \frac{1}{2} \abs{w^*_j}$, the sign of $\dot{w}_j(t_0)$ must be the same as that of $w^*_j - w_j(t_0)$,
which implies that $\abs{w_j(t_0) - w^*_j}$ must be non-increasing at $t_0$.
In addition, if we have $\abs{w_j(t_0) - w^*_j} < \frac{1}{2} \abs{w^*_j}$ for all $j \in S$, we already have $\norm{w^*_S - w_S(t_0)}_2 \leq \norm{w^*_S}/2 < C_0$.
Therefore, it can not be the case that $\norm{w^*_S - w_S(t_0)}_2 = C_0$.

On the other hand, for $j \in R$, \cref{lem:TwoLayerEq_NoiseUpperBound} shows that
we have
\begin{equation*}
  \abs{w_j(t)} \leq \sqrt{2} \alpha^2 \leq \sqrt{2} d^{-1},\qq{for} t \leq \min(T,t_0),
\end{equation*}
where
\begin{equation*}
  T =  \frac{1}{\sqrt{2} M}, \quad M =  \max_{t \leq t_0} \abs{w^*_j + r_j(t)}
  = \max_{t \leq t_0} \abs{r_j(t)} \leq \eta C_0
\end{equation*}
Consequently, we can choose $\eta$ small enough that $T \geq t_* \asymp \log n$.
Hence, $\abs{w_j(t_0)}\leq \sqrt{2} d^{-1} < C_0 d^{-1}$,
so $\norm{w_R(t_0)}_{\infty} = C_0 d^{-1}$ is also impossible.
Combining the two cases, we find that $t_0 = t_*$, which completes the proof of the claim.

\subsection{Diagonal adaptive kernel} \label{subsec:DiagonalKernel}

Let us consider the general version of the theorem in the following.
To describe the properties of the truth function with the greatest generality, let us introduce the following quantities on the truth coefficients $(f_j^*)_{j \geq 1}$, which is also introduced in \citet{li2025_DiagonalOverparameterization}:
\begin{equation}
  \label{eq:SignalAndResidual}
  \caN(\delta;f^*) \coloneqq \# \dk{j : \abs{f_{j}^*} \geq \delta}, \quad
  \caR(\delta;f^*) \coloneqq \sum_{j = 1}^\infty  (f_j^*)^2 \bm{1}\{\abs{f_j^*} < \delta\}.
\end{equation}
These two quantities measure the number of significant coefficients and the sum of residual terms of the truth function, respectively.
Moreover, they can be viewed as the optimal feature error measure for the truth function, where the weights of the feature map are in line with the truth coefficients.
We note that since $\abs{f_{j}^*}$ may not be decreasing in $j$ (for instance, consider \cref{assu:DiagonalKernel_pq}),
so these two quantities are not simply obtained by partitioning first $L$ terms and the rest.
Under \cref{assu:DiagonalKernel_pq}, we have
$\caN(\delta;f^*) \asymp \delta^{-\frac{2}{p+1}}$ and $\caR(\delta;f^*) \asymp \delta^p$.

We also have to make the following mild assumption on the truth coefficients,
which requires that the span of the significant coefficients is not exponentially large
and the significant coefficients decay fast enough that they are summable.
They are satisfied when \cref{assu:DiagonalKernel_pq} holds.

\begin{assumption}
  \label{assu:SignalSpan}
  There exists constant $B_{\infty}$ such that $\abs{f_j^*} \leq B_{\infty}$ for all $j \geq 1$.
  Moreover, there are constants $\kappa \geq 1, B_{\mr{spn}},s_0, B_{\mr{sig}} > 0$ such that
  \begin{align}
    \label{eq:SignificantSpan}
    \max \dk{j : \abs{\theta_{j}^*} \geq \delta} \leq B_{\mr{spn}} \delta^{-\kappa}, \qq{and}
    \caR(\delta;f^*) \leq B_{\mr{sig}} \delta^{-(1-s_0)}, \quad \forall \delta > 0.
  \end{align}
\end{assumption}

Then, \cref{thm:DiagonalKernelDecay} is a direct consequence of the following two theorems.

\begin{theorem}
  \label{thm:DiagonalKernelDecay_Detailed}
  Consider the diagonal adaptive kernel method in \cref{eq:DiagAdaK_Seq} with the empirical loss $\caL_n$ under \cref{assu:EigenSystem} and \cref{assu:SignalSpan}.
  Let $s > 0$ be an arbitrary small constant.
  Then, we can choose $L \asymp (-\hf+s) \log n$,
  a decreasing sequence $\delta_l = C 2^{-l}$ for $l \leq L$ satisfying $\delta_L \leq n^{-\hf+s}$,
  and times $t_0 = 0 < t_1 < \dots < t_{L} = t_* \lesssim \sqrt{n}$ satisfying $t_l \lesssim \delta_l^{-l} \log n$, such that,
  with probability at least $1 - C n^{-2}$, we have
  \begin{equation}
    \label{eq:DiagonalKernelDecay}
    \caE^*(n^{-1};\Phi_{\bm{\theta}(t)},f^*)
    \leq  \caR( \delta_l;f^*) + \xk{\caN(n^{-\frac{1}{2}};f^*) + n^{\frac{1}{2\gamma}}} n^{-1}, \quad \forall t \in [t_l, t_*],~ \forall l = 0,\dots,L,
  \end{equation}
  and in particular,
  \begin{equation*}
    \caE^*(n^{-1};\Phi_{\bm{\theta}(t_*)},f^*) \leq \caR(n^{-\hf+s};f^*) + \xk{ \caN(n^{-\hf};f^*) + n^{\frac{1}{2\gamma}}} n^{-1}
  \end{equation*}
  Additionally, we notice that the upper bound in \cref{eq:DiagonalKernelDecay} is monotonically decreasing in $l$.
\end{theorem}

\begin{theorem}
  \label{thm:DiagonalKernelDecay_Multilayer_Detailed}
  Consider the diagonal adaptive kernel method in \cref{eq:DiagAdaK_Seq_Multilayer} with the empirical loss $\caL_n$ under \cref{assu:EigenSystem} and \cref{assu:SignalSpan},
  where we choose $b_0 \asymp n^{-\frac{1}{2(D+2)}}$.
  Let $s > 0$ be an arbitrary small constant.
  Denote $q = 2^{\frac{2(D+1)}{D+2}}$.
  Then, we can choose $L \asymp (-\hf+s) \log n$,
  a decreasing sequence $\delta_l = C q^{-l}$ for $l \leq L$ satisfying $\delta_L \leq n^{-\hf+s}$,
  and times $t_0 = 0 < t_1 < \dots < t_{L} = t_* \lesssim n^{\frac{D+1}{D+2}}$ satisfying $t_l \lesssim \delta_l^{-l} \log n$, such that,
  with probability at least $1 - C n^{-2}$, we have
  \begin{equation}
    \label{eq:DiagonalKernelDecay_Multilayer}
    \caE^*(n^{-1};\Phi_{\bm{\theta}(t),\bm{b}(t)},f^*)
    \leq  \caR( \delta_l;f^*) + \xk{\caN(n^{-\frac{1}{2}};f^*) + n^{\frac{1+s}{(D+2)\gamma}}} n^{-1}, \quad
    \forall t \in [t_l, t_*],~ \forall l = 0,\dots,L.
  \end{equation}
  In particular,
  \begin{equation*}
    \caE^*(n^{-1};\Phi_{\bm{\theta}(t_*),\bm{b}(t_*)},f^*) \leq \caR(n^{-\hf+s};f^*) + \xk{ \caN(n^{-\hf};f^*) + n^{\frac{1+s}{(D+2)\gamma}}} n^{-1}.
  \end{equation*}
\end{theorem}


\subsubsection{Proof of \cref{thm:DiagonalKernelDecay_Detailed}} \label{subsubsec:Proof_DiagonalKernelDecay}
The proof relies on investigating the proof of Theorem 1 in \citet{li2025_DiagonalOverparameterization}.
Denote
\begin{equation*}
  S = S_1 \cup S_2 = \dk{j \geq 1 : \abs{\theta^*_j} \geq n^{-1/2}\sqrt {\ln n}} \cup
  \dk{j\geq 1: \lambda_j \geq n^{-1/2}}
\end{equation*}
and $R = S^\complement$ as in \citet{li2025_DiagonalOverparameterization}.
Let $ \nu_1 = C n^{-1/2+s}$ and $L = \lfloor \log_2 (B_{\infty}/\nu_1) \rfloor$, where $s > 0$ is an arbitrary small constant.
The shrinkage dynamics in the proof of  \citet[Theorem 1]{li2025_DiagonalOverparameterization} shows that,
letting $\delta_i = 2^{-i} B_{\infty}$, there is a sequence of increasing times $t_i$ such that, $t_i \lesssim \delta_i^{-1} \log n$ and for $i = 0,\dots,L$,
\begin{equation}
  \label{eq:Proof_DiagonalKernelDecay_0}
  \norm{\bm{f}^*_S - \bm{f}_S(t)}_{\infty} \leq \delta_{i+1},\quad \forall t \geq t_{i+1}.
\end{equation}
Also, we have $\delta_{I+1} \lesssim n^{-1/2+s}$ and $t_* = t_{L} \lesssim \sqrt {n}$.
Therefore, we have
\begin{equation}
  \label{eq:Proof_DiagonalKernelDecay_1}
  \theta_k(t_i)^2 \geq \abs{f_k(t_i)} \geq \abs{f_k^*} - \norm{\bm{f}^*_S - \bm{f}_S(t_i)}_{\infty} \geq \abs{f_k^*} - \delta_i
  \geq \delta_i,\quad \forall k, \abs{f_k^*} \geq 2\delta_{i}.
\end{equation}
On the other hand, the noise terms are bounded by
\begin{equation*}
  \abs{f_k(t)} \leq 2 \lambda_k \exp(C \sqrt {\ln n + \ln k}) \qimplies
  \abs{\theta_k(t)}^2 \leq C \lambda_k \exp(C \sqrt {\ln n + \ln k})
  \leq n^{-1/2+s},\quad
  \forall k \in R,~t \leq t_*,
\end{equation*}
where $s > 0$ is an arbitrary small constant.

Now let us consider the feature error measure.
We have
\begin{equation*}
  \caE^*(n^{-1};\bm{\lambda}(t), f^*) \leq \caE(\delta_{i},n^{-1};\bm{\lambda}(t))
   = \# \dk{k \in N : \theta_k(t)^2 \geq \delta_i} n^{-1} + \sum_{k \in N} (f_k^*)^2 \bm{1} \dk{ \theta_k(t)^2 < \delta_i}.
\end{equation*}
For the first term, using the control of the noise terms and that $\delta_i \geq C n^{-\hf+s}$,
we find that for $t \leq t_*$,
\begin{equation*}
  \# \dk{k \in N : \theta_k(t)^2 \geq \delta_i} \leq
  \abs{S} \leq \caN(n^{-\hf};f^*) + n^{\frac{1}{2\gamma}}
\end{equation*}
For the second term, we use \cref{eq:Proof_DiagonalKernelDecay_1} to get
\begin{equation*}
  \sum_{k \in N} (f_k^*)^2 \bm{1} \dk{ \theta_k(t)^2 < \delta_i}
  \leq \sum_{k \in N} (f_k^*)^2 \bm{1} \dk{ \abs{f_k^*} < 2\delta_i} = \caR(2\delta_i;f^*),\quad \forall t \in [t_i,t_*].
\end{equation*}
Combining the two estimates, we obtain the bound on $\caE^*(n^{-1};\bm{\lambda}(t_i), f^*)$.
Particularly, using $\delta_L \lesssim n^{-1/2+s}$, we have
\begin{equation*}
  \caE^*(n^{-1};\bm{\lambda}(t_*), f^*) \leq \caR(n^{-\hf+s};f^*) + \xk{ \caN(n^{-\hf+s};f^*) + n^{\frac{1+s}{2\gamma}}} n^{-1}.
\end{equation*}

\subsubsection{Proof of \cref{thm:DiagonalKernelDecay_Multilayer_Detailed}}

The proof is similar to the proof for the two-layer case in \cref{subsubsec:Proof_DiagonalKernelDecay}
and we follow the proof of Theorem 2 in \citet{li2025_DiagonalOverparameterization} here.
We denote $\tilde{\lambda}_j(t) = (\theta_j(t) b_j^D(t))^2$.
In this case, for some small $s' > 0$, we denote
\begin{equation*}
  S = S_1 \cup S_2 = \dk{j \geq 1 : \abs{f^*_j} \geq n^{-1/2}\sqrt {\ln n}} \cup
  \dk{j\geq 1: \lambda_j \geq n^{-\frac{1+s'}{D+2}}}.
\end{equation*}
Let us still define $\nu_1 = C n^{-\frac{1}{2}+s_1}$, $L = \lfloor \log_2 (B_{\infty}/\nu_1) \rfloor$
and $\delta_i = 2^{-i} B_{\infty}$.
The shrinkage dynamics shows that we have a sequence of increasing times $t_i$ such that \cref{eq:Proof_DiagonalKernelDecay_0} holds.
Moreover, we have
\begin{equation*}
  t_i \lesssim \sum_{j \leq i}\zk{\delta_j^{-\frac{2D+2}{D+2}} + \delta_j^{-1} b_0^{-D} \log n }
  \lesssim \delta_i^{-\frac{2D+2}{D+2}} + \delta_i^{-1} b_0^{-D} \log n
  \lesssim \delta_i^{-1} b_0^{-D} \log n \asymp \delta_i^{-1} n^{\frac{D}{2(D+2)}} \log n,
\end{equation*}
where the last inequality follows from that $\delta_i \gtrsim  n^{-\frac{1}{2}+s_1}$ so the second term dominates.
Regarding the multilayer case, we have $t_L \lesssim n^{\frac{D+1}{D+2}}$ and
\begin{equation*}
  \tilde{\lambda}_k(t_i) \geq \abs{f_k(t_i)}^{\frac{2(D+1)}{D+2}}
  \geq \delta_i^{\frac{2(D+1)}{D+2}},\quad \forall k, \abs{f_k^*} \geq 2\delta_{i}.
\end{equation*}
The noise terms are bounded by
\begin{equation*}
  \abs{\tilde{\lambda}_k(t)} \lesssim \lambda_k b_0^{2D} \exp(E \sqrt {\ln n + \ln k})
  \lesssim n^{-\frac{D+1+s'}{D+2}}.
\end{equation*}
Consequently, the control of the feature error measure follows the same argument as in \cref{subsubsec:Proof_DiagonalKernelDecay}.

\section{Hermite Polynomials and Gaussian Distribution}

In this section, we collect some useful properties of the Hermite polynomials and also the Gaussian distribution.
Let us denote by $\gamma_d$ the standard Gaussian measure on $\R^d$.
We denote by $H_r$ the normalized Hermite polynomials with respect to the standard Gaussian measure $\gamma_1$ such that $\E_{x \sim N(0,1)} H_r(x) H_s(x) = \delta_{rs}$.
It is known that the Hermite polynomials $(H_r)_{r \geq 0}$ form an orthonormal basis of $L^2(\R,\gamma_1)$.
Moreover, the generating function of the Hermite polynomials is given by
\begin{equation}
  \label{eq:Hermite_Generating}
  \exp(xt - \frac{t^2}{2}) = \sum_{r=0}^\infty \frac{H_r(x)}{\sqrt{r!}} t^r.
\end{equation}

For a multi-index $\mm = (m_1,\ldots,m_d)$, we define the tensorized Hermite polynomial $H_{\mm} = \prod_{j=1}^d H_{m_j}$.
Then, the set of tensorized Hermite polynomials $\{H_{\mm} : \mm \in \bbN^d\}$ forms an orthonormal basis of $L^2(\R^d,\gamma_d)$.
The generating function of the multi-index Hermite polynomials is given by
\begin{align*}
  \exp(\ang{x,t} - \frac{1}{2} \norm{t}^2) = \sum_{\mm \in \bbN^p} \frac{H_{\mm}(x)}{\sqrt{\mm!}} t^{\mm},
\end{align*}
where we use the convention $\mm! = \prod_{i=1}^p m_i!$ and $t^{\mm} = \prod_{i=1}^p t_i^{m_i}$.
Let $\mm$, $\nn$ be multi-indices in $\bbN^p$.
We denote the multi-index combinatorial by $ \binom{\mm}{\nn} = \prod_{i=1}^p \binom{m_i}{n_i}.$

\begin{lemma}
  \label{lem:Hermite_OrthTrans}
  Let \(x \sim \mathcal{N}(0, I_d)\), and let \(P \in \St(d, p)\) and \(Q \in \St(d, q)\) be Stiefel matrices.
  Let \(\mm \in \mathbb{N}^p\) and \(\nn \in \mathbb{N}^q\) be multi-indices.
  We have
  \begin{align}
    \begin{aligned}
      \E \zk{H_\mm(P^\top x) H_\nn(Q^\top x)} &=
      \sqrt{\mm! \nn!} [\alpha^\mm \beta^\nn] \exp(\alpha^\top R \beta)
      = \delta_{|\mm|, |\nn|}    \sum_{\Upsilon \in \Pi(\mm, \nn)} \frac{\sqrt{\mm! \nn!}}{\Upsilon!} R^\Upsilon,
    \end{aligned}
  \end{align}
  where \(R = P^\top Q \in \mathbb{R}^{p \times q}\), \(\Pi(\mm, \nn) \subset \mathbb{N}^{p \times q}\) is the set of integer matrices \(\Upsilon = [\Upsilon_{i,j}]\) satisfying
  \[
    \sum_{j=1}^q \Upsilon_{i,j} = m_i \quad \text{for all } i, \quad
    \sum_{i=1}^p \Upsilon_{i,j} = n_j \quad \text{for all } j.
  \]
  and \(\Upsilon! = \prod_{i,j} \Upsilon_{i,j}!\), \(R^\Upsilon  = \prod_{i,j} (R_{i,j})^{\Upsilon_{i,j}}\).
\end{lemma}
\begin{proof}

  Let \(u = P^\top x \in \mathbb{R}^p\) and \(v = Q^\top x \in \mathbb{R}^q\). Then \((u, v)\) is jointly Gaussian with mean \(0\) and covariance matrix
  \[
    \Sigma = \begin{pmatrix}
               I_p    & R   \\
               R^\top & I_q
    \end{pmatrix},
    \quad \text{where } R = P^\top Q \in \mathbb{R}^{p \times q}.
  \]

  For the pair \((u, v) \in \mathbb{R}^p \times \mathbb{R}^q\), consider the joint generating function:
  \[
    \exp\big(\langle \alpha, u \rangle + \langle \beta, v \rangle\big)
    = \exp\big(\alpha^\top P^\top x + \beta^\top Q^\top x\big)
    = \exp\big(\langle P \alpha + Q \beta, x \rangle\big).
  \]
  Since \(x \sim \mathcal{N}(0, I_d)\), the expectation of the exponential is:
  \[
    \mathbb{E}\big[\exp(\langle P \alpha + Q \beta, x \rangle)\big]
    = \exp\Big(\frac{1}{2} \|P \alpha + Q \beta\|^2\Big).
  \]
  Expanding the quadratic term:
  \[
    \|P \alpha + Q \beta\|^2
    = \alpha^\top P^\top P \alpha + \beta^\top Q^\top Q \beta + 2 \alpha^\top P^\top Q \beta
    = \|\alpha\|^2 + \|\beta\|^2 + 2 \alpha^\top R \beta,
  \]
  since \(P^\top P = I_p\) and \(Q^\top Q = I_q\).
  Thus,
  \[
    \mathbb{E}\big[\exp(\langle \alpha, u \rangle + \langle \beta, v \rangle)\big]
    = \exp\xk{\frac{1}{2} \|\alpha\|^2 + \frac{1}{2} \|\beta\|^2 + \alpha^\top R \beta}.
  \]

  Expanding the left-hand side:
  \[
    \exp\big(\langle \alpha, u \rangle + \langle \beta, v \rangle\big)
    = \exp(\frac{1}{2}\norm{\alpha}^2 + \frac{1}{2} \norm{\beta}^2  )\sum_{\mm \in \mathbb{N}^p} \frac{\alpha^\mm}{\sqrt{\mm!}} H_\mm(u)
    \sum_{\nn \in \mathbb{N}^q} \frac{\beta^\nn}{\sqrt{\nn!}} H_\nn(v).
  \]
  Taking the expectation:
  \[
    \mathbb{E}\big[\exp(\langle \alpha, u \rangle + \langle \beta, v \rangle)\big]
    = \exp(\frac{1}{2}\norm{\alpha}^2 + \frac{1}{2} \norm{\beta}^2  ) \sum_{\mm \in \mathbb{N}^p} \sum_{\nn \in \mathbb{N}^q} \frac{\alpha^\mm}{\sqrt{\mm!}} \frac{\beta^\nn}{\sqrt{\nn!}} \mathbb{E}\big[H_\mm(u) H_\nn(v)\big].
  \]
  Therefore, we conclude that
  \begin{align*}
    \sum_{\mm \in \mathbb{N}^p} \sum_{\nn \in \mathbb{N}^q} \frac{\alpha^\mm \beta^\nn}{\sqrt{\mm! \nn!}} \mathbb{E}\big[H_\mm(u) H_\nn(v)\big]
    =    \exp(\alpha^\top R \beta)
    =  \sum_{k=0}^\infty \frac{1}{k!} \big(\alpha^\top R \beta\big)^k.
  \end{align*}
  Matching the coefficients yields
  \begin{align*}
    \E \big[H_\mm(u) H_\nn(v)\big]
    = \sqrt{\mm! \nn!} [\alpha^\mm \beta^\nn] \exp(\alpha^\top R \beta).
  \end{align*}

  It remains to prove the last expression.
  As \( \alpha^{\T} R \beta = \sum_{i,j} R_{i,j} \alpha_i  \beta_j\), we have
  \begin{align*}
    \xk{\alpha^{\T} R \beta}^k = \sum_{\Upsilon : \sum_{ij} \Upsilon_{ij} = k}
    \frac{k!}{\prod_{i,j} \Upsilon_{ij}!} \prod_{i,j}  R_{i,j}^{\Upsilon_{ij}} \alpha_i^{\Upsilon_{ij}} \beta_j^{\Upsilon_{ij}},
  \end{align*}
  so
  \begin{align*}
    \exp(\alpha^\top R \beta)
    &=  \sum_{k=0}^\infty \sum_{\Upsilon : \sum_{ij} \Upsilon_{ij} = k} \frac{1}{\prod_{i,j} \Upsilon_{ij}!} \prod_{i,j}  R_{i,j}^{\Upsilon_{ij}} \alpha_i^{\Upsilon_{ij}} \beta_j^{\Upsilon_{ij}} \\
    &=  \sum_{\Upsilon \in \bbN^{p \times q}} \frac{1}{\Upsilon!}  \prod_{i,j} R_{i,j}^{\Upsilon_{ij}} \alpha_i^{\Upsilon_{ij}} \beta_j^{\Upsilon_{ij}}.
  \end{align*}
  Consequently,
  \begin{align*}
    \zk{\alpha^\mm \beta^\nn} \exp(\alpha^\top R \beta)
    = \sum_{\Upsilon \in \bbN^{p \times q}, \sum_j \Upsilon_{ij} = m_i, \sum_i \Upsilon_{ij} = n_j} \frac{1}{\Upsilon!}  \prod_{i,j} R_{i,j}^{\Upsilon_{ij}}.
  \end{align*}
\end{proof}

\begin{corollary}
  \label{cor:Hermite_InnerProd}
  Let \(x \sim N(0, I_d)\) and $u,v \in \bbS^{d-1}$.
  We have
  \begin{align}
    \label{eq:Hermite_InnerProd_1}
    & \E H_m(\ang{u,x}) H_n(\ang{v,x}) = \delta_{m,n} \ang{u,v}^m, \\
    \label{eq:Hermite_InnerProd_2}
    & \E H_{m}(\ang{u,x}) H_{\nn}(x) = \delta_{m, \abs{\nn}} \binom{m}{\nn}^{\hf} u^{\nn}.
  \end{align}
\end{corollary}
\begin{proof}
  The first identity follows easily from \cref{lem:Hermite_OrthTrans} by taking $P = u$ and $Q = v$,
  where $p = q = 1$ and $\alpha, \beta \in \R$.

  For the second identity, let us take $P = u$ and $Q = I_d$ in \cref{lem:Hermite_OrthTrans},
  so $p = 1, q = d$ and $\alpha \in \R$, $\beta \in \R^d$.
  We find that
  \begin{align*}
    \E H_{m}(\ang{u,x}) H_{\nn}(x) &= \sqrt{m ! \nn!} [\alpha^{m} \beta^{\nn}] \exp(\alpha^\top u^\T \beta) \\
    &=  \sqrt{m ! \nn!}  [\alpha^{m} \beta^{\nn}] \sum_{k \geq 0} \frac{1}{k!} (\alpha u^\T \beta)^k \\
    &=  \sqrt{m ! \nn!} \delta_{m, \abs{\nn}} [\alpha^{m} \beta^{\nn}] \frac{\alpha^m (u^\T \beta)^m }{m!} \\
    &=  \delta_{m, \abs{\nn}}\sqrt{m ! \nn!}  \frac{1}{m!} [\beta^{\nn}] (u^\T \beta)^m \\
    &= \delta_{m, \abs{\nn}}\sqrt{m ! \nn!}  \frac{1}{m!}  \binom{m}{\nn} u^{\nn} \\
    &= \delta_{m, \abs{\nn}} \binom{m}{\nn}^{\hf} u^{\nn}.
  \end{align*}
\end{proof}

\begin{corollary}
  \label{cor:Hermite_InvSubspace}
  Let $r \in \bbN$.
  The space $\caH_r = \spn\{H_{\mm} : \mm \in \bbN^p, \abs{\mm} = r\}$ is invariant under $P_Q$ for any $Q \in \St(p,p)$.
\end{corollary}




\begin{lemma}
  \label{lem:HermiteGaussianConvolution}
  Let \( \sigma \in [-1, 1] \) and $\xi \sim N(0, 1 - \sigma^2)$.
  Then,
  \begin{equation*}
    \E_{\xi} H_m(\sigma y + \xi) = \sigma^m H_m(y).
  \end{equation*}
\end{lemma}
\begin{proof}
  We prove the lemma using the generating function \cref{eq:Hermite_Generating} of the normalized Hermite polynomials.
  Let us define
  \begin{equation*}
    g(t) = \E_{\xi} \exp\left( t (\sigma y + \xi) - \frac{t^2}{2} \right).
  \end{equation*}
  Expanding the right hand side using \cref{eq:Hermite_Generating}, we obtain:
  \begin{equation*}
    g(t) = \sum_{m=0}^\infty \frac{\E_{\xi} H_m(\sigma y + \xi)}{\sqrt{m!}} t^m.
  \end{equation*}
  On the other hand, we can rewrite $g(t)$ as
  \begin{equation*}
    g(t) = \exp\left( t \sigma y - \frac{t^2}{2} \right) \E_{\xi} \exp(t \xi)
    = \exp(t \sigma y - \frac{t^2 \sigma^2}{2}) \exp\left( \frac{t^2 (1 - \sigma^2)}{2} \right)
    = \exp\left( t \sigma y - \frac{t^2 \sigma^2}{2} \right).
  \end{equation*}
  Therefore, using \cref{eq:Hermite_Generating} again, we have
  \begin{equation*}
    g(t) = \sum_{m=0}^\infty \frac{H_m(\sigma y)}{\sqrt{m!}} (t \sigma)^m
    = \sum_{m=0}^\infty \frac{H_m(y)}{\sqrt{m!}} \sigma^m t^m.
  \end{equation*}
  Comparing the two expansions yields the desired result.
\end{proof}

The following is a standard result on the Hermite polynomials.

\begin{lemma}[Recurrence and Derivative]
  \label{lem:Hermite_Recurrence_Derivative}
  Let $n \in \bbN$.
  We have
  \begin{equation}
    \sqrt{n+1} H_{n+1} = x H_n - \sqrt {n} H_{n-1}.
  \end{equation}
  Moreover, for any multi-index $\mm \in \bbN^d$,
  \begin{equation}
    \nabla_{x_i} H_{\mm}(x) = \sqrt{m_i} H_{\mm - e_i}(x),
  \end{equation}
  where $e_i \in \bbN^d$ is the $i$-th unit vector and we use the convention $H_{\mm - e_i} = 0$ if $m_i = 0$.
\end{lemma}

Let $f = \sum_{\mm \in \bbN^d} f_{\mm} H_{\mm}$ be the Hermite expansion of a function on $\R^d$.
Then, with \cref{lem:Hermite_Recurrence_Derivative}, we have
\begin{equation}
  \label{eq:Hermite_ExpansionDerivative}
  \nabla_{x_i} f = \sum_{\mm} f_{\mm} \nabla_{x_i} H_{\mm} = \sum_{\mm} f_{\mm}  \sqrt{m_i} H_{\mm - e_i} = \sum_{\mm} \sqrt{ m_i+1} f_{\mm + e_i} H_{\mm}.
\end{equation}

\subsection{Gaussian distribution}

Let $X \sim N(0,I_d)$, the following Gaussian integral by parts formula is well-known:
\begin{equation*}
  \E X_i h(X) = \E \partial_i h(X),\quad
  \E X_i X_j h(X) = \delta_{ij} \E h(X) + \E \partial_{ij}h(X).
\end{equation*}

%
%
%

\begin{lemma}
  \label{lem:Gaussian_Stein3}
  Let $f,g$ be smooth functions on $\R^d$ such that their derivatives up to third order are continuous and square-integrable with respect to the standard Gaussian measure $\gamma_d$.
  Let $A,B$ be two matrices in $\R^{d \times d}$.
  Suppose that $A \nabla f = 0$ and $B \nabla g = 0$,
  Then, we have
  \begin{equation}
    \E [(\nabla f(X))^\T A X \cdot (\nabla g(X))^\T B X]
    = \E (\nabla f)^{\T} A B^{\T} \nabla g +  \E \Tr (\nabla^2 f A \nabla^2 g B).
  \end{equation}
\end{lemma}
\begin{proof}
  Let $A = (a_{ij})$ and $B = (b_{kl})$.
  We denote by $f_i, f_{ij},g_i,\dots$ the derivatives of $f$ and $g$.
  Then, we can expand the result as
  \begin{equation*}
    I = \E [(\nabla f(X))^\T A X \cdot (\nabla g(X))^\T B X]
    = \E \sum_{ij} a_{ij} f_i x_j \sum_{kl} b_{kl} g_k x_l
    = \sum_{ijkl}a_{ij}  b_{kl} \E x_j x_l (f_i g_k).
  \end{equation*}
  Using the Gaussian integral by parts formula, we have
  \begin{equation*}
    \E x_j x_l (f_i g_k)
    = \delta_{jl} \E f_i g_k + \E \xk{f_{ijl} h_k + f_{il} h_{jk} + f_{ij} h_{kl} + f_i h_{jkl}},
  \end{equation*}
  so we have
  \begin{equation*}
    I = I_0 + I_{1} + I_2 + I_3 + I_4 + I_5.
  \end{equation*}

  To simplify the result, the condition $A \nabla f = B \nabla g = 0$ writes
  \begin{equation*}
    \sum_j a_{ij} f_j = 0,\quad \sum_{l} b_{kl} g_l = 0.
  \end{equation*}
  Consequently, since the derivatives are linear and commutative, we have
  \begin{align*}
    I_1 = \E \sum_{ijkl} a_{ij}  b_{kl} f_{ijl} h_k
    = \E \sum_{ikl} b_{kl} h_k  \sum_{j}a_{ij} f_{jil}
    = \E \sum_{ikl} b_{kl} h_k  \partial_{il}\sum_{j}a_{ij} f_{j} = 0.
  \end{align*}
  Similarly, $I_5 = 0$.
  In addition,
  \begin{equation*}
    I_3 = \E \sum_{ijkl} a_{ij}  b_{kl} f_{ij} h_{kl}
    = \E \sum_{ikl}b_{kl}h_{kl}  \sum_{j}a_{ij} f_{ji}
    = \E \sum_{ikl}b_{kl}h_{kl} \partial_i \sum_{j}a_{ij} f_{j} = 0.
  \end{equation*}
  For the remaining terms, we can write
  \begin{align*}
    I_0 = \E \sum_{ijkl} a_{ij} b_{kl} \delta_{jl} f_i g_k
    = \E \sum_{ijk} a_{ij} b_{kj} f_i g_k
    = \E (\nabla f)^{\T} A B^{\T} \nabla g,
  \end{align*}
  while
  \begin{align*}
    I_2 = \E \sum_{ijkl} a_{ij} b_{kl} f_{il} h_{jk} = \E \sum_l \sum_{ijk}f_{li}a_{ij} h_{jk} b_{kl}
    =\E \sum_l (\nabla^2 f A \nabla^2 g B)_{ll}
    = \E \Tr (\nabla^2 f A \nabla^2 g B).
  \end{align*}

\end{proof}

\section{Proof for the single-index model}

In the following analysis, let us introduce $\rho = \ang{w,w_*}$ as the cosine of the angle between $w$ and $w_*$.
Moreover, the parameterization also gives that
\begin{equation*}
  f(x) = \ang{\bm\beta, \Phi_{w}(x)}_{\ell^2(\bbN)} = \sum_{r \geq 0} \beta_r \lambda_r^{\hf} H_r(\ang{w,x}) = g(\ang{w,x}),
\end{equation*}
where the function $g$ is defined by
\begin{equation*}
  g(u) = \sum_{r \geq 0} g_r H_r(u),\quad g_r = \lambda_r^{\hf} \beta_r.
\end{equation*}
We recall that the eigenvalues are taken as $\lambda_{r} \asymp \exp(-\gamma r)$.
Using \cref{cor:Hermite_InnerProd}, we find that
\begin{equation}
  \label{eq:SIM_Coeff}
  f_{\mm} = \ang{f, H_{\mm}}_{\gamma_d} = \ang{\sum_{r \geq 0} g_r H_r(\ang{w,\cdot}), H_{\mm}}_{\gamma_d} =
  \binom{r}{\mm}^{\hf} w^{\mm} g_r,\quad r = \abs{m}.
\end{equation}
Regarding the projection on the sphere, we introduce the projection operator $P_{w}^{\perp}$ as
$P_{w}^{\perp} v = v - \ang{v,w} w$ for $w \in \bbS^{d-1}$.

\subsection{Basic Properties of the Feature Error Measure}

For the single index model, we can explicitly compute the feature error measure.
First, it is easy to see that
\begin{equation}
  \caE(\delta,\epsilon^2;\Phi_w,f^*) - \caE(\delta,\epsilon^2;\Phi_{w_*},f^*)
  = \caE_{\text{Proj}}(\Phi_w,f^*).
\end{equation}
To compute the projection, recalling that
\begin{equation*}
  f^*(x) = g^*(\ang{w_*,x}) = \sum_{r \geq 0} g^*_r H_r(\ang{w_*,x}),
\end{equation*}
we define
\begin{equation}
  f^*_{r,w} \coloneqq \ang{f^*, H_r(\ang{w,\cdot})}_{\gamma_d}
  = \ang{\sum_{s \geq 0} g^*_r H_s(\ang{w_*,x}), H_r(\ang{w,\cdot}) }_{\gamma_d}
  = \ang{w,w_*}^r g^*_r = \rho^r g^*_r,
\end{equation}
where we apply \cref{cor:Hermite_InnerProd} for the second equality.
Therefore, we have
\begin{equation}
  \label{eq:SIM_FEM_Proj}
  \caE_{\text{Proj}}(\Phi_w,f^*) = \sum_{r \geq 0} (1 - \rho^{2r}) (g^*_r)^2.
\end{equation}
Since $\rho \in [-1,1]$, $\caE_{\text{Proj}}(\Phi_w,f^*) = 0$ iff $w = \pm w_*$.
On the other hand, we compute
\begin{align}
  \notag
  \caE_{\text{Stat}}(\delta,\epsilon^2;\Phi_w,f^*)
  &= \# \dk{ r \geq 0 : \lambda_r \geq \delta} \epsilon^2 +
  \sum_{r \geq 0} (f^*_{r,w})^2 \bm{1} \dk{ \lambda_r < \delta} \\
  \label{eq:SIM_FEM_Stat}
  &= \# \dk{ r \geq 0 : \lambda_r \geq \delta} \epsilon^2
  + \sum_{r \geq 0} \rho^{2r} (g^*_r)^2 \bm{1} \dk{ \lambda_r < \delta}.
\end{align}
Consequently,
\begin{align}
  \notag
  \caE(\delta,\epsilon^2;\Phi_w,f^*)
  &= \caE_{\text{Proj}}(\Phi_w,f^*)  + \caE_{\text{Stat}}(\delta,\epsilon^2;\Phi_w,f^*) \\
  \notag
  &= \sum_{r \geq 0} (1 - \rho^{2r}) (g^*_r)^2 + \# \dk{ r \geq 0 : \lambda_r \geq \delta} \epsilon^2
  + \sum_{r \geq 0} \rho^{2r} (g^*_r)^2 \bm{1} \dk{ \lambda_r < \delta} \\
  \label{eq:SIM_FEM_Total}
  &= \sum_{r \geq 0} \zk{\epsilon^2 + (1 - \rho^{2r}) (g^*_r)^2} \bm{1} \dk{ \lambda_r \geq \delta}
  +\sum_{r \geq 0} (g^*_r)^2 \bm{1} \dk{ \lambda_j < \delta}.
\end{align}

\begin{proposition}
  \label{prop:SIM_FEM_ProjBound}
  Under \cref{assu:SIM_g_decay},
  we have
  \begin{equation}
    \caE(\delta,\epsilon^2;\Phi_w,f^*)
    - \caE(\delta,\epsilon^2;\Phi_{w_*},f^*)
    = \caE_{\text{Proj}}(\Phi_w,f^*)
    \lesssim
    \begin{cases}
      1-\rho, & \alpha > 1, \\
      (1-\rho) \log (1-\rho)^{-1}, & \alpha = 1, \\
      (1-\rho)^\alpha, & \alpha \in (0,1).
    \end{cases}
  \end{equation}
\end{proposition}
\begin{proof}
  Without loss of generality, let us consider $\rho > 0$.
  Using elementary inequalities, we have $1 - \rho^{2r} \leq 2r (1-\rho)$, so
  \begin{equation*}
    \caE_{\text{Proj}}(\Phi_w,f^*)
    = \sum_{r \geq 0} (1 - \rho^{2r}) (g^*_r)^2
    \leq \sum_{r \geq 0} \min(1, 2r (1-\rho)) (g^*_r)^2
    = 2 (1-\rho)\sum_{r \leq L} r (g^*_r)^2 + \sum_{r > L} (g^*_r)^2,
  \end{equation*}
  where $L = \frac{1}{2(1-\rho)}$.
  Since $g^*_r \asymp r^{-\frac{\alpha+1}{2}}$, we have
  $\sum_{r > L} (g^*_r)^2 \lesssim L^{-\alpha}$.
  In the meantime, we have
  \begin{equation*}
    \sum_{r \leq L} r (g^*_r)^2 \lesssim \sum_{r \leq L} r^{-\alpha} \lesssim
    \begin{cases}
      1, & \alpha > 1, \\
      \log L, & \alpha = 1, \\
      L^{1-\alpha}, & \alpha \in (0,1)
    \end{cases}
  \end{equation*}
  Combining the above inequalities, we conclude the result.
\end{proof}

\subsubsection{Initialization}

The following proposition shows the initialization of $\rho$.

\begin{proposition}
  \label{prop:SIM_Init}
  Let $w \sim \mr{Unif}(\bbS^{d-1})$ and $w_* \in \bbS^{d-1}$ be fixed.
  Then, there is an absolute constant $c > 0$ depending on $d$ such that
  \begin{equation}
    \label{eq:SIM_Init_Lower}
    \bbP \dk{ \abs{\ang{w,w_*}} \geq \frac{t}{\sqrt{d}} } \geq 1 - ct,\quad \forall t > 0.
  \end{equation}
  Moreover, we also have
  \begin{equation}
    \label{eq:SIM_Init_Upper}
    \bbP \dk{ \abs{\ang{w,w_*}} \leq \frac{1}{2} } \geq 1 - 2 \exp(-c d),\quad \forall d \geq 1.
  \end{equation}
\end{proposition}
\begin{proof}
  The proof of \cref{eq:SIM_Init_Lower} is quite direct with the explicit density of $\ang{w,w_*}$.
  See, for example, Lemma B.7 in \citet{bietti2022_LearningSingleindex}.
  For \cref{eq:SIM_Init_Upper}, we can use a sub-Gaussian concentration for uniform distribution on the sphere.
\end{proof}

\subsection{Population dynamics}
In this subsection, let us consider the population dynamics of the adaptive feature model for the single index model.
Let us denote by $\caL = \frac{1}{2} \E \xk{f(x) - f^*(x)}^2$ the population loss.
We consider the following equation, which is the population version of \cref{eq:Def_AdaK_SIM}:
\begin{equation}
  \label{eq:Def_AdaK_SIM_Pop}
  \left\{
    \begin{aligned}
      \dot{\beta}_r &= - \nabla_{\beta_r} \caL, \quad \beta_r(0) = 0, \quad r \geq 0, \\
      \dot{w} &= - \nabla_{w}^{\bbS^{d-1}} \caL, \quad w(0) \sim \mr{Unif}(\bbS^{d-1}).
    \end{aligned}
  \right.
\end{equation}
Using \cref{eq:Hermite_InnerProd_1}, the population loss can be computed as
\begin{align*}
  \caL &= \frac{1}{2} \E \zk{f^*(x) - f(x)}^2 = \frac{1}{2}\E \zk{g^*(\ang{w_*,x}) - g(\ang{w,x})}^2 \\
  &= \frac{1}{2}\E \zk{ \sum_{r \geq 0}\xk{ g^*_r H_r(\ang{w_*,x}) - g_r H_r(\ang{w,x})} }^2 \\
  &= \frac{1}{2}\sum_{r \geq 0} \zk{ (g^*_r)^2 + (g_r)^2 - 2 g^*_r g_r \rho^r}.
\end{align*}
Consequently, we find that
\begin{equation}
  \label{eq:SIM_Pop_Grad_beta}
  \nabla_{\beta_r} \caL = -\lambda_r^{\hf}(\rho^r g^*_r - g_r),
\end{equation}
so $\dot{\beta}_r = -\lambda_r^{\hf}(\rho^r g^*_r - g_r)$.
Also, we have
\begin{equation*}
  \nabla_{w} \caL = -\sum_{r \geq 1} r g^*_r g_r \rho^{r-1} w_*.
\end{equation*}
Taking the projection on the sphere, we find that
\begin{equation}
  \label{eq:SIM_Pop_Grad_w}
  -\nabla_{w}^{\bbS^{d-1}} \caL
  = - P_{w}^{\perp} \nabla_{w} \caL
  = \xk{\sum_{r \geq 1} r g^*_r g_r \rho^{r-1}} P_{w}^{\perp} w_*.
\end{equation}
Let us further compute the dynamics of $\rho$.
We have
\begin{equation}
  \dot{\rho} = \ang{\dot{w},w_*} = \xk{\sum_{r \geq 1} r g^*_r g_r \rho^{r-1}} \ang{P_{w}^{\perp} w_*,w_*}
  = \sum_{r \geq 1} r g^*_r g_r \rho^{r-1} (1 - \rho^2),
\end{equation}
where we notice that
\begin{equation}
  \ang{P_{w}^{\perp}w_*,w_*} = \ang{w_* - \ang{w,w_*} w,w_*} = 1 - \rho^2.
\end{equation}
Let us collect the induced dynamics of $g_r = \lambda_r^{\hf} \beta_r$ and $\rho$ from \cref{eq:Def_AdaK_SIM_Pop} in the following
\begin{equation}
  \label{eq:SIM_Pop_Grad_Expanded}
  \left\{
    \begin{aligned}
      \dot{g}_r &= \lambda_r (\rho^r g^*_r - g_r),\quad r \geq 0, \\
      \dot{\rho} &= \sum_{r \geq 1} r g^*_r g_r \rho^{r-1} (1 - \rho^2).
    \end{aligned}
  \right.
\end{equation}

The following proposition shows the basic properties of the population dynamics.

\begin{proposition}
  \label{prop:SIM_Pop_Monotonicity}
  Consider the population dynamics \cref{eq:Def_AdaK_SIM_Pop}.
  Suppose $\rho(0) \neq 0$.
  Then,
  \begin{equation}
    g_r^* g_r(t) \geq 0, \quad \forall r \geq 0, \quad \forall t \geq 0.
  \end{equation}
  Also, for all $t \geq 0$, $\dot\rho(t) \geq 0$ if $\rho(0) > 0$, and $\dot\rho(t) \leq 0$ if $\rho(0) < 0$.
\end{proposition}
\begin{proof}
  It is an easy consequence of the dynamics \cref{eq:SIM_Pop_Grad_Expanded}.
  We only illustrate the proof sketch, while a rigorous proof can be made by the standard ODE continuity argument.
  The dynamics of $g_r$ shows that $g_r(t)$ will have the same sign as $\rho^r g^*_r$.
  Hence, each term $\rho^{r-1} g_r^* g_r(t) = \rho^{-1} (\rho^r g^*_r) g_r(t)$ in $\dot{\rho}$ will have the same sign as $\rho$.
  Consequently, $\dot{\rho}$ has the same sign as $\rho$ and the result follows.
\end{proof}

Following \cref{prop:SIM_Pop_Monotonicity}, we can assume that $g_r^* \geq 0$ and consider $\rho(0) > 0$ without loss of generality in the subsequent analysis.

\begin{proposition}
  \label{prop:SIM_Pop_AppTime}
  Consider the population dynamics \cref{eq:Def_AdaK_SIM_Pop}.
  Let \cref{assu:SIM_InformationIndex} hold.
  Suppose $\rho(0) = \rho_0 \neq 0$.
  Then, $\forall t \geq T^{\mr{app}}$,
  \begin{equation}
    \label{eq:SIM_Pop_AppCond}
    \abs{\rho(t)} \geq \frac{1}{2},\quad
    \abs{g_{\rz}(t)} \geq 2^{-(\rz + 1)} \abs{g^*_{\rz}},
  \end{equation}
  where
  \begin{equation}
    \label{eq:SIM_Pop_AppTime}
    T^{\mr{app}} \lesssim  \log \rho_0^{-1} + \rho_0^{-2(\rz - 1)}.
  \end{equation}
\end{proposition}
\begin{proof}
  We start with the first condition in \cref{eq:SIM_Pop_AppCond}.
  Using \cref{prop:SIM_Pop_Monotonicity}, it suffices to consider $\rho_0 > 0$ and bound the first time when $\rho(t) \geq \frac{1}{2}$.
  Moreover, since each term in $\dot{\rho}$ is non-negative and $1 - \rho^2 \geq 3/4$ when $\rho < 1/2$, we have
  \begin{equation}
    \label{eq:Proof_SIM_Pop_AppTime1}
    \dot{\rho} \geq \frac{3\rz}{4} g^*_{\rz} \rho^{\rz - 1} g_{\rz} = c \rz g^*_{\rz} \rho^{\rz - 1} g_{\rz}.
  \end{equation}
  Also, we recall that
  \begin{equation*}
    \dot{g}_{\rz} = \lambda_{\rz} (\rho^{\rz} g^*_{\rz} - g_{\rz}).
  \end{equation*}
  Let us take $L \coloneqq 1 + \cek{\log_2 \rho_0^{-1}}$.
  We define $\rho_k = 2^k \rho_0$ for $k < L$ and $\rho_L = \frac{1}{2}$.
  We also introduce the times
  \begin{equation}
    \label{eq:Proof_SIM_Times}
    T_{0}^\rho = 0, \quad
    T_{k}^g = \inf \dk{t \geq T_{k}^\rho : g_{\rz}(t) \geq \frac{1}{2} \rho_k^{\rz} g^*_{\rz} }, \quad
    T_{k}^\rho = \inf \dk{t \geq T_{k-1}^g : \rho(t) \geq \rho_k}.
  \end{equation}
  Now, to bound $T_{k}^g$, we have
  \begin{equation*}
    \dot{g}_{\rz} = \lambda_{\rz} (\rho^{\rz} g^*_{\rz} - g_{\rz})
    \geq \lambda_{\rz} \xk{\rho_k^{\rz} g^*_{\rz} - g_{\rz}}
    \geq \frac{1}{2} \lambda_{\rz} \rho_k^{\rz} g^*_{\rz},
    \quad t \in [T_{k}^\rho,T_{k}^g],
  \end{equation*}
  so
  \begin{equation*}
    T_{k}^g - T_{k}^\rho \leq \frac{\frac{1}{2} \rho_k^{\rz} g^*_{\rz}}{\frac{1}{2} \lambda_{\rz} \rho_k^{\rz} g^*_{\rz}} =  \frac{1}{\lambda_{\rz}}.
  \end{equation*}
  On the other hand, for $T_{k}^\rho$, we have
  \begin{equation*}
    \dot{\rho} \geq \frac{\rz}{2} g^*_{\rz} \rho^{\rz - 1} g_{\rz} \geq \frac{\rz}{2} g^*_{\rz} \cdot \rho_{k-1}^{\rz - 1} \cdot \frac{1}{2} \rho_k^{\rz} g^*_{\rz}
    = \frac{\rz}{2^{\rz + 1}} (g^*_{\rz})^2 \rho_k^{2\rz-1}, \quad
    t \in [T_{k-1}^g,T_{k}^\rho],
  \end{equation*}
  so
  \begin{equation*}
    T_{k}^\rho - T_{k-1}^g \leq \frac{\rho_k - \rho_{k-1}}{\frac{\rz}{2^{\rz + 1}} (g^*_{\rz})^2 \rho_k^{2\rz-1}}
    = \rho_0^{-2(\rz - 1)} 2^{\rz} 2^{-2(\rz-1)k} \rz^{-1} (g^*_{\rz})^{-2}.
  \end{equation*}
  Consequently,
  \begin{align*}
    T_{L}^\rho &\leq \sum_{k \leq L} \zk{(T_{k}^\rho - T_{k-1}^g) + (T_{k-1}^g - T_{k-1}^{\rho}) } \\
    & \leq \sum_{k \leq L} \zk{\lambda_{\rz}^{-1} + \rho_0^{-2(\rz - 1)} 2^{\rz} 2^{-2(\rz-1)k} \rz^{-1} (g^*_{\rz})^{-2}} \\
    &= L \lambda_{\rz}^{-1} +  \rho_0^{-2(\rz - 1)} 2^{\rz} \rz^{-1} (g^*_{\rz})^{-2} \sum_{k \leq L}  2^{-2(\rz-1)k} \\
    & \leq L \lambda_{\rz}^{-1} + \rho_0^{-2(\rz - 1)} 2^{\rz} \rz^{-1} (g^*_{\rz})^{-2} \xk{1 - 2^{-2(\rz-1)}}^{-1} \\
    & \leq L \lambda_{\rz}^{-1} + \rho_0^{-2(\rz - 1)} 2^{\rz+1} \rz^{-1} (g^*_{\rz})^{-2} \\
    & \lesssim \lambda_{\rz}^{-1} \log \rho_0^{-1} +(g^*_{\rz})^{-2} \rho_0^{-2(\rz - 1)},
  \end{align*}
  and $T_L^g \leq T_{L}^\rho + \lambda_{\rz}^{-1}$.
\end{proof}

\begin{proposition}
  \label{prop:SIM_Pop_Convergence}
  Consider the population dynamics \cref{eq:Def_AdaK_SIM_Pop}.
  Let \cref{assu:SIM_InformationIndex} hold.
  Suppose that \cref{eq:SIM_Pop_AppCond} holds for some $t_0$.
  Then,
  \begin{equation}
    1 - \abs{\rho(t_0 + t)} \leq \frac{1}{2} \exp(-\rz 2^{-2\rz} (g^*_{\rz})^2 t).
  \end{equation}
\end{proposition}
\begin{proof}
  Without loss of generality, we assume that $\rho(t_0) > 0$ and $g^*_{\rz} > 0$.
  By the monotonicity, we have
  \begin{equation*}
    \dot{\rho} \geq \rz g^*_{\rz} g_{\rz} \rho^{\rz - 1} (1 - \rho^2)
    \geq \rz g^*_{\rz} g_{\rz} \rho^{\rz - 1} (1- \rho)
    \geq \rz 2^{-2\rz}  (g^*_{\rz})^2 (1- \rho),
  \end{equation*}
  so the result follows.
\end{proof}

\begin{proof}[Proof of \cref{thm:SIM_Population}]
  The monotonicity of the feature error measure follows from \cref{eq:SIM_FEM_Total} and the monotonicity of $\abs{\rho}$ in \cref{prop:SIM_Pop_Monotonicity}.
  For the initialization, from \cref{prop:SIM_Init} we have
  \begin{equation*}
    \frac{C}{\sqrt{d}} \leq  \abs{\rho(0)} \leq \frac{1}{2}
  \end{equation*}
  with high probability.
  Therefore, we have
  \begin{equation*}
    \caE_{\text{Proj}}(\Phi_{w(0)},f^*) = \sum_{r \geq 0} (1 - \rho(0)^{2r}) (g^*_r)^2
    \gtrsim \sum_{r \geq 1} (g^*_r)^2 \gtrsim 1.
  \end{equation*}

  On the other hand, since $\rho(0) \geq C /\sqrt {d}$, \cref{prop:SIM_Pop_AppTime} and \cref{prop:SIM_Pop_Convergence} shows that for some $t_0 \asymp \log d + d^{\rz - 1}$, we have
  \begin{equation*}
    1 - \abs{\rho(t_0 + t)} \leq \frac{1}{2} \exp(-C t).
  \end{equation*}
  Hence, the result follow from applying \cref{prop:SIM_FEM_ProjBound} and adjusting the constants.
\end{proof}

\subsection{Sequence model}
In this subsection, we consider the adaptive kernel dynamics in \cref{eq:Def_AdaK_SIM}.

Using the symmetry of the dynamics with respect to negative $g_r^*$ and negative $\rho(0)$ in \cref{eq:SIM_Seq_Grad_g} and \cref{eq:SIM_Seq_Grad_rho_Prelim}, in the subsequence parts,
we will assume that $g_r^* \geq 0$ and $\rho(0) > 0$ without loss of generality.
Also, we will assume \cref{assu:SIM_InformationIndex} holds without mentioning it explicitly.

\subsubsection{Computing the Dynamics}
Let us first compute the dynamics under the noisy sequence model.
We will combine the calculation in the population case and \cref{eq:SeqModel_Grad_Pop} to simplify the computation.

\paragraph{The $\beta_r$ term}
Recalling \cref{eq:SIM_Coeff}, we find that
\begin{equation*}
  \nabla_{\beta_r} f_{\mm} = \lambda_r^{\hf} \binom{r}{\mm}^{\hf} w^{\mm} \delta_{r,\abs{\mm}}.
\end{equation*}
Let us define
\begin{equation}
  \label{eq:SIM_Seq_Def_e}
  e_r  = e_r(w) = \sum_{\abs{\mm} = r} \binom{r}{\mm}^{\hf} w^{\mm} \ep_{\mm}.
\end{equation}
Then, combining with \cref{eq:SIM_Pop_Grad_beta}, \cref{eq:SeqModel_Grad_Pop} shows
\begin{equation*}
  -\nabla_{\theta_r} \hat{\caL} = \lambda_r^{\hf}\xk{\rho^r g_r^* - g_r + e_{r}},
\end{equation*}
and hence
\begin{equation}
  \label{eq:SIM_Seq_Grad_g}
  \dot{g}_r = \lambda_r^{\hf} \dot{\theta}_r = \lambda_r \xk{\rho^r g_r^* - g_r + e_{r}}
\end{equation}

\paragraph{The $w$ term}
For the $w$ term, we have
\begin{equation*}
  \nabla_{w}^{\bbS^{d-1}} f_\mm = g_r \binom{r}{\mm}^{\hf} \nabla_{w}^{\bbS^{d-1}} (w^{\mm}).
\end{equation*}
Hence, using \cref{eq:SIM_Pop_Grad_w} and \cref{eq:SeqModel_Grad_Pop}, we have
\begin{equation*}
  \dot{w} = \xk{\sum_{r \geq 1} r g^*_r g_r \rho^{r-1}} P_{w}^{\perp} w_* + E,\quad
  E = \sum_{r \geq 0}  g_r\sum_{\abs{\mm} = r}  \binom{r}{\mm}^{\hf} \nabla^{\bbS^{d-1}}_{w} (w^{\mm}) \ep_{\mm}.
\end{equation*}
Moreover,
\begin{align}
  \notag
  \dot{\rho} &= \ang{\dot{w},w_*} = \zk{\sum_{r \geq 1} r g^*_r g_r \rho^{r-1}} \ang{P_{w}^{\perp} w_*,w_*} + \ang{E,w_*} \\
  \label{eq:SIM_Seq_Grad_rho_Prelim}
  &= \sum_{r \geq 1} r g^*_r g_r \rho^{r-1} (1 - \rho^2) + \tau, \quad \tau = \ang{E,w_*}.
\end{align}

\subsubsection{Bounding the perturbation terms}

Now, we will bound the perturbation terms $e_r$ and $\tau$ in \cref{eq:SIM_Seq_Grad_g} and \cref{eq:SIM_Seq_Grad_rho_Prelim} respectively by computing their covariance and using uniform bounds for Gaussian processes.
We note that $e_r$ and $\tau$ depend on the parameters $w$ and $\bm\beta$, it is necessary for us to bound them uniformly over the parameter space.

\begin{proposition}
  \label{prop:SIM_Seq_Error_e}
  Let $e_r$ be defined in \cref{eq:SIM_Seq_Def_e}.
  Then, $\Cov(e_r(u), e_r(v)) = \ang{u,v}^r/n$.
  Hence, with probability at least $1 - 4\exp(-d)$, we have
  \begin{equation}
    \label{eq:SIM_Seq_Error_e}
    \sup_{w \in \bbS^{d-1}} \abs{e_r(w)}
    \lesssim \sqrt{\frac{d \log r}{n}},\quad \forall r \geq 0.
  \end{equation}
\end{proposition}
\begin{proof}
  Let us first compute the covariance function.
  Using the binomial theorem, we have
  \begin{equation*}
    \mr{Cov}(e_{r}(u), e_{r}(v)) = \frac{1}{n}\sum_{\abs{\mm} = r}\binom{r}{\mm} u^{\mm} v^{\mm} = \frac{1}{n} \ang{u,v}^r
  \end{equation*}
  For the high probability bound, we apply \cref{lem:GaussianProcess_UniformBound} and notice that $u^{\mm}$ is Lipschitz in $u$ with Lipschitz constant $\abs{\mm}$,
  where we use a union bound on $r$ so that the inequality holds for all $r$ simultaneously.
\end{proof}

\begin{proposition}
  \label{prop:SIM_Seq_Error_tau}
  Let $\tau$ be defined in \cref{eq:SIM_Seq_Grad_rho_Prelim}.
  Then, with probability at least $1 - C \exp(-d)$,
  \begin{equation}
    \label{eq:SIM_Seq_Error_tau_1}
    \abs{\tau(w)} \lesssim \sqrt {\frac{d}{n}}\sum_{r \geq 0} r^{\hf} (\log r)^{\hf} \abs{g_r},\quad \forall w \in \bbS^{d-1}.
  \end{equation}
  Moreover, for any fixed $\nu \in (0,1)$, with probability at least $1 - C \exp(-d)$,
  \begin{equation}
    \label{eq:SIM_Seq_Error_tau_2}
    \abs{\tau(w)} \lesssim \sqrt {1-\rho^2} \sqrt {\frac{d \log \nu^{-1}}{n}} \sum_{r \geq 0} r^{\hf} (\log r)^{\hf} \abs{g_r},\quad \forall w \in \bbS^{d-1}\backslash B_{\nu}(w_*).
  \end{equation}
\end{proposition}
\begin{proof}
  First, we can write
  \begin{align*}
    \tau &= (w_*)^{\T} \sum_{r \geq 0}  g_r\sum_{\abs{\mm} = r}  \binom{r}{\mm}^{\hf} \nabla^{\bbS^{d-1}}_{w} (w^{\mm}) \ep_{\mm} \\
    &= (w_*)^{\T} \sum_{r \geq 0}  g_r\sum_{\abs{\mm} = r}  \binom{r}{\mm}^{\hf} P_{w}^\perp \nabla_{w} (w^{\mm}) \ep_{\mm} \\
    &= \sum_{r \geq 0} g_r \tau_r,
  \end{align*}
  where
  \begin{equation*}
    \tau_r = \tau_r(w) = (P_{w}^\perp w_*)^{\T} \sum_{\abs{\mm} = r}  \binom{r}{\mm}^{\hf} \nabla^{\bbS^{d-1}}_{w} (w^{\mm}) \ep_{\mm}.
  \end{equation*}
  To compute the covariance of $\tau_r$, let us define $a = P_{u}^{\perp}w_*$ and $b = P_{v}^{\perp}w_*$.
  Then,
  \begin{align*}
    \kappa(u,v) &\coloneqq n \Cov(\tau_r(u), \tau_r(v))
    = \E \zk{\sum_i a_i \sum_{\abs{\mm} = r} \binom{r}{\mm}^{\hf} \nabla_{u_i} u^{\mm} \ep_{\mm}}
    \zk{\sum_j b_j \sum_{\abs{\mm} = r} \binom{r}{\mm}^{\hf} \nabla_{v_j} v^{\mm} \ep_{\mm}} \\
    &= \sum_{i,j} a_i b_j \sum_{\abs{\mm} = r} \binom{r}{\mm} \nabla_{u_i} u^{\mm} \nabla_{v_j} v^{\mm} \\
    &= \sum_{i,j} a_i b_j \nabla_{u_i} \nabla_{v_j} \sum_{\abs{\mm} = r} \binom{r}{\mm} u^{\mm} v^{\mm} \\
    &= \sum_{i,j} a_i b_j \nabla_{u_i} \nabla_{v_j} \ang{u,v}^r \\
    &= \sum_{i,j} a_i b_j \zk{r(r-1)\ang{u,v}^{r-2} v_i u_j + r \ang{u,v}^{r-1} \delta_{ij}} \\
    &= r(r-1)\ang{u,v}^{r-2} \sum_{i,j} a_i v_i b_j u_j +
    r \ang{u,v}^{r-1} \sum_{i} a_i b_i.
  \end{align*}
  Rewriting it in vector form, we have
  \begin{align*}
    \kappa(u,v)
    = r(r-1)\ang{u,v}^{r-2} \ang{P_{u}^{\perp}\eta^*, v} \ang{P_{v}^{\perp}\eta^*, u}
    + r \ang{u,v}^{r-1} \ang{P_{u}^{\perp}\eta^*, P_{v}^{\perp}\eta^*}
  \end{align*}
  Introducing
  \begin{align*}
    \rho_u = \ang{u,\eta^*},\quad \rho_v = \ang{v,\eta^*},\quad q = \ang{u,v},
  \end{align*}
  and noticing that
  \begin{align*}
    P_u^{\perp} \eta^* = \eta^* - \rho_u u,\quad P_v^{\perp} \eta^* = \eta^* - \rho_v v,
  \end{align*}
  we have
  \begin{align*}
    \kappa(u,v)
    = r(r-1) q^{r-2} (\rho_v - q \rho_u) (\rho_u - q \rho_v) + r q^{r-1} (1 + q \rho_u \rho_v - \rho_u^2 - \rho_v^2)
  \end{align*}

  When $u=v$, the first term vanishes as $\rho_v = \rho_u$ and $q=1$,
  so
  \begin{align*}
    \kappa(u,u) = r (1-\rho^2).
  \end{align*}

  Now, let us bound the derivative of $\kappa(u,v)$.
  We will frequently use the fact that
  \begin{equation*}
    \norm{\nabla (fg)} = \norm{f \nabla g + g \nabla f} \leq \abs{f} \norm{\nabla g} + \abs{g} \norm{\nabla f}.
  \end{equation*}
  We have
  \begin{equation*}
    \abs{q} \leq 1, \quad \abs{\rho_v - q \rho_u} \leq 2, \quad \abs{\rho_u - q \rho_v} \leq 2, \quad \abs{1 + q \rho_u \rho_v - \rho_u^2 - \rho_v^2} \leq 2.
  \end{equation*}
  and
  \begin{equation*}
    \nabla_u q = v, \quad \nabla_u \rho_u = \eta^*, \quad \nabla_u \rho_v = 0,
  \end{equation*}
  so
  \begin{gather*}
    \nabla_u (\rho_v - q \rho_u) = - (\rho_u v + q \eta^*), \quad
    \nabla_u (\rho_u - q \rho_v) = \eta_* - \rho_v v, \\
    \nabla_u (1 + q \rho_u \rho_v - \rho_u^2 - \rho_v^2) = \rho_u \rho_v v + q \rho_v \eta^* - 2 \rho_u \eta^*.
  \end{gather*}
  and
  \begin{equation*}
    \norm{\nabla_u (\rho_v - q \rho_u) } \leq 2,\quad
    \norm{\nabla_u (\rho_u - q \rho_v) } \leq 2,\quad
    \norm{\nabla_u (1 + q \rho_u \rho_v - \rho_u^2 - \rho_v^2)} \leq 4.
  \end{equation*}
  Combining these, we have
  \begin{equation*}
    \nabla_u \kappa_r(u,v) \leq C (r-1)(r-2),
  \end{equation*}
  where \( C \) is an absolute constant.
  Consequently, we can apply \cref{lem:GaussianProcess_UniformBound} on $\sqrt {n} \tau_r$ to obtain that with probability at least $1 - C r^{-2} \exp(-d)$, we have
  \begin{equation*}
    \sup_{w \in \bbS^{d-1}} \abs{\tau_r(w)} \lesssim \sqrt{d (\log r)/n}.
  \end{equation*}
  Taking the summation, we have
  \begin{equation*}
    \sup_{w \in \bbS^{d-1}} \abs{\tau(w)} \lesssim n^{-\hf} \sum_{r \geq 0} \sqrt{r d \log r} \abs{g_r} = \sqrt{d/n} \sum_{r \geq 0} r^{\hf} (\log r)^{\hf} \abs{g_r}.
  \end{equation*}

  Furthermore, let us introduce the scaled version
  \begin{equation*}
    \bar{\tau}_r(w) = \sqrt{n} \zk{r (1 - \rho_w^2)}^{-\hf} \tau_r(w),\quad
    \bar{\kappa}_r(u,v) = \Cov(\bar{\tau}_r(u),) \bar{\tau}_r(v)) = \zk{r^2(1 - \rho_u^2)(1 - \rho_v^2)}^{-\hf} \kappa_r(u,v).
  \end{equation*}
  so that $\Var(\bar{\tau}_r(w)) = 1$.
  We find
  \begin{equation*}
    \nabla_u \bar{\kappa}_r(u,v) = \rho_u \eta^* (1 - \rho_u^2)^{-\frac{3}{2}} (1 - \rho_v^2)^{-\frac{1}{2}} \kappa_r(u,v) + (1 - \rho_u^2)^{-\frac{1}{2}} (1 - \rho_v^2)^{-\frac{1}{2}} \nabla_u \kappa_r(u,v)
  \end{equation*}
  so
  \begin{equation*}
    \norm{\nabla_u \bar{\kappa}_r(u,v) }
    \leq (1 - \rho_u^2)^{-1} + C (1 - \rho_u^2)^{-\frac{1}{2}} (1 - \rho_v^2)^{-\frac{1}{2}} (r-1)(r-2).
  \end{equation*}
  Consequently, if we have
  \begin{equation*}
    1 - \rho_u^2 \geq \delta, \quad 1 - \rho_v^2 \geq \delta,
  \end{equation*}
  we have
  \begin{equation*}
    \norm{\nabla_u \bar{\kappa}_r(u,v) } \leq C \delta^{-1} (r-1)(r-2).
  \end{equation*}

  Now we are ready to apply \cref{lem:GaussianProcess_UniformBound} on $\bar{\tau}_r$ to obtain the high probability bound.
  With probability at least $1 - C r^{-2} \exp(-d)$, we have
  \begin{equation*}
    \sup_{w \in \bbS^{d-1}\backslash B_\delta(w_*)} \abs{\bar\tau_r(w)} \lesssim \sqrt{d (\log r + \log \delta^{-1})}.
  \end{equation*}
  Returning to $\tau_r$ and taking the summation, we conclude the second bound.
\end{proof}

\begin{proposition}
  \label{prop:SIM_Seq_Control_g}
  For any $r \geq 0$, we have
  \begin{equation}
    \abs{g_r(t)} \leq \xk{ \sup_{s \leq t} \abs{\rho(s)}^r \abs{g_r^*} + \abs{e_r}} \min(1,\lambda_r t).
  \end{equation}
  Moreover, if $g^*_r \geq 0$ and $\rho(s) \geq 0$ for $s \in [0,t]$, then
  \begin{equation}
    g_r(t) \geq - \abs{e_r} \min(1,\lambda_r t).
  \end{equation}
\end{proposition}
\begin{proof}
  The first part follows from the dynamics in \cref{eq:SIM_Seq_Grad_g} and the comparison theorem.
  For the second part, we simply notice
  \begin{equation*}
    \lambda_r (\rho^r g_r^* - g_r + e_r) \geq \lambda_r (- \abs{e_r} - g_r )
  \end{equation*}
  and apply the comparison theorem.
\end{proof}

\begin{corollary}
  \label{cor:SIM_Seq_Controls}
  Let \cref{eq:SIM_Seq_Error_e} hold and $n \gtrsim d$.
  Then, we have
  \begin{equation}
    \abs{g_r(t)} \lesssim \xk{\abs{g^*_r} + \sqrt{\frac{d \log r}{n}}}\min(1,\lambda_r t).
  \end{equation}
  Consequently, the summation in \cref{prop:SIM_Seq_Error_tau} can be bounded by
  \begin{equation}
    \label{eq:SIM_Seq_Control_tau}
    \sum_{r \geq 0} r^{\hf} (\log r)^{\hf} \abs{g_r}
    \lesssim 1 + (\log^+ t)^2.
  \end{equation}
\end{corollary}
\begin{proof}
  It suffices to show the second part, where we can apply  \cref{prop:Series_2}.
\end{proof}

\subsubsection{Training dynamics around initialization}

In this subsection, we will consider the training dynamics around initialization.

\begin{proposition}
  \label{prop:SIM_Seq_Init}
  Under \cref{assu:SIM_InformationIndex}, assume $n \gtrsim d^{2\rz + s}$ for some $s > 0$.
  Let \cref{eq:SIM_Seq_Error_e} and \cref{eq:SIM_Seq_Error_tau_1} hold.
  Let $\delta > 0$ be fixed.
  Then, with probability at least $1 - \delta$, when $n,d$ is large enough,
  we have
  \begin{equation}
    \label{eq:SIM_Seq_Init}
    \inf \dk{t \geq 0 : \abs{\rho(t)} \geq 1/2}
    = T^{\mr{app}} \lesssim \log d + d^{\rz - 1},
  \end{equation}
  where the constant in the $\lesssim$ notation can depend on $\delta$.
  Moreover, after a constant time, $\abs{\rho(t)}$ is monotone increasing when $t \leq T^{\mr{app}}$.
\end{proposition}
\begin{proof}
  First, according to \cref{prop:SIM_Init},
  we have $\abs{\rho(0)} \geq c d^{-1/2}$ for some $c = c(\delta) > 0$ with probability at least $1 - \delta$.
  Also, by the symmetry of the dynamics, we consider $\rho(0) > 0$ without loss of generality.
  In addition, we focus only on $\rho(t) \leq 1/2$ and we will not mention it explicitly.
  Taking $c_0 = c/2$, we claim that we will have $\rho(t) \geq \rho_0 \coloneqq c_0 d^{-1/2}$ for the range of $t$ we are interested in.
  We will prove this claim later.

  Let us recall the dynamics of the component $\rz$:
  \begin{equation*}
    \dot{g}_{\rz} = \lambda_{\rz} \xk{\rho^{\rz} g_{\rz}^* - g_{\rz} + e_{\rz}}
  \end{equation*}
  Thus, when $\rho(t) \geq c_0 d^{-1/2}$, we have
  \begin{equation*}
    \rho^{\rz} g_{\rz}^*/2 \geq C c_0^{\frac{\rz}{2}} d^{-\frac{\rz}{2}} \geq C \sqrt {\frac{d}{n}} \geq \abs{e_{\rz}},
  \end{equation*}
  when $n$ is large enough since $n = \Omega(d^{\rz+1+s})$,
  which implies that
  \begin{equation}
    \label{eq:Proof_SIM_Seq_g_rz_grad}
    \dot{g}_{\rz} \geq \lambda_{\rz} ( \rho^{\rz}_0 g_{\rz}^*/2 - g_{\rz}).
  \end{equation}
  Consequently, $g_{\rz}(t)$ is monotone increasing and we have
  \begin{equation*}
    T_0^g \coloneqq \inf \dk{t \geq 0 : g_{\rz}(t) \geq \frac{1}{4} \rho^{\rz}_0 g_{\rz}^*} \leq \frac{1}{\lambda_{\rz}} \lesssim 1.
  \end{equation*}

  Now we prove the claim for $t \leq T_0^g$.
  We introduce
  \begin{equation*}
    S_0 = \dk{r \geq 1 : (c_0 d^{-1/2})^r \abs{g_r^*} \geq C \sqrt {d(\log r)/n} \geq \abs{e_r}}
  \end{equation*}
  and write
  \begin{align*}
    \dot{\rho} &= \sum_{r \geq 1} r\rho^{r-1} g^*_r g_r  (1 - \rho^2) + \tau \\
    &= \sum_{r \in S_0} r \rho^{r-1} g^*_r g_r  (1 - \rho^2)
    + \sum_{r \notin S_0} r \rho^{r-1} g^*_r g_r  (1 - \rho^2) + \tau \\
    &= P_0 + P_1 + \tau.
  \end{align*}
  For each $r \in S_0$, \cref{eq:SIM_Seq_Grad_g} shows that $g_r \geq 0$, so $P_0 \geq 0$.
  On the other hand, for $r \notin S_0$, \cref{prop:SIM_Seq_Control_g} gives that
  \begin{equation*}
    g_r(t) \geq - C \min(1,\lambda_r t) \abs{e_r} \geq -C \min(1,\lambda_r t) \sqrt {d(\log r)/n}.
  \end{equation*}
  Hence, using $\lambda_r \lesssim e^{-\gamma r}$ and \cref{prop:Series_2}, we have
  \begin{align*}
    P_1 & \geq - C \sum_{r \notin S_0} r \rho^{r-1} g^*_r \min(1,\lambda_r t) \sqrt {d(\log r)/n}  (1 - \rho^2) \\
    & \geq - C \sqrt {d/n} \sum_{r \notin S_0} r 2^{-r} \sqrt {\log r} \abs{g^*_r} \min(1,\lambda_r t) \\
    & \geq - C \sqrt {d/n}.
  \end{align*}
  Moreover, \cref{eq:SIM_Seq_Error_tau_1} and \cref{eq:SIM_Seq_Control_tau} give
  \begin{equation*}
    \abs{\tau(t)} \lesssim \sqrt {d/n} (1 + (\log^+ t)^2).
  \end{equation*}
  Therefore, we have
  \begin{equation*}
    \dot{\rho}(t) \geq - C \sqrt {d/n} (1 + (\log^+ t)^2) \geq - C \sqrt {d/n}, \qq{when} t \leq T_0^g,
  \end{equation*}
  and thus
  \begin{equation*}
    \rho(T_0^g) \geq \rho(0) - C \sqrt {d/n} T_0^g \geq c_0 d^{-1/2} = \rho_0,
  \end{equation*}
  since we have $n \gtrsim d^{2+s}$.

  When $t \geq T_0^g$, we find that
  \begin{equation*}
    P_0 \geq \rz \rho^{\rz - 1} g^*_{\rz} g_{\rz} (1 - \rho^2)
    \geq C \rho^{\rz - 1}  g^*_{\rz} \cdot \rho^{\rz} g_{\rz}^*
    \gtrsim \rho_0^{2\rz - 1}
  \end{equation*}
  Combining it with the bounds for $P_1, \abs{\tau}$ and using $n \gtrsim d^{2\rz + s}$,
  as long as $t$ is polynomial in $d,n$, we have
  \begin{equation*}
    \rho_0^{2\rz - 1} \gtrsim  \sqrt {d/n} (1 + (\log^+ t)^2)
    \gtrsim -P_1 + \abs{\tau(t)},
  \end{equation*}
  and thus
  \begin{equation}
    \label{eq:Proof_SIM_Seq_rho_grad}
    \dot{\rho}(t) \geq P_0 + P_1 - \abs{\tau} \geq  c\rho^{\rz - 1} g^*_{\rz} g_{\rz}.
  \end{equation}
  With \cref{eq:Proof_SIM_Seq_g_rz_grad} and \cref{eq:Proof_SIM_Seq_rho_grad}, we can follow the same argument as in the proof of \cref{prop:SIM_Pop_AppTime}.
  We take $L \coloneqq 1 + \cek{\log_2 \rho_0^{-1}}$ define the times $T^g_k$ and $T^{\rho}_k$ similarly to \cref{eq:Proof_SIM_Times} but replacing the constant $1/2$ by $1/4$, and deduce that
  \begin{equation*}
    T_{k}^g - T_{k}^\rho \leq \lambda_{\rz}^{-1} \lesssim 1, \quad
    T_{k}^\rho - T_{k-1}^g
    \lesssim \rho_0^{-2(\rz - 1)} 2^{\rz} 2^{-2(\rz-1)k},
  \end{equation*}
  and
  \begin{equation*}
    T_{L}^\rho
    \lesssim \log \rho_0^{-1} + \rho_0^{-2(\rz - 1)} \lesssim \log d + d^{\rz - 1}.
  \end{equation*}
  The bound on $T_{L}^\rho$ also shows that \cref{eq:Proof_SIM_Seq_rho_grad} is valid for the whole time interval $t \leq T_{L}^\rho$,
  which implies the claim of $\rho(t) \geq \rho_0$ as well.
\end{proof}

\subsubsection{Training dynamics around convergence}


\begin{proposition}
  \label{prop:SIM_Seq_Conv}
  Under \cref{assu:SIM_InformationIndex}, assume $n \gtrsim d^{1 + s}$ for some $s > 0$.
  Let \cref{eq:SIM_Seq_Error_e}, \cref{eq:SIM_Seq_Error_tau_1} and \cref{eq:SIM_Seq_Error_tau_2} hold with $\nu = 1/n$.
  Suppose \cref{eq:SIM_Seq_Init} holds for some time $t_0 \lesssim  \mr{poly}(d)$.
  Then, there is some $t_1 \leq t_0 + C$ such that
  \begin{equation*}
    \abs{\rho(t_1 + s)} \text{ is monotone increasing and }    1 - \abs{\rho(t_1 + s)} \lesssim \exp(- c s),
  \end{equation*}
  provided that
  \begin{equation*}
    1 - \abs{\rho(t_1+s)} \gtrsim \frac{d \polylog(n,d)}{n}.
  \end{equation*}
\end{proposition}
\begin{proof}
  We will first show that we have $g_{\rz}(t) \geq c g^*_{\rz} \geq c$, $\forall t \geq t_1 = t_0 + C$,
  while we claim that $\rho(t) \geq 1/4$ for $t \in [t_0,t_1]$.
  Recalling the dynamics of $g^*_{\rz}$, since $\rho(t) \geq 1/4$ and $e_{\rz}$ is bounded by \cref{eq:SIM_Seq_Error_e}, we have
  \begin{equation*}
    \dot{g}_{\rz} = \lambda_{\rz} \xk{\rho^{\rz} g_{\rz}^* - g_{\rz} + e_{\rz}}
    \geq \lambda_{\rz}  \xk{c g_{\rz}^* - g_{\rz}},
  \end{equation*}
  so $ g_{\rz}$ is monotone increasing and we have $ g_{\rz}(t_1) \geq \frac{c}{2} g_{\rz}^*$ for $t_1 \leq t_0 + \lambda_{\rz}^{-1}$.
  Now let us prove the claim by lower bounding $\rho(t)$,
  which shares similar argument as in the corresponding part in the proof of \cref{prop:SIM_Seq_Init}.
  We define
  \begin{equation*}
    S_0' = \dk{ r \geq 1 : 4^{-r} \abs{g_r^*} \geq C \sqrt{d \log r/n} \geq \abs{e_r}},
  \end{equation*}
  and decompose
  \begin{equation*}
    \dot{\rho} = \sum_{r \in S_0'} r \rho^{r-1} g^*_r g_r (1 - \rho^2) + \sum_{r \notin S_0'} r \rho^{r-1} g^*_r g_r (1 - \rho^2) + \tau = P_0' + P_1' + \tau.
  \end{equation*}
  It is easy to see that $P_0' \geq 0$, while
  \begin{align*}
    P_1' &\geq - C \sum_{r \notin S_0} r \rho^{r-1} \abs{g^*_r} \min(1,\lambda_r t) \sqrt {d(\log r)/n}  (1 - \rho^2) \\
    & \geq - C \sqrt {d/n} \sum_{r \notin S_0} r (\log r)^{\hf}\abs{g^*_r}  \min(1,\lambda_r t) \\
    & \geq - C \sqrt {d/n} \xk{1 + (\log^+ t)^2},
  \end{align*}
  Similarly, we use \cref{eq:SIM_Seq_Error_tau_1} and \cref{eq:SIM_Seq_Control_tau} to get $\abs{\tau(t)} \lesssim \sqrt {d/n} \xk{1 + (\log^+ t)^2}$.
  These show that
  \begin{equation*}
    \dot{\rho} \geq - C \sqrt {d/n} \xk{1 + (\log^+ t)^2} \geq - C \sqrt {d/n} (\log d)^2, \qq{when} t \in [t_0,t_1],
  \end{equation*}
  and thus we prove the claim.

  Now, let us show the convergence until $\rho(t) = 1 - d \polylog(n,d) / n$.
  After $t \geq t_1$, we use \cref{eq:SIM_Seq_Error_tau_2} with $\nu = 1/n$ and  \cref{eq:SIM_Seq_Control_tau} to get
  \begin{align*}
    \dot{\rho} &\geq \rz \rho^{\rz - 1} g^*_{\rz} g_{\rz} (1 - \rho^2) + P_1' - \abs{\tau} \\
    & \geq c (1 - \rho^2) - C (1 - \rho^2) \sqrt {d/n} \xk{1 + (\log^+ t)^2} - (1 - \rho^2)^{\hf} \sqrt {\frac{d \log n}{n}} \xk{1 + (\log^+ t)^2} \\
    & \geq c (1 - \rho^2) - (1 - \rho^2)^{\hf} \sqrt {\frac{d \log n}{n}} \xk{1 + (\log^+ t)^2}.
  \end{align*}
  Therefore, as long as $t$ is polynomial in $n,d$ and
  \begin{equation*}
    1 - \rho(t)^2 \gtrsim \frac{d \log n}{n} \polylog(n,d), \qq{namely}
    \rho(t) \leq 1 - C \frac{d \polylog(n,d)}{n},
  \end{equation*}
  we have
  \begin{equation*}
    \dot{\rho} \geq c (1 - \rho^2) \geq c (1 - \rho),\qimplies
    \rho(t_1 + s) \geq 1 - \frac{3}{4} \exp(-c s).
  \end{equation*}
  Consequently, it suffices to take an extra $\log n$ time for $\rho(t)$ to increase to $1 - d \polylog(n,d) / n$.
  The requirement on $t$ is polynomial in $n,d$.
\end{proof}

\subsubsection{Proof of \cref{thm:SIM_Sequence}}

First, we can apply \cref{prop:SIM_Seq_Error_e} and \cref{prop:SIM_Seq_Error_tau} with $\nu = 1/n$ that the estimates hold with probability at least $1 - C \exp(-d)$.
Then, we can apply \cref{prop:SIM_Seq_Init} for the initialization, taking $T_1 = T^{\mr{app}}$,
and \cref{prop:SIM_Seq_Conv} for the convergence.
The monotonicity of $\caE(\delta,\epsilon^2;\Phi_w,f^*)$ follows from \cref{eq:SIM_FEM_Total} and the monotonicity of $\rho$ derived in the two propositions.
Here, we notice that if $\rho(t)$ enters the monotone increasing phase in \cref{prop:SIM_Seq_Init}, it will keep increasing until the convergence in \cref{prop:SIM_Seq_Conv}.
Finally, the bounds $\caE^*(\epsilon^2;\Phi_{w(T_1 + s)},f^*) - \caE^*(\epsilon^2;\Phi_{w_*},f^*)$ and $\caE^*(\epsilon^2;\Phi_{w(T_2)},f^*) - \caE^*(\epsilon^2;\Phi_{w_*},f^*)$
come from the bound of $1 - \abs{\rho(t)}$ and \cref{prop:SIM_FEM_ProjBound}.


\section{Proof for the multi-index model}\label{sec:Multi_Proof}

Let us recall the multi-index model
\begin{equation*}
  f^*(x) = g^*(W_*^\T x), \quad W_* \in \mr{St}(d,p^*),
\end{equation*}
and the parameterization
\begin{align*}
  &\Phi_{W}(x) = \xk{\lambda_{\mm}^{\hf} H_{\mm}(W^\T x)}_{\mm \in \bbN^p}, \quad W \in \mr{St}(d,p), \\
  & f(x) = \ang{\bm\beta, \Phi_{W}(x)}_{\ell^2(\bbN^p)} = \sum_{\mm \in \bbN^p} \beta_{\mm} \lambda_{\mm}^{\hf} H_{\mm}(W^\T x)
  = g(W^\T x), \\
  &\quad g(u) = \sum_{\mm \in \bbN^p} \beta_{\mm} H_{\mm}(u) \in L^2(\gamma_p),
\end{align*}
Regarding the weight sequence $\lambda_{\mm}$, we recall that we take
$\lambda_{\mm} = \mu_{\abs{\mm}} \asymp \exp(-\gamma \abs{\mm})$ for some fixed $\gamma > 0$.

With auxiliary operators that will be introduced in \cref{subsubsec:Multi_Prelim_Aux},
we can write $f^* = P_{W_*} g^*$ and $f = P_W g$.
Moreover, under \cref{assu:Multi_RotationInvariance}, we will show in \cref{subsubsec:Multi_RotInv} that we can express
\begin{equation*}
  g_{2\rr} = \nu_{\rr} h_r, \quad g^*_{2\rr} = \nu_{\rr} h_r^*, \quad r = \abs{\rr}, \rr \in \bbN^d,
\end{equation*}
while the other coefficients are zero,
where $\nu_{\rr}$ is a set of coefficients defined in \cref{lem:Multi_RotInvExpansion},
Consequently, we can suppose that the information index $m_0 = 2 \rz$ for some $\rz > 0$.

Furthermore, let us denote $\Rho = W^\T W_*$.
We consider the singular value decomposition $\Rho = U \Sigma V^\T$.
We will also define some auxiliary quantities and we collect them here.
\begin{equation}
  \label{eq:Multi_Proof_PhiOmega}
  \begin{gathered}
    \phi_r = \sum_{\abs{\rr} = r} \nu_{\rr}^2 \sigma^{2\rr},\qquad \phi_0 = 1,\quad \phi_1 = \frac{1}{p} \sum_{i=1}^p \sigma_i^2 \\
    \omega = -\frac{1}{K}\log(\Tr e^{-K\Sigma^2}) = -\frac{1}{K} \log(\sum_{i=1}^p e^{-K \sigma_i^2})\leq \min(\sigma_1^2,\ldots,\sigma_p^2) ,
  \end{gathered}
\end{equation}
where $K$ is a constant to be determined later.
We refer to \cref{eq:Multi_Prelim_PhiR} for the definition of $\phi_r$.

For convenience, we use $\ang{A,B}_{\gamma_p} = \int A(x) B(x) \dd \gamma_p(x)$ for compatible matrices (or vectors) $A,B$.
Also, we extend the definition of $P_W, P_W^{\T}, \caA_{M}$ (defined in \cref{subsubsec:Multi_Prelim_Aux}) to vector-valued functions by element-wise application.

\subsection{Preliminaries}\label{subsec:Multi_Prelim}

Let us first introduce some notations.
For a matrix $A \in \R^{p \times p}$, we denote by $\Diag(A)$ the diagonal matrix with the diagonal entries of $A$ and $\Sym(A) = (A + A^\T)/2$ the symmetric part of $A$.
We note here that $\Diag(A) = \Diag(A^\T)$ and $\Diag(A \Lambda) = \Diag(A) \Lambda = \Lambda \Diag(A)$ for a diagonal matrix $\Lambda$.

\subsubsection{The Stiefel Manifold \texorpdfstring{\(\St(d,p)\)}{St(d,p)}}

The Stiefel manifold \(\St(d,p)\) is the set of all real \((d \times p)\)-matrices whose columns are orthonormal:
\begin{equation*}
  \St(d,p) = \dk{X \in \R^{d \times p} \mid X^\T X = I_p}.
\end{equation*}
It is well-known that \(\St(d,p)\) is a Riemannian manifold with the metric induced by the Euclidean metric on \(\R^{d \times p}\).
For each \(X \in \St(d,p)\), the tangent space \(T_X \St(d,p)\) is given by
\begin{equation*}
  T_X \St(d,p) = \dk{Z \in \R^{d \times p} \mid X^\T Z + Z^\T X = 0}.
\end{equation*}

Moreover, we can compute the Riemannian gradient of a function restricted to \(\St(d,p)\) using its Euclidean gradient.
Let \(F : \St(d,p) \to \mathbb{R}\) be a smooth function.
Denote by \(\nabla F(X)\) the usual (Euclidean) gradient of \(F\) at \(X\)
and by \(\nabla^{\St} F(X)\) the Riemannian gradient of \(F\) on \(\St(d,p)\).
Then, \(\nabla^{\St} F(X)\) is just the orthogonal projection of \(\nabla F(X)\) onto the tangent space \(T_X \St(d,p)\).
A convenient formula for this projection is given by
\begin{align}
  \label{eq:RiemannianGradient}
  \nabla^{\St} F(X) = \proj_{T_X \St}( \nabla F(X)), \quad
  \proj_{T_W \St}(Z) = Z - W \Sym(W^\T Z),
\end{align}
where \(\Sym(A) = (A + A^\T)/2\) denotes the symmetric part of a matrix.


\subsubsection{Auxiliary operators}\label{subsubsec:Multi_Prelim_Aux}

Let us introduce some auxiliary operators that will be useful in the analysis of the multi-index model~\citep{bietti2023_LearningGaussian}.
For a matrix $W \in \St(d,p)$, we define the operator $P_W : L^2(\gamma_p) \to L^2(\gamma_d)$ by
\begin{equation}
  f = P_W g, \quad f(x) = g(W^\T x).
\end{equation}
Then, since $W^\T x \sim N(0,I_p)$ for $x \sim N(0,I_d)$, $P_W$ is isometric that
\begin{equation*}
  \norm{P_W g}_{\gamma_d}^2 = \E_{x \sim \gamma_d} g(W^\T x)^2 = \E g(y)^2 = \norm{g}_{\gamma_p}^2.
\end{equation*}
Consequently, we can define its adjoint operator $P_W^\T : L^2(\gamma_d) \to L^2(\gamma_p)$ by
\begin{equation*}
  \ang{P_W g, h}_{\gamma_d} = \ang{g, P_W^\T  h}_{\gamma_p}.
\end{equation*}
Since $P_W$ is isometric, $P_W^\T$ is the orthogonal projection onto the space
\begin{equation*}
  L^2_{\Phi_W} = \dk{f = g(W^\T x), g \in L^2(\gamma_p)}.
\end{equation*}
In addition, $P_W^{\T} P_W = I_{L^2(\gamma_p)}$ and $P_W P_W^{\T}$ is the orthogonal projection in $L^2(\gamma_p)$ onto the space $L^2_{\Phi_W}$.

\begin{proposition}
  \label{prop:Multi_Prelim_PW_T}
  The adjoint operator $P_W^\T : L^2(\gamma_d) \to L^2(\gamma_p)$ is given explicitly by
  \begin{equation}
    \label{eq:Multi_Prelim_PW_T}
    (P_W^\T h)(y)    = \E h(W y + \xi), \quad \xi \sim N(0, I_d - W W^\T).
  \end{equation}
\end{proposition}
\begin{proof}
  Let us introduce $W^\perp$ such that $\tilde{W} = (W, W_\perp)$ is an orthogonal matrix.
  Then, letting $\tilde{y} = (y,y')^{\T} = \tilde{W}^\T x$, we note that
  \begin{align*}
    \tilde{W}y = (W,W_\perp) (y,y')^{\T} = W y + W_\perp y',
  \end{align*}
  so
  \begin{align*}
    \ang{P_W g, h}_{\gamma_d}
    &= \int g(W^\T x) h(x) \dd \gamma_d(x)
    = \int g(y) h(\tilde{W} y) \dd \gamma_d(y)
    = \int g(y) h(W y + W_\perp y') \dd \gamma_d(y) \\
    &= \int \int g(y) h(W y + W_\perp y') \dd \gamma_p(y) \dd \gamma_{d-p}(y') \\
    &= \int g(y) \zk{\int h(W y + W_\perp y') \dd \gamma_{d-p}(y')} \dd \gamma_p(y).
  \end{align*}
  Therefore, we conclude the following formula for $P_W^\T : L^2(\gamma_d) \to L^2(\gamma_p)$:
  \begin{align*}
  (P_W^\T h)(y)
    = \int h(W y + W_\perp y') \dd \gamma_{d-p}(y'),
  \end{align*}
  where $W_\perp$ is a matrix such that $\tilde{W} = (W,W_\perp)$ is an orthogonal matrix.
  Alternatively, we can write
  \begin{align*}
  (P_W^\T h)(y)
    = \E h(Wy + \xi),\quad \xi \sim N(0, W_\perp W_\perp^\T).
  \end{align*}
  Now, $W_\perp W_\perp^\T$ is the orthogonal projection onto the orthogonal complement of the column space of $W$, so
  \begin{align*}
    W_\perp W_\perp^\T = I_d - W W^\T.
  \end{align*}
  Finally, we conclude that
  \begin{align}
  (P_W^\T h)(y)
    = \E h(Wy + \xi),\quad \xi \sim N(0, I_d - W W^\T).
  \end{align}
\end{proof}


Moreover, for a matrix $M \in \R^{p_1 \times p_2}$ with $\norm{M} \leq 1$, we define
\begin{equation}
  \caA_{M} : L^2(\gamma_{p_2}) \to L^2(\gamma_{p_1}),\quad \caA_{M} h(y) = \E h(M^\T y + \xi),\quad \xi \sim N(0, I_{p_2} - M^\T M).
\end{equation}

The following proposition shows some basic properties of the operator $\caA_{M}$ and its relation to the operator $P_W$.

\begin{proposition}
  \label{prop:Multi_Prelim_A_M}
  We have the following properties:
  \begin{enumerate}[(i)]
    \item Let $W_1 \in \St(d,p_1)$ and $W_2 \in \St(d,p_2)$.
    Then, $P_{W_1}^\T P_{W_2} = \caA_{W_1^\T W_2}$.
    \item $\caA_{M_1} \caA_{M_2} = \caA_{M_1 M_2}$, $\caA_M^\T = \caA_{M^\T}$.
    \item If $Q$ is an orthogonal matrix, then $P_Q^\T = P_{Q^\T}$ and $\caA_{Q} = P_Q$.
  \end{enumerate}
\end{proposition}
\begin{proof}
  To prove (i), we denote $\caA_{W_1,W_2} = P_{W_1}^\T P_{W_2}$
  and compute that
  \begin{align*}
    \caA_{W_1,W_2} h(y) &= \E (P_{W_2} h)(W_1 y + \xi) = \E h(W_2^\T (W_1 y + \xi))\\
    &= \E h(W_2^\T W_1 y + W_2^\T \xi),
  \end{align*}
  where $\xi \sim N(0, I_d - W_1 W_1^\T)$.
  Now, denoting $M = W_1^\T W_2 \in \R^{p_1 \times p_2}$, we have
  \begin{align*}
    W_2^\T \xi \sim N(0, W_2^\T (I_d - W_1 W_1^\T) W_2) = N(0, I_{p_2} - M^\T M).
  \end{align*}
  Consequently, we have
  \begin{align*}
    \caA_{W_1,W_2} h(y) = \caA_{M} h(y) = \E h(M^\T y + \xi),\quad \xi \sim N(0, I_{p_2} - M^\T M).
  \end{align*}

  For (ii), we note that
  \begin{align*}
    \caA_{M_1 M_2} h(y)
    &= \E h( (M_1 M_2)^\T y + \xi), \quad \xi \sim N(0, I_p - (M_1 M_2)^\T (M_1 M_2)) = N(0, I_p - M_2^\T M_1^\T M_1 M_2)
  \end{align*}
  Now, we can take
  \begin{align*}
    \xi = M_2^{\T} \xi_1 + \xi_2, \quad \xi_1 \sim N(0,I_m - M_1^\T M_1),\quad \xi_2 \sim N(0,I_p - M_2^\T M_2),
  \end{align*}
  we can check that the variance of $M_2^{\T} \xi_1 + \xi_2$
  is $ M_2^{\T} (I_m - M_1^\T M_1) M_2 + I_p - M_2^\T M_2 = I_p - M_2^\T M_1^\T M_1 M_2$.
  Therefore, we have
  \begin{align*}
    \caA_{M_1 M_2} h(y)
    &= \E h( (M_1 M_2)^\T y + \xi) = \E h( M_2^\T M_1^\T y + M_2^\T \xi_1 + \xi_2)\\
    &= \E h( M_2^\T(M_1^\T y + \xi_1) + \xi_2) \\
    &=  \caA_{M_1} \caA_{M_2} h(y).
  \end{align*}
  Furthermore, for the adjoint of $\caA_{M}$,
  \begin{align*}
    \ang{\caA_{M} h, g}_{\gamma_{p_1}}
    &= \E h(M^T y + \xi_2) g(y),\quad \xi_2 \sim N(0, I_{p_2} - M^\T M),\quad y \sim N(0,I_{p_1})
  \end{align*}
  Let us take
  \begin{align*}
    z = M^{\T} y + \xi_2 \sim N(0,I_{p_2}),
  \end{align*}
  we find that
  \begin{align*}
    \xi_1 = y - M z = (I_{p_1} - M M^\T) y + M \xi_2 \sim N(0,I_{p_1} - M M^\T)
  \end{align*}
  is independent of $z$.
  Therefore,
  \begin{align*}
    \ang{\caA_{M} h, g}_{\gamma_{p_1}}
    &= \E h(M^T y + \xi_2) g(y)
    = \E h(z) g(M z + \xi_1 )
    = \ang{h, \caA_{M^\T} g}_{\gamma_{p_2}}
  \end{align*}

  The statement (iii) is straightforward from \cref{prop:Multi_Prelim_PW_T} and the definition of $\caA$.
\end{proof}

\begin{proposition}
  \label{prop:Multi_Prelim_A_Sigma}
  Let $\Sigma$ be a diagonal matrix with diagonal entries $\bm{\sigma} = (\sigma_1,\ldots,\sigma_p)$ such that $\norm{\Sigma} \leq 1$.
  Then,
  \begin{equation*}
    \caA_{\Sigma} H_{\mm} = \sigma^{\mm} H_{\mm}.
  \end{equation*}
\end{proposition}
\begin{proof}
  From the definition of $\caA_{\Sigma}$, we have
  \begin{equation*}
    \caA_{\Sigma} H_{\mm}(y)
    = \E_{\xi} H_{\mm}(\Sigma y + \xi) = \prod_{j=1}^{p} \E_{\xi_j} H_{m_j}(\sigma_j y_j + \xi_j),
    \quad \xi \sim N(0, I_p - \Sigma^2).
  \end{equation*}
  Then, the result follows from using \cref{lem:HermiteGaussianConvolution} for each $j$.
\end{proof}


First, we have
\begin{align*}
  \nabla_x (P_W g)(x) = \nabla_x g(W^\T x) = W (\nabla_u g)(W^\T x) = W P_W (\nabla_u g)
\end{align*}

\begin{lemma}
  \label{lem:Multi_Prelim_Op_Grad}
  Let $W \in \St(d,p)$, $g \in L^2(\gamma_p)$ and $h \in L^2(\gamma_d)$.
  Then,
  \begin{align}
    \label{eq:Multi_Prelim_Op_nabla_x}
    \nabla_x (P_W g)(x) &= P_W (W \nabla g),\quad \nabla_x (P_W^{\T} h)(x) = W^\T P_W^{\T} \nabla h, \\
    \nabla_W P_W g &= x P_W (\nabla g)^{\T} \\
    \label{eq:Multi_Prelim_Op_nabla_W}
    \nabla_W^{\St} \ang{P_W g, h}_{\gamma_d} &= -\proj_{T_W \St} \int \nabla h(x) (P_W \nabla g)^\T \dd \gamma_d(x).
  \end{align}
\end{lemma}
\begin{proof}
  To show \cref{eq:Multi_Prelim_Op_nabla_x}, we compute
  \begin{equation*}
    \nabla_x (P_W g)(x) = \nabla_x g(W^\T x) = W (\nabla g)(W^\T x) = W P_W (\nabla g) = P_W (W \nabla g).
  \end{equation*}
  Also, letting $\xi \sim N(0,I_d - W W^\T)$, we have
  \begin{equation*}
    \nabla_x (P_W^{\T} h)(x) = \nabla_x \E h(Wx+\xi) = \E W^\T (\nabla h)(Wx+\xi) = W^\T P_W^{\T} \nabla h.
  \end{equation*}

  For \cref{eq:Multi_Prelim_Op_nabla_W}, we first compute
  \begin{align*}
    \nabla_W \ang{P_W g, h}_{\gamma_d}
    &= \nabla_W \int g(W^\T x) h(x) \dd \gamma_d(x)
    = \int \nabla_W g(W^\T x) h(x) \dd \gamma_d(x) \\
    &= \int x (P_W (\nabla g)^\T) h(x) \dd \gamma_d(x).
  \end{align*}
  Using the Stein identity, we have
  \begin{align*}
    \nabla_W \ang{P_W g, h}_{\gamma_d}
    &= -\int \nabla_x \zk{(P_W(\nabla g)^\T) h(x)} \dd \gamma_d(x) \\
    &=-\int \zk{h(x) \nabla_x (P_W(\nabla g)^\T) + ( \nabla h(x)) P_W(\nabla g)^\T} \dd \gamma_d(x)
  \end{align*}
  For the first part, we find that
  \begin{equation*}
    \int h(x) \nabla_x (P_W(\nabla g)^\T)\dd \gamma_d(x)
    = \int h(x) W P_W (\nabla^2 g) \dd \gamma_d(x) = W \int P_W (\nabla^2 g) h(x) \dd \gamma_d(x).
  \end{equation*}
  Since $\int P_W (\nabla^2 g) h(x) \dd \gamma_d(x)$ is a symmetric matrix, we find that
  \begin{equation*}
    \proj_{T_W \St} \int h(x) \nabla_x (P_W(\nabla g)^\T) \dd \gamma_d(x) = 0.
  \end{equation*}
  Therefore, we only have the second part in our final result.
\end{proof}

\begin{lemma}
  \label{lem:Multi_Prelim_A_M_Grad}
  For $M \in \R^{p_1 \times p_2}$ with $\norm{M} \leq 1$ and $f,f_i \in L^2(\gamma_{p_1})$,$i=1,2$, we have
  \begin{align}
    \label{eq:Multi_Prelim_A_M_Grad}
    &\nabla_x (\caA_M f)(x) =  \caA_M (M \nabla f)
  \end{align}
\end{lemma}
\begin{proof}
  The proof of \cref{eq:Multi_Prelim_A_M_Grad} is similar to the proof of \cref{eq:Multi_Prelim_Op_nabla_x} in \cref{lem:Multi_Prelim_Op_Grad},
  where we recall that $\caA_M f(x) = \E f(M^\T x + \xi)$, $\xi \sim N(0, I_{p_2} - M^\T M)$.
\end{proof}

\subsubsection{Derivatives and Singular Value Decomposition}

Let us be given a flow of matrix $X(t)$, we take its singular value decomposition (SVD) $X(t) = U(t) \Sigma(t) V(t)^\T$, where $U(t),V(t)$ are orthogonal matrices and $\Sigma(t)$ is a diagonal matrix.
Then we have
\begin{equation*}
  \dot{X} = \dot{U} \Sigma V^\T + U \dot{\Sigma} V^\T + U \Sigma \dot{V}^\T.
\end{equation*}
Since $U,V$ are orthogonal, we have $\dot{U}^\T U + U^\T \dot{U} = 0,\quad \dot{V}^\T V + V^\T \dot{V} = 0.$
To show the dynamics of $\Sigma$, we have
\begin{equation*}
  \dot{\Sigma} = U^{\T} \dot{X} V - \xk{U^{\T}\dot{U} \Sigma + \Sigma \dot{V}^\T V}
  = \Diag( U^{\T} \dot{X} V) - \Diag\xk{U^{\T}\dot{U} \Sigma + \Sigma \dot{V}^\T V}
\end{equation*}
Since $U^{\T}\dot{U}$ and $\dot{V}^\T V$ are skew-symmetric and $\Sigma$ is diagonal, we have
\begin{equation*}
  \Diag(U^{\T}\dot{U} \Sigma) = \Diag(\Sigma \dot{V}^\T V) = 0.
\end{equation*}
Therefore,
\begin{equation}
  \label{eq:Multi_Prelim_SVD_Dynamics}
  \dot{\Sigma} = \Diag( U^{\T} \dot{X} V).
\end{equation}
Consequently, for the directional derivative, we also have
\begin{equation}
  \label{eq:Multi_Prelim_SVD_DirectionalDerivative}
  D_{H} \Sigma = \Diag( U^{\T} (D_H X) V),\quad D_H \sigma_j = u_j^{\T} (D_H X) v_j,
\end{equation}
where $u_j, v_j$ are the $j$-th columns of $U,V$ respectively.

\subsubsection{Matrix calculus}
Let $M$ be a symmetric matrix and $f$ be a function.
Then, we can define $f(M)$ via the spectral decomposition $M = Q\Lambda Q^\T$ that
$f(M) = Q f(\Lambda) Q^\T$, where $\Lambda$ is a diagonal matrix with the eigenvalues of $M$ and $f(\Lambda)$ is applied on the diagonal entry-wise.
Let $X$ be a general matrix, we can also define $f(X^\T X)$, $f(XX^\T)$.
Suppose $X = U \Sigma V^\T$ is the SVD of $X$.
Then, it is easy to see that
\begin{equation}
  \label{eq:Multi_Prelim_MatrixCal}
  \begin{aligned}
    & f(X^\T X) =  V f(\Sigma^2) V^\T, \qquad
    f(XX^\T) = U f(\Sigma^2) U^\T, \\
    & X f(X^\T X) = f(XX^\T) X = U f(\Sigma^2) \Sigma V^{\T}
    \qquad
    X^{\T} f(X X^\T) = f(X^\T X) X^{\T} = V \Sigma f(\Sigma^2) U^\T,
  \end{aligned}
\end{equation}

\subsubsection{Initialization}

The following proposition shows the behavior of the singular values of a random initialization on the Stiefel manifold~\citep{absil2006_LargestPrincipal}.
It is adapted from Lemma 3.14 in \citet{bietti2023_LearningGaussian}
\begin{proposition}
  \label{prop:Multi_RandomInit}
  Let $W \sim \mr{Unif}(\St(d,p))$ and $W_* \in \St(d,p)$ be fixed.
  Then, for any $\delta > 0$, there are constants $c_1,c_2 > 0$ depending on $\delta, p$ such that
  \begin{equation}
    \bbP \dk{ \frac{c_1}{\sqrt {d}} \leq \sigma_p \leq \sigma_1 \leq \frac{c_2}{\sqrt{d}} } \geq 1 - \delta.
  \end{equation}
\end{proposition}

\subsubsection{Rotationally Invariant Functions}
\label{subsubsec:Multi_RotInv}

Let us consider a rotationally invariant function $g \in L^2(\gamma_{p})$ whose expansion is given by
$g = \sum_{\mm \in \bbN^p} g_{\mm} H_{\mm}$.
We will give a more explicit form of the gradient of $g$.

\begin{lemma}
  \label{lem:Multi_RotInvExpansion}
  Let $\caG(p)$ be the subspace of rotationally invariant functions in $L^2(\gamma_p)$.
  Then, we have
  \begin{equation}
    \label{eq:Multi_RotInvExpansion}
    \caG(p) = \dk{ g = \sum_{r \geq 0} h_r  \sum_{\abs{\rr} = r} \nu_{\rr} H_{2\rr} : \sum_{r\geq 0} h_r^2 < \infty },
  \end{equation}
  where the coefficients $\nu_{\rr}$ are given by
  \begin{equation}
    \nu_{\rr} \coloneqq C_{r}^{-\hf} \binom{2\rr}{\rr}^{1/2} =  C_{r}^{-\hf} \frac{\sqrt {(2\rr)!}}{(\rr)!}
    ,\quad
    C_{r} = 4^{r} \frac{(p/2)_{r}}{{r!}},~ r = \abs{\rr},
  \end{equation}
  satisfying $\sum_{\abs{\rr} = r} \nu_{\rr}^2 = 1$.
  Here, $(a)_r = a(a+1)\cdots(a+r-1)$ is the rising factorial.
\end{lemma}
\begin{proof}
  First, since $g$ is rotationally invariant, $g$ must be even in each variable,
  so we must have $g_{\mm} = 0$ if $\mm$ has an odd entry.
  It remains to consider those $\mm = 2\rr$.
  Fix $\rr$ and let $r = \abs{\rr}$.
  For any orthogonal matrix $Q$, let us compute the Hermite coefficients of $P_Q g$.
  Using \cref{lem:Hermite_OrthTrans}, we have
  \begin{align*}
    \ang{P_Q g, H_{2\rr}}_{\gamma_d} &= \ang{\sum_{\bm{s}} g_{2\bm{s}} H_{ 2\bm{s}}, P_Q^\T H_{2\rr}}_{\gamma_d} \\
    &= \sum_{\bm{s} : \abs{\bm{s}} = r} g_{2\bm{s}} \ang{H_{2\bm{s}}, P_Q^\T H_{2\rr}} \\
    &= \sum_{\bm{s} : \abs{\bm{s}} = r} g_{2\bm{s}} \sqrt {(2\bm{s})! (2\rr)!} [\alpha^{2\rr} \beta^{2\bm{s}}] \exp(\alpha^\T Q^\T \beta) \\
    &=  \sqrt {2\rr!} [\alpha^{2\rr}] \sum_{\bm{s} : \abs{\bm{s}} = r} g_{2\bm{s}} \sqrt {(2\bm{s})!} [\beta^{2\bm{s}}] \exp((Q \alpha)^\T \beta) \\
    &= \sqrt {2\rr!} [\alpha^{2\rr}] \sum_{\bm{s} : \abs{\bm{s}} = r} (Q \alpha)^{2\bm{s}} ((2\bm{s})!)^{-\hf} g_{2\bm{s}}.
  \end{align*}

  On one hand, let us suppose that $g$ is of the form $g = \sum_{r \geq 0} \sum_{\abs{\rr} = r} \nu_{\rr} h_r H_{2\rr}$.
  We denote $v = Q \alpha$.
  Plugging in the expression of $g_{2\bm{s}}$ yields
  \begin{align*}
    \ang{P_Q g, H_{2\rr}}_{\gamma_d} &= C_{\abs{r}}^{-\hf}\sqrt {(2\bm{r})!}  [\alpha^{2\rr}] \sum_{\bm{s} : \abs{\bm{s}} = r} \frac{1}{\bm{s}!} h_{r} v^{2 \bm{s}} \\
    &= C_{\abs{r}}^{-\hf}\sqrt {(2\bm{r})!}  h_r [\alpha^{2\rr}] \sum_{\bm{s} : \abs{\bm{s}} = r}  \frac{1}{\bm{s}!}  \prod_i (v_i^2)^{s_i} \\
    &= C_{\abs{r}}^{-\hf}\sqrt {(2\bm{r})!}  h_r [\alpha^{2\rr}] \frac{1}{r!} (\sum_{i} v_i^2)^{r} \\
    &= C_{\abs{r}}^{-\hf}\sqrt {(2\bm{r})!}  h_r [\alpha^{2\rr}] \frac{1}{r!} \norm{Q \alpha}^{2 r} \\
    &= C_{\abs{r}}^{-\hf}\sqrt {(2\bm{r})!}  h_r [\alpha^{2\rr}] \frac{1}{r!} \norm{\alpha}^{2 r} \\
    &= C_{\abs{r}}^{-\hf} \frac{\sqrt {(2\rr)!}}{(\rr)!} h_r,
  \end{align*}
  so $P_Q g$ and $g$ have the same coefficients, showing that $P_Q g = g$ and thus $g \in \caG(p)$.

  On the other hand, let us suppose that $P_Q g = g$ for all orthogonal matrices $Q$.
  Let us define the polynomial
  \begin{equation*}
    p(\alpha) = \sum_{\abs{\bm{s}} = r} g_{2\bm{s}} \frac{\alpha^{2\bm{s}}}{\sqrt{(2\bm{s})!}},
  \end{equation*}
  which is a homogeneous polynomial of degree \(2r\) with even exponents.
  Then, for any orthogonal matrix \(Q\), we have
  \begin{equation*}
    g_{2\rr} = \ang{P_Q g, H_{2\rr}}_{\gamma_p} = \sqrt{(2\rr)!} [\alpha^{2\rr}] p(Q \alpha),\quad \qq{for} \rr \in \bbN^p,
  \end{equation*}
  which shows that  \(p(Q^\T \alpha)\) has the same coefficients of $\alpha^{2\rr}$ as \(p(\alpha)\).
  Since the set \(\dk{ \alpha^{2\bm{s}} : \abs{\bm{s}} = r }\) spans the space of even homogeneous polynomials of degree \(2r\),
  this equation implies that \(p(Q^\T \alpha)\) is the same for all \(Q \in O(p)\).
  For degree \(2r\), such rotationally invariant polynomials are known to be multiples of \(\norm{\alpha}^{2r}\).
  Hence, using the multinomial theorem, we get
  \begin{equation*}
    p(\alpha) = c_r \norm{\alpha}^{2r} = c_r \sum_{\abs{\bm{s}} = r} \frac{r!}{\bm{s}!} \alpha^{2\bm{s}}.
  \end{equation*}
  Equating coefficients gives
  \begin{equation*}
    g_{2\bm{s}} \frac{1}{\sqrt{(2\bm{s})!}} = c_r \frac{r!}{\bm{s}!} \implies g_{2\bm{s}} = c_r r! \frac{\sqrt{(2\bm{s})!}}{\bm{s}!}.
  \end{equation*}
  Therefore, $g$ must be the form in \cref{eq:Multi_RotInvExpansion}.

  Finally, the normalizing constant $C_r$ is computed by \cref{prop:sum_lambda} with $\bm{\lambda} = \bm{1}$.
\end{proof}

\begin{proposition}
  \label{prop:sum_lambda}
  Let $p,r\in \mathbb{N}$ and $\bm{\lambda} \in \R^{p}_+$.
  Then
  \begin{equation}
    f(\bm{\lambda}) =
    \sum_{\abs{\rr} = r} \binom{2\rr}{\rr} \bm{\lambda}^{\rr} = [z^r] \prod_{j=1}^p (1 - 4 \lambda_{j} z)^{-\hf}.
  \end{equation}
  Particularly, if $\bm{\lambda} = \bm{1}$, we have
  \begin{equation*}
    f(\bm{1}) = [z^r] (1 - 4 z)^{-p/2} = 4^r \frac{(p/2)_r}{r!}.
  \end{equation*}
\end{proposition}
\begin{proof}
  Recall the identity
  \begin{equation*}
  (1 - 4 \alpha)
    ^{-1/2} = \sum_{n=0}^{\infty} \binom{2n}{n} \alpha^{n}.
  \end{equation*}
  Taking $\alpha = \lambda_j z_j$, we have
  \begin{align*}
    \prod_{j=1}^p (1 - 4 \lambda_j z)^{-1/2} & =  \prod_{j=1}^p   \sum_{r_j=0}^{\infty}    \binom{2r_j}{r_j} \lambda_j^{r_j} z^{r_j} \\
    &= \sum_{n_1,\dots,n_p} \prod_{j=1}^p \binom{2r_j}{r_j} \lambda_j^{r_j} z^{r_j} \\
    &= \sum_{\rr} \frac{(2\rr)!}{(\rr!)^2} \lambda^{\rr} z^{\abs{\rr}} \\
    &= \sum_{r \geq 0} \zk{\sum_{\abs{\rr} = r}\frac{(2\rr)!}{(\rr!)^2}\lambda^{\rr}} z^r.
  \end{align*}
  The proposition follows by comparing the coefficients.
\end{proof}

Let $\Rho \in \R^{p \times p}$ with $\norm{\Rho} \leq 1$.
Let  $\Rho = U \Sigma V^{\T}$ be the SVD of $\Rho$ and $\sigma$ be the diagonal of $\Sigma$.
Then,  using \cref{prop:sum_lambda}, we can introduce the function
\begin{equation}
  \label{eq:Multi_Prelim_PhiR}
  \phi_r(\Rho) \coloneqq  \sum_{\abs{\rr} = r} \nu_{\rr}^2 \sigma^{2\rr} = C_r^{-1} [z^r] \det(I - 4 \Rho^T \Rho z)^{-\hf}, \qquad C_r = 4^r \frac{(p/2)_r}{r!},
\end{equation}
where we define $\phi_0 = 1$.
This function is well-defined and depends only on the singular values of $\Rho$.
To see this, we use \cref{prop:sum_lambda} to obtain
\begin{align*}
  \sum_{\abs{\rr} = r} \nu_{\rr}^2 \sigma^{2\rr} &= C_r^{-1} \sum_{\abs{\rr} = r} \binom{2\rr}{\rr} \sigma^{2\rr} \\
  &= C_r^{-1}[z^r] \prod_{j=1}^p (1 - 4 \sigma_{j}^2 z)^{-\hf} = [z^r] \det(I - 4 \Sigma^2 z)^{-\hf}  \\
  &= C_r^{-1} [z^r] \det(I - 4 \Rho^T \Rho z)^{-\hf}.
\end{align*}
Moreover, we notice that $\phi_1 = \frac{1}{p} \sum_{j=1}^p \sigma_j^2.$

\begin{corollary}
  \label{cor:Multi_RotInv_GradProj}
  Consider a loss $L(g)$ on the space of rotationally invariant functions.
  Let $g = \sum_{\mm} g_{\mm} H_{\mm}$.
  Denote by $\nabla_{g_{\mm}}$ the classical gradient with respect to the coefficient $g_{\mm}$
  and by $\nabla_{g_{\mm}}^{\caG(p)}$ the gradient $\caG(p)$.
  Then, we have
  \begin{equation}
    \nabla_{g_{2\rr}}^{\caG(p)}L(g) = \nu_{\rr} \nabla_{h_r} L(g),\quad
    \nabla_{h_r} L(g) = \sum_{\abs{\rr} = r} \nu_{\rr} \nabla_{g_{2\rr}} L(g),\quad r = \abs{\rr}.
  \end{equation}
\end{corollary}

Regarding the coefficient, we also have the following properties.

\begin{proposition}
  Let the coefficients $\nu_{\rr}$ be defined as in \cref{lem:Multi_RotInvExpansion}.
  Then, we have
  \begin{equation}
    \label{eq:Multi_Prelim_nu_sum_rj}
    \sum_{\abs{\rr} = r} r_k \nu_{\rr}^2 = \frac{r}{p},\quad \forall k = 1,\ldots,p.
  \end{equation}
  Moreover, defining
  \begin{equation*}
    A_{ij} = \begin{cases}
               \sum_{\abs{\rr} = r} \nu_{\rr}^2 (2r_i) (2r_j) \nu_{\rr}^2, & i \neq j \\
               \sum_{\abs{\rr} = r} \nu_{\rr}^2 (2r_i)(2r_i-1) \nu_{\rr}^2, & i = j
    \end{cases}
  \end{equation*}
  we have
  \begin{equation}
    \label{eq:Multi_Prelim_nu_sum_rij}
    \sum_{i=1}^p A_{ki} = \frac{2r(2r-1)}{p},\quad \forall k = 1,\ldots,p.
  \end{equation}
\end{proposition}
\begin{proof}
  Let $A_j = \sum_{\abs{\rr} = r} r_j \nu_{\rr}^2 $.
  Then, by symmetry, we have $A_1 = \dots = A_p$, so
  \begin{equation*}
    A_j = \frac{1}{p} \sum_{j=1}^p A_j = \frac{1}{p} \sum_{\abs{\rr} = r} \nu_{\rr}^2 (r_1 + \dots + r_p) = \frac{1}{p} \sum_{\abs{\rr} = r} \nu_{\rr}^2 r = \frac{r}{p} \sum_{\abs{\rr} = r} \nu_{\rr}^2 = \frac{r}{p}.
  \end{equation*}
  For the statement regarding $A_{ij}$, we have
  \begin{align*}
    \sum_{i=1}^p A_{ki} &= \sum_{\abs{\rr} = r} (2r_k)\xk{\sum_{i=1}^p (2r_i) -1} \nu_{\rr}^2
    = \sum_{i=1}^p A_{ki} \\
    &= \sum_{\abs{\rr} = r} (2r_k) \xk{2 r -1} \nu_{\rr}^2
    = (2r-1) \sum_{\abs{\rr} = r} 2r_k \nu_{\rr}^2 \\
    & = \frac{2r(2r-1)}{p}.
  \end{align*}
\end{proof}

\begin{proposition}
  \label{prop:Multi_Prelim_RotInvGrad}
  Let $g \in \caG(p)$ be a rotationally invariant function.
  Then, for any orthogonal matrix $Q$, we have
  \begin{equation*}
    \nabla g = Q P_Q \nabla g = P_Q Q \nabla g,\qquad
    \nabla^2 g = Q (P_Q \nabla^2 g) Q^\T.
  \end{equation*}
\end{proposition}
\begin{proof}
  Since $g$ is rotationally invariant, we have
  \begin{align*}
    \nabla_x g(x) = \nabla_x (P_Q g)(x) = P_Q (Q \nabla g) = Q P_Q (\nabla g).
  \end{align*}
  Taking the second derivative, we have
  \begin{equation*}
    \nabla^2_x g(x) = \nabla_x^\T [Q P_Q (\nabla g)] = Q \nabla_x^\T (P_Q (\nabla g)) = Q (P_Q \nabla^2 g) Q^\T.
  \end{equation*}
\end{proof}

\begin{corollary}
  \label{cor:Multi_Prelim_Hbar_Grad}
  Let $\bar{H}_r = \sum_{\abs{\rr} = r} \nu_{\rr} H_{2\rr}$ be the projection of $H_r$ onto $\caG(p)$.
  Then, we have
  \begin{equation*}
    \int \nabla \bar{H}_r (\nabla \bar{H}_r)^{\T} \dd \gamma_p = \frac{2r}{p} I_p.
  \end{equation*}
  \begin{equation*}
    \int (\nabla^2 \bar{H}_r)^2 \dd \gamma_p = c_r I_p,\quad c_r = \frac{2r(2r-1)}{p} \leq 4r^2.
  \end{equation*}
\end{corollary}
\begin{proof}
  Using \cref{prop:Multi_Prelim_RotInvGrad}, we find that
  \begin{align*}
    A &=\int \nabla \bar{H}_r (\nabla \bar{H}_r)^{\T} \dd \gamma_p
    = \int Q (P_Q \nabla \bar{H}_r) (P_Q \nabla \bar{H}_r)^{\T} Q^{\T} \dd \gamma_p  \\
    &= Q \int  (P_Q \nabla \bar{H}_r) (P_Q \nabla \bar{H}_r)^{\T} \dd \gamma_p Q^{\T}
    = Q \int \nabla \bar{H}_r (\nabla \bar{H}_r)^{\T} \dd \gamma_p Q^{\T} \\
    & = Q A Q^\T,
  \end{align*}
  so $A$ is a scalar multiple of the identity.

  Using the derivative of the Hermite polynomial, we further compute that
  \begin{align*}
    \nabla_{x_1} \bar{H}_r = \sum_{\abs{\rr} = r} \nu_{\rr} \sqrt{2r_1}  H_{2\rr-2e_1},
  \end{align*}
  so
  \begin{align*}
    A_{11} = \int (\nabla_{x_1} \bar{H}_r)^2 \dd \gamma_p
    = \sum_{\abs{\rr} = r} \nu_{\rr}^2 2r_1 = \frac{2r}{p}.
  \end{align*}

\end{proof}

\subsection{The Feature Error Measure}

Let us recall that the subspace
\begin{equation*}
  L^2_{\Phi_W} = \overline{\spn} \dk{ H_{\mm}(W^\T x),~ \mm \in \bbN^p} = \dk{f = g(W^\T x), ~ g \in L^2(\gamma_p)}.
\end{equation*}
Then, as $P_W^\T$ is the projection onto $L^2_{\Phi_W}$, the orthogonal projection of $f^*$ onto $L^2_{\Phi_W}$ is given by
\begin{equation*}
  f^*_W \coloneqq (P_W P_W^\T) f^* = P_W P_W^\T P_{W_*} g^* = P_W \caA_{\Rho} g^* = P_W \caA_{U} \caA_{\Sigma} \caA_{V^\T} g^* = P_W P_U \caA_{\Sigma} g^*,
\end{equation*}
where we use the SVD of $\Rho = U \Sigma V^\T$, \cref{prop:Multi_Prelim_A_M} and the rotation invariance of $g^*$.
Consequently, using \cref{prop:Multi_Prelim_A_Sigma}, we have
\begin{equation*}
  \norm{f^*_W}_{\gamma_d}^2 = \norm{P_W P_U \caA_{\Sigma} g^*}_{\gamma_d}^2 = \norm{\caA_{\Sigma} g^*}_{\gamma_p}^2
  = \sum_{\mm \in \bbN^p} \sigma^{2\mm} (g^*_{\mm})^2.
\end{equation*}
Then, the projection error is
\begin{equation*}
  \caE_{\text{Proj}}(\Phi_W,f^*)
  = \norm{f^* - f^*_W }_{\gamma_d}^2
  = \norm{f^*}_{\gamma_d}^2 - \norm{f^*_W}_{\gamma_d}^2
  = \norm{g^*}_{\gamma_p}^2 - \norm{f^*_W}_{\gamma_d}^2
  = \sum_{\mm \in \bbN^p} (1 - \sigma^{2\mm}) (g^*_{\mm})^2.
\end{equation*}

To compute the statistical error, recalling the definition, let us introduce
\begin{equation*}
  g^*_{W,\mm} \coloneqq \ang{f^*, H_{\mm}(W^\T x)}_{\gamma_d} = \ang{f^*, P_W H_{\mm}}_{\gamma_d} = \ang{g^*_W, H_{\mm}}_{\gamma_p},
\end{equation*}
where
\begin{equation*}
  g^*_W\coloneqq P_W^\T f^* = P_W^\T P_{W_*} g^* = \caA_{\Rho} g^* = \caA_{U} \caA_{\Sigma} \caA_{V^\T} g^* = P_U \caA_{\Sigma} g^*.
\end{equation*}
Then, since $\lambda_{\mm} = \mu_{\abs{\mm}}$, we have
\begin{align*}
  \caE_{\text{B}} = \sum_{\mm \in \bbN^p} \bm{1}\dk{\lambda_{\mm} < \delta} (g^*_{W,\mm})^2
  = \sum_{r \geq 0} \bm{1}\dk{\mu_{r} < \delta}\sum_{\abs{\mm} = r}  (g^*_{W,\mm})^2.
\end{align*}
Moreover, let us consider the subspace $\caH_r = \spn \dk{ H_{\mm},~\abs{\mm} = r}$, and let $P_{\caH_r}$ the orthogonal projection onto $\caH_r$.
Since $\caH_r$ and $\caH_r^\perp$ are both invariant under $P_U$  from \cref{cor:Hermite_InvSubspace}, we have
\begin{align*}
  \sum_{\abs{\mm} = r}  (g^*_{W,\mm})^2 &= \norm{P_{\caH_r} g^*_{W}}_{\gamma_p}^2
  = \norm{P_{\caH_r} P_U \caA_{\Sigma} g^*}_{\gamma_p}^2
  = \norm{ P_U P_{\caH_r} \caA_{\Sigma} g^*}_{\gamma_p}^2 \\
  &= \norm{P_{\caH_r} \caA_{\Sigma} g^*}_{\gamma_p}^2
  = \sum_{\abs{\mm} = r} \sigma^{2\mm} (g^*_{\mm})^2.
\end{align*}
Consequently,
\begin{equation*}
  \caE_{\text{B}} = \sum_{r \geq 0} \bm{1}\dk{\mu_{r} < \delta} \sum_{\abs{\mm} = r} \sigma^{2\mm} (g^*_{\mm})^2.
\end{equation*}
On the other hand,
\begin{equation*}
  \caE_{\text{V}} = \# \dk{\mm \in \bbN^p : \lambda_{\mm} \geq \delta} \epsilon^2
  =\epsilon^2  \sum_{r \geq 0} \bm{1} \dk{\mu_r \geq \delta}  \sum_{\abs{\mm} = r}1
  = \epsilon^2 \sum_{r \geq 0} \bm{1} \dk{\mu_r \geq \delta}  \binom{p+r-1}{r}.
\end{equation*}
Merge the two terms, we have the following proposition.

\begin{proposition}
  \label{prop:Multi_FEM_Proj}
  Under \cref{assu:Multi_RotationInvariance}, we have
  \begin{gather}
    \label{eq:Multi_FEM_Proj}
    \caE_{\text{Proj}}(\Phi_W,f^*) = \sum_{\mm \in \bbN^p} (1 - \sigma^{2\mm}) (g^*_{\mm})^2, \\
    \caE_{\text{B}} = \sum_{r \geq 0} \bm{1}\dk{\mu_{r} < \delta} \sum_{\abs{\mm} = r} \sigma^{2\mm} (g^*_{\mm})^2,
    \quad
    \caE_{\text{V}} = \epsilon^2 \sum_{r \geq 0} \bm{1} \dk{\mu_r \geq \delta}  \binom{p+r-1}{r}.
  \end{gather}
\end{proposition}

\begin{proposition}
  \label{prop:Multi_FEM_ProjBound}
  Under \cref{assu:Multi_g_decay},
  we have
  \begin{equation}
    \caE(\delta,\epsilon^2;\Phi_W,f^*)
    - \caE(\delta,\epsilon^2;\Phi_{W_*},f^*)
    = \caE_{\text{Proj}}(\Phi_W,f^*)
    \lesssim
    \begin{cases}
      1-\rho, & \alpha > 1, \\
      (1-\rho) \log (1-\rho)^{-1}, & \alpha = 1, \\
      (1-\rho)^\alpha, & \alpha \in (0,1).
    \end{cases}
  \end{equation}
  where $\rho = \min_j \sigma_j^2$.
\end{proposition}
\begin{proof}
  The proof resembles that of \cref{prop:SIM_FEM_ProjBound}, but we deal with multi-index now.
  With \cref{eq:Multi_FEM_Proj}, taking $L = c (1 - \rho)^{-1}$, we have
  \begin{align*}
    \caE_{\text{Proj}}(\Phi_W,f^*)
    &= \sum_{\mm \in \bbN^p} (1 - \sigma^{2\mm}) (g^*_{\mm})^2
    \leq \sum_{\mm \in \bbN^p} (1 - \rho^{\abs{\mm}}) (g^*_{\mm})^2
    \leq \sum_{\mm \in \bbN^p} \min(1, \abs{\mm} (1-\rho)) (g^*_{\mm})^2 \\
    &= \sum_{\abs{\mm} \leq L} 2\abs{\mm} (1-\rho) (g^*_{\mm})^2 + \sum_{\abs{\mm} > L} (g^*_{\mm})^2 \\
    &= I_1 + I_2
  \end{align*}
  We will use the fact that
  \begin{equation*}
    \# \dk{\mm \in \bbN^p : \abs{\mm} = r} = \binom{p+r-1}{r} \lesssim r^{p-1}.
  \end{equation*}
  Also, recall that $\abs{g^*_{\mm}} \lesssim \abs{\mm}^{-\frac{\alpha+p}{2}}$.

  The first term is bounded by
  \begin{align*}
    I_1 &\lesssim 2 (1 - \rho) \sum_{\abs{\mm} \leq L} \abs{\mm}^{-(\alpha + p - 1)}
    = 2 (1 - \rho)  \sum_{r \leq L} \sum_{\abs{\mm} = r} \abs{\mm}^{-(\alpha + p - 1)} \\
    &\lesssim 2 (1 - \rho) \sum_{r \leq L} r^{p-1} r^{-(\alpha + p - 1)}
    = 2 (1 - \rho) \sum_{r \leq L} r^{-\alpha}.
  \end{align*}
  Therefore,
  \begin{equation*}
    I_1 \lesssim
    \begin{cases}
      1-\rho, & \alpha > 1, \\
      (1-\rho) \log (1-\rho)^{-1}, & \alpha = 1, \\
      (1-\rho)^\alpha , & \alpha \in (0,1).
    \end{cases}
  \end{equation*}
  For the second term, we have similarly
  \[
    I_2 \lesssim \sum_{\abs{\mm} > L} \abs{\mm}^{-(\alpha + p)}
    \lesssim \sum_{r > L} r^{p-1} r^{-(\alpha + p)} = \sum_{r > L} r^{-(\alpha+1)} \lesssim L^{-\alpha}.
  \]
  Combining the bounds, we conclude the proposition.
\end{proof}

\subsection{Population Dynamics}

Let us introduce the population dynamics of the multi-index model.
Let us denote by $\caL = \frac{1}{2} \norm{f - f^*}^2_{L^2(\gamma_d)}$ the population loss.
Following \cref{eq:Def_AdaK_Multi}, we consider
\begin{equation}
  \label{eq:Def_AdaK_Multi_Pop}
  \left\{
    \begin{aligned}
      \dot{\bm\beta}(t) &= - \nabla_{\bm\beta}^{\caG(p)} \caL, \quad \bm{\beta}(0) = \bm{0}, \\
      \dot{W}(t) &= - \nabla_{W}^{\mr{St}(d,p)}\caL, \quad W(0) \sim \mr{Unif}(\mr{St}(d,p)).
    \end{aligned}
  \right.
\end{equation}

\subsubsection{Computing the Gradient Flow}\label{subsubsec:Multi_PopGrad}

Using the auxiliary operators and noticing that $f = P_W g$ and $f^* = P_{W_*} g^*$, we can write the population loss as
\begin{align*}
  \caL &= \frac{1}{2} \norm{f - f^*}^2_{L^2(\gamma_d)} =
  \frac{1}{2} \norm{f}_{L^2(\gamma_d)}^2 + \frac{1}{2} \norm{f^*}_{L^2(\gamma_d)}^2
  - \ang{f, f^*}_{L^2(\gamma_d)} \\
  &=\frac{1}{2} \norm{g}_{\gamma_p}^2 + \frac{1}{2} \norm{g^*}_{\gamma_{p^*}}^2
  - \ang{P_W g, P_{W_*} g^*}_{\gamma_d}
  =\frac{1}{2} \norm{g}_{\gamma_p}^2 + \frac{1}{2} \norm{g^*}_{\gamma_{p^*}}^2
  - \ang{g, \caA_{W^\T W_*} g^*}_{\gamma_p},
\end{align*}
where we use the fact that $P_W$ is isometric.
Moreover, using \cref{prop:Multi_Prelim_A_M} and the rotation invariance of $g$ and $g^*$, we have
\begin{align*}
  \ang{g, \caA_{W^\T W_*} g^*}_{\gamma_p}
  &= \ang{g, \caA_{\Rho} g^*}_{\gamma_p}
  = \ang{g, \caA_{U} \caA_{\Sigma} \caA_{V^\T} g^*}_{\gamma_p} \\
  &= \ang{\caA_{U^\T} g, \caA_{\Sigma} \caA_{V^\T} g^* }_{\gamma_p}
  = \ang{P_{U^\T} g, \caA_{\Sigma} P_{V^\T} g^* }_{\gamma_p}\\
  &= \ang{g, \caA_{\Sigma} g^* }_{\gamma_p}
\end{align*}
Consequently, we obtain
\begin{equation}
  \label{eq:Multi_Pop_Loss}
  \caL = \frac{1}{2} \norm{g}_{\gamma_p}^2 + \frac{1}{2} \norm{g^*}_{\gamma_{p^*}}^2 - \ang{g, \caA_{\Sigma} g^*}_{\gamma_p}.
\end{equation}

\paragraph{The dynamics of $h_r$.}

Using \cref{eq:Multi_Pop_Loss}, we have
\begin{equation*}
  \nabla_{g_{\mm}} \caL = \nabla_{g_{\mm}} \frac{1}{2} \norm{g}_{\gamma_p}^2 - \nabla_{g_{\mm}}\ang{g, \caA_{\Sigma} g^*}_{\gamma_p}
  = g_{\mm} - \xk{\caA_{\Sigma} g^*}_{\mm},
\end{equation*}
Using \cref{cor:Multi_RotInv_GradProj}, we find that for $\mm = 2 \rr$, $\abs{\rr} = r$, we have
\begin{equation*}
  \nabla_{g_{\mm}}^{\caG(p)}  \caL = \nu_{\rr} G_r,\quad G_r = \sum_{\bm{s}: \abs{\bm{s}} = r}  \nu_{\bm{s}} \zk{g_{2\bm{s}} - \xk{\caA_{\Sigma} g^*}_{2\bm{s}}}
  =  \sum_{\bm{s}: \abs{\bm{s}} = r} \nu_{\bm{s}} \zk{g_{2\bm{s}} - \xk{\caA_{\Sigma} g^*}_{2\bm{s}}}.
\end{equation*}
Consequently, we have
\begin{equation*}
  \dot{\beta}_{\mm} = -\nabla_{g_{\mm}}^{\caG(p)} \caL \cdot  \pdv{g_{\mm}}{\beta_{\mm}}
  = -\lambda_{\mm}^{\hf} \nu_{\rr} G_r,
\end{equation*}
and thus
\begin{equation*}
  \dot{g}_{\mm} = \lambda_{\mm}^{\hf} \dot{\beta}_{\mm}  = -\lambda_{\mm} \nu_{\rr} G_r, \qquad
  \dot{h}_r= -\mu_{2r} G_r,
\end{equation*}
where we recall that $\lambda_{\mm} = \mu_{\abs{\mm}} = \mu_{2r}$.
Let us further compute $G_r$.
\cref{prop:Multi_Prelim_A_Sigma} gives that
\begin{equation*}
  \xk{\caA_{\Sigma} g^*}_{2\bm{s}} = \sigma^{2\bm{s}} g^*_{2\bm{s}}
  = \sigma^{2\bm{s}} \nu_{\bm{s}} h_{r}^*.
\end{equation*}
so
\begin{align*}
  G_r
  &=\sum_{\bm{s}: \abs{\bm{s}} = r} \nu_{\bm{s}} \zk{\nu_{\bm{s}} h_{r} - \sigma^{2\bm{s}} \nu_{\bm{s}} h_{r}^*}
  = \sum_{\bm{s}: \abs{\bm{s}} = r}\nu_{\bm{s}}^2 h_r - \sum_{\bm{s}: \abs{\bm{s}} = r} \nu_{\bm{s}}^2 \sigma^{2\bm{s}} h_r^* \\
  &= h_r - \sum_{\bm{s}: \abs{\bm{s}} = r} \nu_{\bm{s}}^2 \sigma^{2\bm{s}} h_r^*
\end{align*}
Let us introduce
\begin{equation}
  \phi_r \coloneqq \sum_{\rr: \abs{\rr} = r} \nu_{\rr}^2 \sigma^{2\rr}
\end{equation}
Consequently, we finally get
\begin{equation*}
  \dot{h}_r = \mu_{2r} \zk{\phi_r h_r^* - h_r}.
\end{equation*}

\paragraph{The dynamics of $W$.}
Recall that we have $\caL = \frac{1}{2} \norm{g}_{\gamma_p}^2 + \frac{1}{2} \norm{g^*}_{\gamma_{p^*}}^2 - \ang{P_W g, f^*}_{\gamma_d}$.
Therefore, using \cref{lem:Multi_Prelim_Op_Grad}, we have
\begin{align*}
  \dot{W} = -\nabla_W^{\St} \caL &= \nabla_W^{\St} \ang{P_W g, f^*}_{\gamma_d}
  = \proj_{T_W \St} \int (\nabla f^*)(x) (P_W \nabla g)^\T \dd \gamma_d(x)\\
  &= \proj_{T_W \St} \int W_* (P_{W_*} (\nabla g^*)) (P_W \nabla g)^\T \dd \gamma_d(x) \\
  &= \proj_{T_W \St} W_* B^\T,
\end{align*}
where
\begin{equation*}
  B = \int P_{W_*} (P_W \nabla g) (P_{W_*} (\nabla g^*))^\T \dd \gamma_d(x).
\end{equation*}
Consequently, we have
\begin{align*}
  \dot{\Rho} &= \dot{W}^{\T} W_* = (\proj_{T_W \St} W_* B^\T)^{\T} W_*
  = \zk{W_* B^\T - W \Sym(W^\T W_* B^\T)}^{\T} W_* \\
  &= B - \Sym(\Rho B^\T) \Rho.
\end{align*}

\paragraph{The dynamics of $\Sigma$}.
To compute the dynamics of $\Sigma$, we use \cref{eq:Multi_Prelim_SVD_Dynamics} to get
\begin{align*}
  \dot{\Sigma} &= \Diag(U^\T \dot{\Rho} V) = \Diag(U^\T (B - \Sym(\Rho B^\T) \Rho) V)
  = \Diag(U^\T B V) - \Diag(U^\T \Sym(\Rho B^\T) \Rho V) \\
  &= \Diag(U^\T B V) - \frac{1}{2}\Diag(U^\T U \Sigma V^\T B^\T U \Sigma V^\T V)
  - \frac{1}{2} \Diag(U^\T B \Rho^\T \Rho V) \\
  &= \Diag(U^\T B V) - \frac{1}{2} \Diag(\Sigma V^{\T} B^{\T} U \Sigma) - \frac{1}{2} \Diag(U^\T B V \Sigma^2).
\end{align*}
Now, let us define $\tilde{B} = U^\T B V \in \R^{p \times p}$.
Then,
\begin{align*}
  \dot{\Sigma} &= \Diag(\tilde{B}) - \frac{1}{2} \Diag(\Sigma \tilde{B}^{\T} \Sigma) - \frac{1}{2} \Diag(\tilde{B} \Sigma^2) \\
  &=  \Diag(\tilde{B}) - \frac{1}{2}\Sigma  \Diag( \tilde{B}^{\T})  \Sigma - \frac{1}{2} \Sigma^2 \Diag(\tilde{B}) \\
  &= \Diag(\tilde{B}) - \Sigma^2 \Diag(\tilde{B}) \\
  &= (I_{p} - \Sigma^2) \Diag(\tilde{B}).
\end{align*}
It remains to compute $\Diag(\tilde{B})$.
Using the SVD and \cref{prop:Multi_Prelim_A_M}, we can write $B$ as
\begin{align*}
  B &= \ang{P_W \nabla g, P_{W_*} \nabla g^*}_{\gamma_d} = \ang{\nabla g, P_W^\T P_{W_*} (\nabla g^*)^\T}_{\gamma_p}
  = \ang{\nabla g, \caA_{\Rho} (\nabla g^*)^\T}_{\gamma_p} \\
  &= \ang{\nabla g, \caA_{U} \caA_{\Sigma} \caA_{V^\T} (\nabla g^*)^\T}_{\gamma_p} \\
  &= \ang{\caA_{U^\T} \nabla g, \caA_{\Sigma} \caA_{V^\T} (\nabla g^*)^\T}_{\gamma_p}.
\end{align*}
Hence, we get
\begin{align*}
  \tilde{B} &= U^\T B V = U^\T \ang{\caA_{U^\T} \nabla g, \caA_{\Sigma} \caA_{V^\T} (\nabla g^*)^\T}_{\gamma_p} V
  = \ang{ U^\T \caA_{U^\T} \nabla g, \caA_{\Sigma} \caA_{V^\T} (V^\T \nabla g^*)^\T}_{\gamma_p} \\
  &\stackrel{(a)}{=} \ang{ \nabla (\caA_{U^\T} g), \caA_{\Sigma} \nabla (\caA_{V^\T} g^*)^\T}_{\gamma_p} \\
  &\stackrel{(b)}{=} \ang{\nabla (P_{U^\T} g), \caA_{\Sigma} (\nabla (P_{V^\T} g^*))^\T}_{\gamma_p} \\
  &\stackrel{(c)}{=} \ang{\nabla g, \caA_{\Sigma} (\nabla g^*)^\T}_{\gamma_p},
\end{align*}
where we use  \cref{eq:Multi_Prelim_A_M_Grad} in \cref{lem:Multi_Prelim_A_M_Grad} to get (a),
\cref{prop:Multi_Prelim_A_M} to get (b),
and the rotation invariance of $g$ and $g^*$ to get (c).
Finally, we have
\begin{align*}
  \tilde{B}_{ii} &= \ang{ \nabla_{x_i} g, \caA_{\Sigma}  \nabla_{x_i} g^*}_{\gamma_p}
  = \ang{ \sum \sqrt{m_i + 1} g_{\mm + e_i} H_{\mm}, \caA_{\Sigma} \sum \sqrt{n_i + 1} g^*_{\nn + e_i} H_{\nn}}_{\gamma_p} \\
  &=  \sum \sqrt{(n_i + 1)(m_i+1)} g_{\mm + e_i} g^*_{\nn + e_i} \ang{H_{\mm}, \caA_{\Sigma} H_{\nn}}_{\gamma_p} \\
  &=  \sum \sqrt{(n_i + 1)(m_i+1)} g_{\mm + e_i} g^*_{\nn + e_i} \ang{H_{\mm}, \sigma^{\nn} H_{\nn}}_{\gamma_p} \\
  &=  \sum_{\mm} (m_i + 1)  \sigma^{\mm} g_{\mm + e_i} g^*_{\mm + e_i} \\
  &=  \sum_{\mm} m_i  \sigma^{\mm-e_i} g_{\mm} g^*_{\mm} \\
  &  =  \sum_{r \geq 1} h_r h^*_r \sum_{\abs{\rr} = r} 2 r_i \sigma^{2 \rr-e_i} \nu_{\rr}^2,
\end{align*}
and thus
\begin{equation*}
  \dot{\sigma}_i = (1 - \sigma_i^2) \sum_{r \geq 1} h_r h^*_r \sum_{\abs{\rr} = r} 2 r_i \sigma^{2 \rr-e_i} \nu_{\rr}^2,
\end{equation*}
and
\begin{equation*}
  \dv{t} \sigma_i^2 = 2 \sigma_i \dot{\sigma}_i
  = 4(1 - \sigma_i^2) \sum_{r \geq 1} h_r h^*_r \sum_{\abs{\rr} = r} r_i \nu_{\rr}^2 \sigma^{2 \rr}.
\end{equation*}

\paragraph{The dynamics of $\phi_r$ and $\omega$}

For $\phi_r$, we have
\begin{align*}
  \dv{t}\phi_r &= \sum_{\abs{\rr} = r}\nu_{\rr}^2 \dv{t} \sigma^{2\rr}
  = \sum_{\abs{\rr} = r} \nu_{\rr}^2 \sum_{i=1}^p \sigma^{2(\rr-e_i)} 2r_i \dv{t} \sigma_i^2
\end{align*}

For $\omega$, let us recall that
\begin{equation*}
  \omega \coloneqq -\frac{1}{K} \log(\sum_{i=1}^p e^{-K \sigma_i^2}).
\end{equation*}
Using \cref{prop:Aux__Softmin}, we get
\begin{align*}
  \dot{\omega}
  &= \sum_{i=1}^p \pdv{\omega}{\sigma_i^2} \dv{t} \sigma_i^2
  = \sum_{i=1}^p \frac{e^{-K \sigma_i^2}}{\sum_{j=1}^p e^{-K \sigma_j^2}} \cdot (1 - \sigma_i^2) \sum_{r \geq 1} h_r h^*_r \sum_{\abs{\rr} = r}  r_i \nu_{\rr}^2 \sigma^{2 \rr}
\end{align*}

\paragraph{Summary}

Collecting the results, we have the following proposition.
\begin{proposition}
  \label{prop:Multi_Pop_GradEq}
  Consider the population dynamics \cref{eq:Def_AdaK_Multi_Pop}.
  Then, we have the following dynamics:
  \begin{equation}
    \label{eq:Multi_Pop_GradEq}
    \begin{aligned}
      \dot{h}_r & = \mu_{2r} \xk{\phi_r h_r^* - h_r} \\
      \dot{W} & = \proj_{T_W \St} W_* B^\T, \quad B = \ang{P_W \nabla g, (P_{W_*} \nabla g^*)^\T}_{\gamma_d} \in \R^{p \times p}\\
      \dot{\Rho} &= B - \Sym(\Rho B^\T) \Rho \\
      \dv{t} \sigma_i^2 &= 4(1 - \sigma_i^2) \sum_{r \geq 1} h_r h^*_r \sum_{\abs{\rr} = r} r_i \sigma^{2 \rr} \nu_{\rr}^2 \\
      \dot{\phi}_r &= \sum_{\abs{\rr} = r} \nu_{\rr}^2 \sum_{i=1}^p \sigma^{2(\rr-e_i)} 2r_i \dv{t} \sigma_i^2,\quad
      \dot{\phi}_1 = \frac{1}{p}\sum_{i=1}^p \dv{t} \sigma_i^2 \\
      \dot{\omega} &= \frac{4}{\sum_{j=1}^p e^{-K \sigma_j^2}}\sum_{i=1}^p e^{-K \sigma_i^2} \dv{t} \sigma_i^2.
    \end{aligned}
  \end{equation}
\end{proposition}

\subsubsection{Analysis of the Dynamics}

\begin{proposition}
  \label{prop:Multi_Pop_Monotonicity}
  Consider the population dynamics \cref{eq:Def_AdaK_Multi_Pop}
  Suppose $\rho(0) \neq 0$.
  Then, for all $r \geq 0, t \geq 0$, we have
  \begin{equation*}
    h_r^* h_r(t) \geq 0, \qquad   \dot{\phi}_r(t) \geq 0,\quad \dot{\omega} \geq 0.
  \end{equation*}
  Moreover, we have the bound
  \begin{equation}
    \begin{aligned}
      \dot{\omega} & \geq C (1-\phi_1) \sum_{r \geq 1} r \omega^{r} h_r h^*_r \\
      \dot{\phi}_1 &\geq C (1 - \phi_1)\sum_{r \geq 1} r \omega^{r} h_r h^*_r \\
      \dot{h}_r & \geq \mu_{2r} \xk{\omega^{r} h_r^* - h_r}, \quad \text{(assuming $h_r^* > 0$)}.
    \end{aligned}
  \end{equation}
\end{proposition}
\begin{proof}
  Without loss of generality, we can assume that $h_r^*$ is positive.
  To prove the first statement, we observe that if $\phi_r(t) \geq 0$, then $h_r$ will be non-negative.
  Then, the dynamics of $\phi_r$ shows that $\dot{\phi}_r \geq 0$,
  which in turn guarantees that $h_r$ is non-negative.
  A rigorous proof can be made by a standard contradiction argument in the ODE theory.

  For the second statement, first we have
  \begin{align*}
    \dv{t} \sigma_i^2
    &= 4 (1-\sigma_i^2) \sum_{r \geq 1} h_r h^*_r \sum_{\abs{\rr} = r} r_i \nu_{\rr}^2 \sigma^{2 \rr}
    \geq 4 (1-\sigma_i^2) \sum_{r \geq 1} h_r h^*_r \sum_{\abs{\rr} = r} r_i \nu_{\rr}^2  \omega^{\abs{\rr}} \\
    &\stackrel{(a)}{=} 4 (1-\sigma_i^2) \sum_{r \geq 1} h_r h^*_r \omega^r \frac{r}{p}
    = \frac{4}{p} (1-\sigma_i^2) \sum_{r \geq 1} r \omega^r h_r h^*_r \\
    &\eqqcolon (1-\sigma_i^2) A,\quad A = \frac{4}{p} \sum_{r \geq 1} r \omega^r h_r h^*_r,
  \end{align*}
  where we apply \cref{eq:Multi_Prelim_nu_sum_rj} in (a).
  Plugging this into the dynamics of $\phi_r$, we have
  \begin{align*}
    \dot{\phi}_r &= \sum_{\abs{\rr} = r} \nu_{\rr}^2 \sum_{i=1}^p \sigma^{2(\rr-e_i)} 2r_i \dv{t} \sigma_i^2
    \geq \sum_{\abs{\rr} = r} \nu_{\rr}^2 \sum_{i=1}^p \omega^{\abs{\rr}-1} 2r_i (1-\sigma_i^2) A \\
    &= A \omega^{r-1} \sum_{i=1}^p (1-\sigma_i^2)\sum_{\abs{\rr} = r} \nu_{\rr}^2 2r_i
    = A \omega^{r-1} \sum_{i=1}^p (1-\sigma_i^2) \frac{2r}{p} \\
    &= 2r A (1 - \phi_1) \omega^{r-1}.
  \end{align*}
  Particularly,
  \begin{equation*}
    \dot{\phi}_1 \geq C (1 - \phi_1)  \sum_{r \geq 1} r \omega^r h_r h^*_r.
  \end{equation*}
  For $\dot{\omega}$, we have
  \begin{align*}
    \dot{\omega} &= \frac{1}{\sum_{j=1}^p e^{-K \sigma_j^2}} \sum_{i=1}^p e^{-K \sigma_i^2} \dv{t} \sigma_i^2
    \geq \frac{1}{\sum_{j=1}^p e^{-K \sigma_j^2}} \sum_{i=1}^p e^{-K \sigma_i^2} (1-\sigma_i^2) A \\
    & \geq (1 - \phi_1) A,
  \end{align*}
  where we apply \cref{prop:Aux__Softmin_2} in the last inequality.

\end{proof}

\begin{proposition}
  \label{prop:Multi_Pop_AppTime}
  Consider the population dynamics \cref{eq:Def_AdaK_Multi_Pop}.
  Let \cref{assu:Multi_InformationIndex} hold and $\rz = m_0/2$.
  Suppose $\rho_0 = \min_{1 \leq j \leq p} \sigma_j(0) \neq 0$.
  Then, by taking $K \geq 2 \rho_0^{-2} \log p$, we have
  $\forall t \geq T^{\mr{app}}$,
  \begin{equation}
    \label{eq:Multi_Pop_AppCond}
    \min_{1\leq j \leq p}\abs{\sigma_j(t)} \geq \frac{1}{2},\quad
    \abs{h_{\rz}(t)} \geq 2^{-(m_0 + 1)} \abs{h^*_{\rz}},
  \end{equation}
  where
  \begin{equation}
    \label{eq:Multi_Pop_AppTime}
    T^{\mr{app}} \lesssim  \log \rho_0^{-1} + \rho_0^{-2(m_0 - 1)}.
  \end{equation}
\end{proposition}
\begin{proof}
  We can focus on the $t$ such that $\phi_1(t) \leq 1-\frac{1}{2p}$:
  If $\phi_1(t) \geq 1 -\frac{1}{2p}$, then since $p \phi_1(t) \leq p-1 + \sigma_j^2$, we already have $\sigma_j^2 \geq 1/2.$
  Using the property in \cref{prop:Aux__Softmin} of $\omega$, as long as we take $K \geq 2 \rho_0^{-2} \log p$, we have
  \begin{equation*}
    \omega(0) \geq \rho_0^2 - \frac{1}{K} \log p \geq \frac{1}{2} \rho_0^2.
  \end{equation*}
  Now, we can use the dynamics of $\dot{\omega}$ to get
  \begin{equation*}
    \dot{\omega} \geq c p^{-2} \sum_{r \geq 1} r \omega^{r} h_r h^*_r
    \geq c \omega^{\rz} h_{\rz} h^*_{\rz},
  \end{equation*}
  while
  \begin{equation*}
    \dot{h}_{\rz} \geq c \xk{\omega^{\rz} h_{\rz}^* - h_{\rz}}.
  \end{equation*}
  Therefore, we can follow the same idea of analysis as in \cref{prop:SIM_Pop_AppTime}.
  Let us define
  \begin{equation}
    \label{eq:Proof_Multi_Times}
    T_{0}^\rho = 0, \quad
    T_{k}^g = \inf \dk{t \geq T_{k}^\rho : h_{\rz}(t) \geq \frac{1}{2} \rho_k^{2\rz} h^*_{\rz} }, \quad
    T_{k}^\rho = \inf \dk{t \geq T_{k-1}^g : \omega(t) \geq \rho_k^2},
  \end{equation}
  where $\rho_k = 2^k \rho_0$ and $k \leq L$, $L \coloneqq 1 + \cek{\log_2 \rho_0^{-1}}$.
  When $t \in [T_{k}^\rho,T_{k}^g]$,
  we have
  \begin{equation*}
    \dot{h}_{\rz} \geq \frac{1}{2}c \rho_k^{2\rz} h^*_{\rz} = c \rho_k^{2\rz} h^*_{\rz}
  \end{equation*}
  so
  \begin{equation*}
    T_{k}^g - T_{k}^\rho \leq C.
  \end{equation*}
  On the other hand, when $t \in [T_{k-1}^g,T_{k}^\rho]$, we have
  \begin{equation*}
    \dot{\omega} \geq c \rho_k^{2\rz} \cdot \rho_k^{2\rz} h^*_{\rz} \cdot h_{\rz}^*
    = c \rho_k^{4\rz} (h^*_{\rz})^2.
  \end{equation*}
  Hence,
  \begin{equation*}
    T_{k}^\rho - T_{k-1}^g
    \leq \frac{\rho_k^2 - \rho_{k-1}^2}{\rho_k^{4\rz} (h^*_{\rz})^2}
    = C \rho_0^{-2(2\rz-1)} 2^{-2(2\rz-1)k}.
  \end{equation*}
  Consequently,
  \begin{equation*}
    T_{L}^\rho \leq \sum_{k \leq L} \zk{(T_{k}^\rho - T_{k-1}^g) + (T_{k-1}^g - T_{k-1}^{\rho}) } \lesssim
    \log \rho_0^{-1} + \rho_0^{-2(2\rz - 1)}.
  \end{equation*}
  Recalling that $m_0 = 2\rz$, we obtain the desired bound.
\end{proof}

\begin{proposition}
  \label{prop:Multi_Pop_Convergence}
  Consider the population dynamics \cref{eq:Def_AdaK_Multi_Pop}.
  Let \cref{assu:Multi_InformationIndex} hold and $\rz = m_0/2$.
  Suppose that \cref{eq:Multi_Pop_AppCond} holds for some $t_0$.
  Then,
  \begin{equation}
    1 - \phi_1(t_0+t) \leq \frac{1}{2} \exp(-c t).
  \end{equation}
\end{proposition}
\begin{proof}
  When the condition \cref{eq:Multi_Pop_AppCond} holds, we have
  \begin{equation*}
    \dot{\phi}_1 \geq C (1 - \phi_1)\sum_{r \geq 1} r \omega^{r} h_r h^*_r
    \geq C (1 - \phi_1) \rz \omega^{\rz} h_{\rz} h^*_{\rz}
    \geq c (1 - \phi_1).
  \end{equation*}
\end{proof}

\begin{proof}[Proof of \cref{thm:Multi_Population}]
  The monotonicity of the feature error measure follows from \cref{eq:Multi_FEM_Proj} and the monotonicity of $\sigma_j$ in \cref{prop:Multi_Pop_Monotonicity}.
  For the initialization, from \cref{prop:Multi_RandomInit} we have
  \begin{equation*}
    \frac{c_1}{\sqrt{d}} \leq \sigma_j(0) \leq \frac{c_2}{\sqrt{d}},\quad \forall j = 1,\dots,p.
  \end{equation*}
  with high probability.
  Therefore, we have
  \begin{equation*}
    \caE_{\text{Proj}}(\Phi_{W(0)},f^*) = \sum_{\mm \in \bbN^p} (1 - \sigma^{2\mm}) (g^*_{\mm})^2,
    \gtrsim  \sum_{\mm \in \bbN^p} (g^*_{\mm})^2 \gtrsim 1.
  \end{equation*}

  On the other hand, since $\rho_0 = \min_j \sigma_j(0) \geq c_1 /\sqrt {d}$, \cref{prop:Multi_Pop_AppTime} and \cref{prop:Multi_Pop_Convergence} shows that for some $t_0 \asymp \log d + d^{\rz - 1}$, we have
  \begin{equation*}
    1 - \abs{\sigma_j(t_0 + t)} \leq \frac{1}{2} \exp(-C t),\quad \forall j = 1,\dots,p.
  \end{equation*}
  Hence, the result follow from applying \cref{prop:Multi_FEM_ProjBound} and adjusting the constants.
\end{proof}

\subsection{Sequence Model}

\subsubsection{Computing the Dynamics}
Let us first compute the dynamics of the adaptive kernel model.
We modify the computations in \cref{subsubsec:Multi_PopGrad} with \cref{eq:SeqModel_Grad_Pop} for the computation.

\paragraph{Dynamics of \( h_r \).}
First, using \cref{eq:SeqModel_Grad_Pop} and \cref{cor:Multi_RotInv_GradProj}, we have
\begin{equation*}
  -  \nabla_{g_{\mm}}^{\caG(p)} \bar{\caL}_n
  = - \nabla_{g_{\mm}}^{\caG(p)} \caL + \sum_{\nn \in \bbN^d} \ep_{\nn} \nabla_{g_{\mm}}^{\caG(p)} f_{\nn}
  = - \nabla_{g_{\mm}}^{\caG(p)} \caL + \sum_{\nn \in \bbN^d} \ep_{\nn}  \nu_{\rr} \nabla_{h_r} f_{\nn},
\end{equation*}
where $\mm = 2\rr$, $r = \abs{\rr}$, and
\begin{equation*}
  \nabla_{h_r} f_{\nn} = \sum_{\abs{\rr} = r} \nu_{\rr} \nabla_{g_{2\rr}} f_{\nn},\quad
  f_{\nn} \coloneqq \ang{f,H_{\nn}}_{\gamma_d} = \ang{P_W g,H_{\nn}}_{\gamma_d}.
\end{equation*}
Therefore, we have
\begin{equation*}
  \dot{\beta}_{\mm} = - \nabla_{g_{\mm}}^{\caG(p)} \bar{\caL}_n \cdot \pdv{g_{\mm}}{\beta_{\mm}}
  =\lambda_{\mm}^{\hf} \zk{ -\nabla_{g_{\mm}}^{\caG(p)} \caL +  \nu_{\rr} \sum_{\nn \in \bbN^d} \ep_{\nn} \nabla_{h_r} f_{\nn}},
\end{equation*}
and thus
\begin{equation*}
  \dot{h}_r = \mu_{2r} \xk{-G_r + e_r},\quad e_r = \sum_{\nn \in \bbN^d} \ep_{\nn} \nabla_{h_r} f_{\nn}.
\end{equation*}
Let us further compute $e_r$.
First, we have
\begin{align*}
  \nabla_{h_r} f_{\nn} &=
  \nabla_{h_r} \ang{P_W g,H_{\nn}}_{\gamma_d}
  = \nabla_{h_r} \ang{ g, P_W^\T H_{\nn}}_{\gamma_p} \\
  &=   \nabla_{h_r} \ang{\sum_{s \geq 0} h_{s} \sum_{\abs{\bm{s}} = s} \nu_{\bm{s}} H_{2\bm{s}}  , P_W^\T H_{\nn}}_{\gamma_p} \\
  &=   \nabla_{h_r} \sum_{s \geq 0} h_{s} \sum_{\abs{\bm{s}} = s} \nu_{\bm{s}}  \ang{H_{2\bm{s}}  , P_W^\T H_{\nn}}_{\gamma_p} \\
  &=   \sum_{\abs{\rr} = r}\nu_{\rr}\ang{H_{2\rr}  , P_W^\T H_{\nn}}_{\gamma_p}.
\end{align*}
Plugging this into $e_r$, we get
\begin{align*}
  e_r &= \sum_{\nn \in \bbN^d} \ep_{\nn}  \sum_{\abs{\rr} = r}\nu_{\rr}\ang{H_{2\rr}  , P_W^\T H_{\nn}}_{\gamma_d}
  =  \sum_{\nn \in \bbN^d} \ep_{\nn} \ang{  P_W\xk{\sum_{\abs{\rr} = r}\nu_{\rr} H_{2\rr}},  H_{\nn}}_{\gamma_d}\\
  &= \sum_{\nn \in \bbN^d} \ep_{\nn} \ang{  P_W\bar{H}_r,  H_{\nn}}_{\gamma_d},
\end{align*}
where $\bar{H}_r = \sum_{\rr:\abs{\rr} = r}\nu_{\rr} H_{2\rr}$.

\paragraph{Dynamics of $\Sigma$.}
Using \cref{eq:SeqModel_Grad_Pop}, we have
\begin{equation*}
  \dot{W} = -\nabla_W^{\St} \bar{\caL}_n = - \nabla_W^{\St} \caL + \sum_{\nn \in \bbN^d} \ep_{\nn} \nabla_W^{\St} f_{\nn},
\end{equation*}
Consequently, following the computation in the population case, we have $\dot{\Rho} = \dot{W}^{\T} W_*$
and
\begin{equation*}
  \dot{\Sigma} = \Diag(U^\T \dot{\Rho} V)
  = (I_{p} - \Sigma^2) \Diag(\tilde{B}) + \Diag(U^\T (\sum_{\nn \in \bbN^d} \ep_{\nn} \nabla_W^{\St} f_{\nn})^\T W_* V ).
\end{equation*}
The last term is the error term that we need to analyze.
Let us denote
\begin{equation*}
  \Delta \coloneqq \sum_{\nn \in \bbN^d} \ep_{\nn} U^\T (\nabla_W^{\St} f_{\nn})^\T W_* V.
\end{equation*}
First, recalling the definition of $f_{\nn}$ and \cref{eq:RiemannianGradient}, we obtain
\begin{align*}
  \nabla_W^{\St} f_{\nn}
  = \nabla_W^{\St} \ang{P_W g, H_{\nn}}_{\gamma_d}
  = \proj_{T_W \St} \nabla_W \ang{P_W g, H_{\nn}}_{\gamma_d}
  = \proj_{T_W \St} Z_{\nn},
\end{align*}
where we denote $Z_{\nn} = \nabla_W \ang{P_W g, H_{\nn}}_{\gamma_d} \in \R^{d \times p}$.
Expanding $\proj_{T_W \St} Z_{\nn}$ with \cref{eq:RiemannianGradient}, we have
\begin{align*}
  \Delta &=  \sum_{\nn} \ep_{\nn} U^\T \zk{Z_{\nn} - W \Sym(W^\T Z_{\nn})}^\T W_* V \\
  &= \sum_{\nn} \ep_{\nn} U^\T \zk{Z_{\nn}^\T - \frac{1}{2} (W^\T Z_{\nn} + Z_{\nn}^\T W) W^\T} W_* V.
\end{align*}

Now, let us write
\begin{align*}
  & W_* = WQ + E,\quad Q \coloneqq UV^{\T}, \quad E \coloneqq W_* - W Q,
\end{align*}
we have
\begin{align*}
  W^{\T} E = W^{\T} W_* - Q = U \Sigma V^{\T} - U V^{\T} = U (\Sigma - I_p) V^{\T}.
\end{align*}
Then, we have
\begin{align*}
  \Delta &=  \sum_{\nn} \ep_{\nn} U^\T \zk{Z_{\nn}^\T - \frac{1}{2} (W^\T Z_{\nn} + Z_{\nn}^\T W) W^\T} (WU + EV) \\
  &= \sum_{\nn} \ep_{\nn} U^\T \zk{Z_{\nn}^\T W - \frac{1}{2} (W^\T Z_{\nn} + Z_{\nn}^\T W) W^\T W } U
  + \sum_{\nn} \ep_{\nn} U^\T \zk{Z_{\nn}^\T E V - \frac{1}{2} (W^\T Z_{\nn} + Z_{\nn}^\T W) W^\T E V}  \\
  &= \sum_{\nn} \ep_{\nn} U^\T \zk{\frac{1}{2} Z_{\nn}^\T W - \frac{1}{2} W^\T Z_{\nn}} U
  +\sum_{\nn} \ep_{\nn} U^\T \zk{Z_{\nn}^\T E V - \frac{1}{2} (W^\T Z_{\nn} + Z_{\nn}^\T W) U (\Sigma - I_p)} \\
  &= \frac{1}{2} \sum_{\nn} \ep_{\nn} U^\T \zk{\frac{1}{2} Z_{\nn}^\T W - \frac{1}{2} W^\T Z_{\nn}} U
  + \frac{1}{2}\sum_{\nn} \ep_{\nn} U^\T \zk{W^\T Z_{\nn} + Z_{\nn}^\T W} U (I - \Sigma)
  + \sum_{\nn} \ep_{\nn} U^\T Z_{\nn}^\T E V \\
  &= \sum_{\nn} \ep_{\nn} \xk{\Delta_{\nn}^{(0)} + \Delta_{\nn}^{(1)} + \Delta_{\nn}^{(2)}},
\end{align*}
where
\begin{align*}
  \Delta_{\nn}^{(0)} &= \frac{1}{2} U^\T \zk{\frac{1}{2} Z_{\nn}^\T W - \frac{1}{2} W^\T Z_{\nn}} U, \\
  \Delta_{\nn}^{(1)} &= \frac{1}{2} U^\T \zk{W^\T Z_{\nn} + Z_{\nn}^\T W} U (I - \Sigma), \\
  \Delta_{\nn}^{(2)} &= U^\T Z_{\nn}^\T E V.
\end{align*}

Now, let us further introduce
\begin{equation}
  \label{eq:Proof_Multi_TildeQuantities}
  \tilde{W} = WU,\quad \tilde{W}_* = W_* V= WU + EV,\quad
  \tilde{E} = \tilde{W}_* - \tilde{W} = W_* V - WU,\quad
  \tilde{Z}_{\nn} \coloneqq Z_{\nn} U.
\end{equation}
For $ \Delta_{\nn}^{(0)}$, it is easy to find that
\begin{align*}
  \xk{\Delta_{\nn}^{(0)}}^{\T} = -  \Delta_{\nn}^{(0)},
\end{align*}
so the diagonal of $\Delta_{\nn}^{(0)}$ is zero:
\begin{equation*}
  \Diag(\Delta_{\nn}^{(0)}) = 0.
\end{equation*}
For $\Delta_{\nn}^{(1)}$, since $\frac{1}{2} U^\T \zk{W^\T Z_{\nn} + Z_{\nn}^\T W} U = \Sym(U^\T W^\T Z_{\nn} U)$, we have
\begin{align*}
  \Diag(\Delta_{\nn}^{(1)}) &=
  \frac{1}{2} \Diag(U^\T \xk{W^\T Z_{\nn} + Z_{\nn}^\T W} U) (I - \Sigma) \\
  &= \Diag(U^\T W^\T Z_{\nn} U) (I - \Sigma) \\
  &= \Diag((I - \Sigma)\tilde{W}^\T \tilde{Z}_{\nn}) \\
  &= \Diag( (\tilde{W}(I - \Sigma))^\T \tilde{Z}_{\nn}).
\end{align*}
The last term $\Delta_{\nn}^{(2)}$ can be written as
\begin{align*}
  \Diag(\Delta_{\nn}^{(2)}) &= \Diag(U^\T Z_{\nn}^\T E V) = \Diag( (Z_{\nn} U)^\T  (W_* - WQ) V)
  =\Diag( (Z_{\nn} U)^\T  (W_* V - WU) ) \\
  &= \Diag( (\tilde{W}_* - \tilde{W})^\T \tilde{Z}_{\nn}).
\end{align*}
Consequently,
\begin{align*}
  \Diag(\Delta_{\nn}) &= \Diag( \Delta_{\nn}^{(0)} + \Delta_{\nn}^{(1)} + \Delta_{\nn}^{(2)}) \\
  &= \Diag( (\tilde{W}(I - \Sigma))^\T \tilde{Z}_{\nn}) + \Diag( (\tilde{W}_* - \tilde{W})^\T \tilde{Z}_{\nn}) \\
  &= \Diag\zk{(\tilde{W}(I - \Sigma) + \tilde{W}_* - \tilde{W})^\T \tilde{Z}_{\nn}} \\
  &= \Diag\xk{(\tilde{W}_* - \tilde{W} \Sigma)^\T \tilde{Z}_{\nn}}.
\end{align*}
Substituting back to $W,W_*$, we find that
\begin{align*}
  \tilde{W}_* - \tilde{W} \Sigma
  & = W_* V - WU \Sigma = (W_* - W U \Sigma V^{\T}) V = (W_* - W \Rho) V = (W_* - W W^{\T} W_*) V \\
  &= P_W^{\perp} W_* V,
\end{align*}
where $P_W^{\perp} = I - W W^{\T}$ is the projection orthogonal to $W$.
Hence,
\begin{equation*}
  \Diag(\Delta_\nn) =
  \Diag\xk{(\tilde{W}_* - \tilde{W} \Sigma)^\T \tilde{Z}_{\nn}}
  = \Diag(V^\T (P_W^{\perp} W_*)^\T Z_{\nn} U ).
\end{equation*}
In summary, we obtain that
\begin{equation*}
  \dot{\sigma}_i = (1 - \sigma_i^2) \sum_{r \geq 1} h_r h^*_r \sum_{\abs{\rr} = r} 2 r_i \sigma^{2 \rr-e_i} \nu_{\rr}^2
  + \sum_{\nn \in \bbN^d} \ep_{\nn} \Diag(V^\T (P_W^{\perp} W_*)^\T Z_{\nn} U )_{ii},
\end{equation*}
and also
\begin{align*}
  \dv{t}  \dot{\sigma}_i^2
  &= 4(1 - \sigma_i^2) \sum_{r \geq 1} h_r h^*_r \sum_{\abs{\rr} = r}  r_i \sigma^{2 \rr} \nu_{\rr}^2
  + \sum_{\nn \in \bbN^d} \ep_{\nn} \Diag(\Sigma V^\T (P_W^{\perp} W_*)^\T Z_{\nn} U )_{ii} \\
  &= (1 - \sigma_i^2) A_i + \sum_{\nn \in \bbN^d} \ep_{\nn} \Diag(\Sigma V^\T (P_W^{\perp} W_*)^\T Z_{\nn} U )_{ii},
\end{align*}
where $A_i = 4 \sum_{r \geq 1} h_r h^*_r \sum_{\abs{\rr} = r}  r_i \nu_{\rr}^2 \sigma^{2 \rr} $.

\paragraph{Dynamics of $\phi_1$ and $\omega$}

Using the chain rule, we can compute the dynamics of  $\phi_1$ and $\omega$.
It suffices to focus on the noise term.
Let us introduce $M = \Rho^{\T} \Rho$.
For $\phi_1$, the noise term writes
\begin{equation*}
  \xi = \sum_{i=1}^p \sum_{\nn \in \bbN^d} \ep_{\nn} \Diag(\Sigma V^\T (P_W^{\perp} W_*)^\T Z_{\nn} U )_{ii}
  = \sum_{\nn \in \bbN^d} \ep_{\nn} \Tr \zk{\Sigma V^\T (P_W^{\perp} W_*)^\T Z_{\nn} U}.
\end{equation*}
While for $\omega$, using the chain rule with \cref{prop:Aux__Softmin}, we have
\begin{align*}
  \zeta &= \frac{1}{\sum_{j=1}^p e^{-K \sigma_j^2}} \sum_{i = 1}^p e^{-K \sigma_i^2} \sum_{\nn \in \bbN^d} \ep_{\nn} \Diag(\Sigma V^\T (P_W^{\perp} W_*)^\T Z_{\nn} U )_{ii} \\
  &= \frac{1}{\sum_{j=1}^p e^{-K \sigma_j^2}} \sum_{i = 1}^p  \sum_{\nn \in \bbN^d} \ep_{\nn} \Diag(e^{-K \sigma_i^2} \Sigma V^\T (P_W^{\perp} W_*)^\T Z_{\nn} U )_{ii}\\
  &= \frac{1}{\sum_{j=1}^p e^{-K \sigma_j^2}} \sum_{\nn \in \bbN^d} \ep_{\nn} \Tr \zk{e^{-K \Sigma^2}\Sigma V^\T (P_W^{\perp} W_*)^\T Z_{\nn} U},
\end{align*}
where $e^{-K \Sigma^2}$ represents a matrix function.
Moreover, the denominator can be written as
\begin{equation*}
  \sum_{j=1}^p e^{-K \sigma_j^2} = \Tr e^{-K \Sigma^2} = \Tr \exp(-K M)
\end{equation*}
is independent of the SVD decomposition of $\Rho$.

Therefore, let us consider in general a function $\varphi$ and the noise term in the form of
\begin{equation*}
  \chi = \sum_{\nn \in \bbN^d} \ep_{\nn} \Tr \zk{\varphi(\Sigma^2) \Sigma V^\T (P_W^{\perp} W_*)^\T Z_{\nn} U}
  = \sum_{\nn \in \bbN^d} \ep_{\nn} T_{\nn}.
\end{equation*}
Then, we can express $T_{\nn}$ as
\begin{align*}
  T_{\nn} &= \Tr \zk{\varphi(\Sigma^2)\Sigma V^\T (P_W^{\perp} W_*)^\T Z_{\nn} U} \\
  &\stackrel{(a)}{=}\Tr \zk{U \varphi(\Sigma^2)\Sigma V^\T (P_W^{\perp} W_*)^\T \ang{\nabla_W  (P_W g), H_{\nn}}} \\
  &\stackrel{(b)}{=} \Tr \zk{U \Sigma \varphi(\Sigma^2) V^\T (P_W^{\perp} W_*)^\T \ang{x (P_W \nabla g)^\T, H_{\nn}}} \\
  &\stackrel{(c)}{=} \ang{\Tr\zk{\Rho \varphi(M) (P_W^{\perp} W_*)^\T x (P_W \nabla g)^\T }, H_{\nn}} \\
  &\stackrel{(d)}{=}  \ang{{(P_W \nabla g)^\T W^T W_* \varphi(M) (P_W^{\perp} W_*)^\T x  }, H_{\nn}} \\
  &\stackrel{(e)}{=}  \ang{{ (\nabla_x (P_W g))^\T W_* \varphi(M) (P_W^{\perp} W_*)^\T x  }, H_{\nn}} \\
  &= \ang{{ (\nabla_x (P_W g))^\T W_* \varphi(M) W_*^\T P_W^{\perp}  x  }, H_{\nn}},
\end{align*}
where we use the commutative property of the trace in $(a,d)$,
gradient formula \cref{lem:Multi_Prelim_Op_Grad} in $(b,e)$,
and matrix calculus \cref{eq:Multi_Prelim_MatrixCal} in $(c)$.
We note here that the final result is independent of the singular value decomposition.

\paragraph{Summary}

We summarize the results in the following proposition.
\begin{proposition}
  \label{prop:Multi_Seq_Equations}
  Consider the dynamics \cref{eq:Def_AdaK_Multi} under  \cref{assu:Multi_RotationInvariance}.
  Then, we have
  \begin{equation}
    \begin{aligned}
      \dot{h}_r &= \mu_{2r} \zk{\phi_r h_r^* - h_r + e_r}, \\
      \dot{\omega} &= \frac{1}{\sum_{j=1}^p e^{-K \sigma_j^2}} \sum_{i=1}^p e^{-K \sigma_i^2} (1 - \sigma_i^2) A_i + \zeta \\
      \dot{\phi}_1 &= \sum_{i=1}^p (1 - \sigma_i^2) A_i + \xi, \\
    \end{aligned}
  \end{equation}
  where $A_i = 4 \sum_{r \geq 1} h_r h^*_r \sum_{\abs{\rr} = r} r_i \nu_{\rr}^2 \sigma^{2 \rr},$
  and the noise terms are given by
  \begin{align*}
    e_r &= e_r(W) = \sum_{\nn \in \bbN^d} \ep_{\nn} \ang{  P_W\bar{H}_r,  H_{\nn}}_{\gamma_d}, \\
    \zeta &=\zeta(W,g) =
    \frac{1}{\Tr \exp(-K M)}
    \sum_{\nn \in \bbN^d} \ep_{\nn}\ang{{ (\nabla_x (P_W g))^\T W_* e^{-KM} W_*^\T P_W^{\perp}  x  }, H_{\nn}}, \\
    \xi &= \xi(W,g) = \sum_{\nn \in \bbN^d} \ep_{\nn}\ang{{ (\nabla_x (P_W g))^\T W_* W_*^\T P_W^{\perp}  x  }, H_{\nn}},
  \end{align*}
  where $M = \Rho^\T \Rho$.
\end{proposition}

\subsubsection{Bounding the perturbation terms}

Let us now apply \cref{lem:GaussianProcess_UniformBound} to bound the perturbation terms in \cref{prop:Multi_Seq_Equations}.
\begin{proposition}
  \label{prop:Multi_Seq_Error_e}
  Let $e_r$ be defined in \cref{prop:Multi_Seq_Equations}.
  Then,
  \begin{equation*}
    \kappa(W,W') \coloneqq \Cov(e_r(W), e_r(W')) = \frac{1}{n}\sum_{\abs{\rr} = r} \nu_{\rr}^2 \lambda^{2\rr},
  \end{equation*}
  where $\lambda$ is the singular values of $W^\T W'$.
  Moreover, $\kappa(W,W')$ is Lipschitz with respect to the Euclidean norm in $\R^{d\times p}$ with Lipschitz constant $2r$.
  Hence, with probability at least $1 - 4\exp(-d)$, we have
  \begin{equation}
    \label{eq:Multi_Seq_Error_e}
    \sup_{W \in \St(d,p)} \abs{e_r(W)}
    \lesssim \sqrt{\frac{dp \log r}{n}},\quad \forall r \geq 0.
  \end{equation}
\end{proposition}
\begin{proof}
  Recall that $\bar{H}_r$ is a rotation invariant polynomial of degree $2r$ from \cref{lem:Multi_RotInvExpansion}.
  We can compute the covariance function of $e_r(W)$:
  \begin{align*}
    \kappa(W,W') &= \Cov(e_r(W), e_r(W'))
    = \frac{1}{n} \sum_{\nn \in \bbN^d} \ang{  P_W\bar{H}_r,  H_{\nn}}_{\gamma_d}\ang{  P_{W'}\bar{H}_r,  H_{\nn}}_{\gamma_d} \\
    &= \frac{1}{n} \ang{P_W\bar{H}_r,
      P_{W'}\bar{H}_r}_{\gamma_d} = \frac{1}{n}\ang{\bar{H}_r, P_W^\T P_{W'} \bar{H}_r}_{\gamma_p} \\
    &=  \frac{1}{n}  \ang{\bar{H}_r, \caA_{W^\T W'} \bar{H}_r}_{\gamma_p}.
  \end{align*}
  Let us consider the SVD $W^\T W' = Q_1 \Lambda Q_2^\T$ and let $\lambda$ be the diagonal entries of $\Lambda$.
  Then,
  \begin{align*}
    \kappa(W,W') &= \frac{1}{n} \ang{\bar{H}_r, \caA_{Q_1 \Lambda  Q_2^\T} \bar{H}_r}_{\gamma_p}
    = \frac{1}{n} \ang{\bar{H}_r, \caA_{Q_1} \caA_\Lambda \caA_{Q_2^\T} \bar{H}_r}_{\gamma_p} \\
    &= \frac{1}{n} \ang{\bar{H}_r, \caA_\Lambda \bar{H}_r}_{\gamma_p} \\
    &= \frac{1}{n}\sum_{\abs{\rr} = r} \nu_{\rr}^2 \lambda^{2\rr}.
  \end{align*}

  Let us now view $\kappa$ as a binary function on $\dk{W \in \R^{d \times p} : \norm{W} \leq 1}$ with the last expression.
  We can compute the derivative of $k$ with respect to $W$.
  Take a tangent direction $H \in \R^{d \times p}$, we use \cref{eq:Multi_Prelim_SVD_DirectionalDerivative} to obtain
  \begin{equation*}
    D_H \Lambda = \Diag(U^{\T} D_H(W^\T W') V) = \Diag(U^{\T} H^\T W' V),\qquad
    D_H \lambda_j = u_j^{\T} H^\T W' v_j,
  \end{equation*}
  so
  \begin{align*}
    \abs{D_H \lambda_j} \leq \norm{H} \norm{W'} \leq \norm{H}.
  \end{align*}
  Then, using
  \begin{equation*}
    D_H \lambda^{2\rr}
    = \sum_{j=1}^p 2 r_j \lambda^{2\rr-e_j} D_H \lambda_j,
  \end{equation*}
  we have
  \begin{equation*}
    \abs{D_H \lambda^{2\rr}}
    \leq 2 \sum_{j=1}^p r_j \lambda^{2\rr-e_j} \abs{D_H \lambda_j}
    \leq 2 \sum_{j=1}^p r_j \norm{H}
    = 2\abs{\rr} \norm{H}.
  \end{equation*}
  Finally,
  \begin{align*}
    \abs{D_{H} k(W,W')} &= \abs{\sum_{\abs{\rr} = r} \nu_{\rr}^2 D_{H} \lambda^{2\rr}}
    \leq \sum_{\abs{\rr} = r} \nu_{\rr}^2  \abs{D_{H} \lambda^{2\rr}} \\
    & \leq 2 \sum_{\abs{\rr} = r} \nu_{\rr}^2 \abs{\rr} \norm{H}
    \leq 2 r \norm{H} \sum_{\abs{\rr} = r} \nu_{\rr}^2 \\
    & = 2 r \norm{H} \leq 2 r \norm{H}_2.
  \end{align*}
  Therefore, we find that
  \begin{equation*}
    \norm{D k(W,W')}_2 \leq 2 r,
  \end{equation*}
  which implies that $k(W,W')$ is Lipschitz with respect to the Euclidean norm.
  Applying \cref{lem:GaussianProcess_UniformBound}, we obtained the desired bound.
\end{proof}

Before bounding $\xi$ and $\zeta$, let us first make some preliminary computation.
Let us consider a noise term in the form of
\begin{equation*}
  \chi = \sum_{\nn \in \bbN^d} \ep_{\nn}\ang{{ (\nabla_x (P_W \bar{H}_r))^\T W_* \varphi(M) W_*^\T P_W^{\perp}  x  }, H_{\nn}}.
\end{equation*}
We can decompose $g$ as
\begin{equation*}
  g = \sum_{r \geq 0} \sum_{\mm = 2 \rr,~\abs{\rr} = r} g_{\mm} H_{\mm}
  = \sum_{r \geq 0}  h_r \bar{H}_r,\quad \bar{H}_r = \sum_{\abs{\rr} = r} \nu_{\rr} H_{2\rr}.
\end{equation*}
Then,  we have
\begin{equation*}
  \chi = \sum_{r \geq 0} h_r \chi_r,\quad
  \chi_r =
  \sum_{\nn \in \bbN^d} \ep_{\nn}\ang{{ (\nabla_x (P_W \bar{H}_r))^\T W_* \varphi(M) W_*^\T P_W^{\perp}  x  }, H_{\nn}},
\end{equation*}
where $\chi_r$ is independent of $g$.

Let us take $W'$ and denote the corresponding quantities with a prime (such as $\Rho',M'$).
We compute the covariance of $\chi_r(W)$ and $\chi_r(W')$.
\begin{align*}
  k_{\chi,r}(W,W') &= \Cov(\chi_r(W), \chi_r(W')) \\
  &=  \ang{{ (\nabla_x (P_W \bar{H}_r))^\T W_* \varphi(M) W_*^\T P_W^{\perp}  x  },
      { (\nabla_x (P_{W'} \bar{H}_r))^\T W_* \varphi(M') W_*^\T P_{W'}^{\perp}  x  }}_{\gamma_d}
\end{align*}
where we use the fact that $(H_{\nn})_{\nn \in \bbZ^d}$ is an orthogonal basis.
Noticing that
\begin{equation*}
  P_W^{\perp} \nabla_x (P_W \bar{H}_r) = P_W^{\perp} W P_W (\nabla \bar{H}_r) = 0,
\end{equation*}
we can use \cref{lem:Gaussian_Stein3} with $f = P_W \bar{H}_r$, $A = W_* \varphi(M) W_*^\T P_W^{\perp}$
to get
\begin{equation*}
  k_{\chi,r}(W,W') = k_{\chi,r}^{(1)}(W,W')  + k_{\chi,r}^{(2)}(W,W'),
\end{equation*}
where
\begin{align*}
  k_{\chi,r}^{(1)}(W,W') &= \E (\nabla_x (P_W \bar{H}_r))^\T W_* \varphi(M) W_*^\T P_W^{\perp} (W_* \varphi(M') W_*^\T P_{W'}^{\perp})^{\T}\nabla_x (P_{W'} \bar{H}_r) \\
  &=\E  (P_W \nabla_x \bar{H}_r)^\T W^\T W_* \varphi(M) W_*^\T P_W^{\perp} P_{W'}^{\perp} W_* \varphi(M') W_*^\T W' (P_{W'} \nabla_x \bar{H}_r) \\
  &= \E (P_W \nabla \bar{H}_r)^\T \Rho \varphi(M) W_*^\T P_W^{\perp} P_{W'}^{\perp} W_* \varphi(M') (\Rho')^\T (P_{W'} \nabla \bar{H}_r)
\end{align*}
and
\begin{align*}
  k_{\chi,r}^{(2)}(W,W') & = \E \Tr \nabla_x^2 (P_W \bar{H}_r) \xk{W_* \varphi(M) W_*^\T P_W^{\perp}}
  \nabla_x^2 (P_{W'} \bar{H}_r) (W_* \varphi(M') W_*^\T P_{W'}^{\perp}) \\
  &= \E \Tr W (P_W \nabla^2 \bar{H}_r) {W^\T W_* \varphi(M) W_*^\T P_W^{\perp}}
  W' (P_{W'} \nabla^2 \bar{H}_r)(W')^\T W_* \varphi(M') W_*^\T P_{W'}^{\perp}.
\end{align*}

We can further compute the variance.
We find that $P_W^\perp W' = 0$ when $W = W'$, so
\begin{align*}
  k_{\chi,r}(W,W)  &= k_{\chi,r}^{(1)}(W,W)
  = \E (P_W \nabla \bar{H}_r)^\T \Rho \varphi(M) W_*^\T P_W^{\perp}  W_* \varphi(M) \Rho^\T (P_{W} \nabla \bar{H}_r) \\
  &= \E (P_W \nabla \bar{H}_r)^\T \Rho \varphi(M) (I - M) \varphi(M) \Rho^\T (P_{W} \nabla \bar{H}_r) \\
  &\stackrel{(a)}{=} \Tr \Rho \varphi(M) (I - M) \varphi(M) \Rho \zk{\E (\nabla \bar{H}_r) (\nabla \bar{H}_r)^{\T}} \\
  &= \frac{2r}{p} \Tr \Rho \varphi(M) (I - M) \varphi(M) \Rho^\T \\
  &= \frac{2r}{p} \Tr \varphi(M)^2 M (I-M).
\end{align*}


\begin{proposition}
  \label{prop:Multi_Seq_Error_Xi}
  With probability at least $1 - 4\exp(-d)$, it holds that for any $W \in \St(d,p)$ satisfying
  $\phi_1(W) \leq 1 - \nu$ and any $g$,
  \begin{equation}
    \label{eq:Multi_Seq_Error_Xi}
    \abs{\xi(W,g)} \lesssim
    (1 - \phi_1)^{\hf} \sqrt {\frac{pd}{n}} \sqrt {\log (p\nu^{-1})} \sum_{r \geq 1} (r \log r)^{\hf} \abs{h_r}.
  \end{equation}
\end{proposition}
\begin{proof}
  Taking $\varphi \equiv 1$ in the previous computation, we find that
  \begin{equation*}
    k_{\xi,r}(W,W) = \frac{2r}{p} \Tr (I- \Sigma^2) \Sigma^2
    \leq \frac{2r}{p} \Tr (I- \Sigma^2) = 2r (1-\phi_1).
  \end{equation*}

  To bound the derivative of $k_{\xi,r}(W,W')$, we can use $\Tr A \leq p\norm{A}$ and the chain rule of derivative iteratively.
  The quantities $\Rho = W^\T W_*$ and $P_W^\T = I - WW^\T$ are all polynomial in $W$.
  Also, we recall that $\nabla_W (P_W h) = x P_W (\nabla g)^{\T}$.
  Combining this with the derivative formula of Hermite polynomials and that $\norm{W} \leq 1$,
  the derivative of each entry in $P_W (\nabla^2 \bar{H}_r)$ is bounded by a polynomial of degree at most $2r+1$ with coefficients being a polynomial of $r$,
  so its $L^2(\gamma_d)$ norm is also bounded by a polynomial of $r$.
  Consequently, we can deduce that $\norm{\nabla_W k_{\xi,r}(W,W')}_2$ is bounded by some polynomial in $p,r$.
  Moreover, since
  \begin{equation*}
    \phi_1 = \sum_{i=1}^p \sigma_i^2 = \Tr \Rho^{\T} \Rho,
  \end{equation*}
  the derivative of $\phi_1$ is also bounded by a polynomial in $p,r$.
  Let us take $\bar\xi = \frac{1}{2r}(1-\phi_1)^{-\hf} \xi$ and
  the scaled covariance
  \begin{equation*}
    \bar{k}_{\xi,r}(W,W')
    \coloneqq \frac{1}{2r} \zk{(1-\phi_1)(1-\phi_1')}^{-\hf} k_{\xi,r}(W,W').
  \end{equation*}
  Then, as long as $1-\phi_1 \geq \nu$,
  the derivative of $\bar{k}_{\xi,r}(W,W')$ is bounded by some polynomial of $p,r,\nu^{-1}$.
  Finally, we can apply \cref{lem:GaussianProcess_UniformBound} together with a union bound to obtain that
  \begin{equation*}
    \abs{\bar{\xi}_r} \lesssim \sqrt{\log (r p \nu^{-1}) + C}\sqrt {pd},\quad \forall r \geq 1
  \end{equation*}
  with probability at least $1-4\exp(-d)$.
  Returning to $\xi_r$ and taking a summation over $r$ yields the desired bound.
\end{proof}

\begin{proposition}
  \label{prop:Multi_Seq_Error_Zeta}
  With probability at least $1 - 4\exp(-d)$, it holds that for any $W \in \St(d,p)$ and any $g$,
  \begin{equation}
    \label{eq:Multi_Seq_Error_Zeta}
    \abs{\zeta(W,g)} \lesssim \sqrt {\frac{pd}{n}} \sqrt {\log (pK)} \sum_{r \geq 1} (r \log r)^{\hf} \abs{h_r}.
  \end{equation}
\end{proposition}
\begin{proof}
  Taking $\varphi(t) = e^{-Kt}$ in the previous computation, we find that
  \begin{equation*}
    k_{\zeta,r}(W,W') = k_{\zeta,r}^{(1)}(W,W')  + k_{\zeta,r}^{(2)}(W,W') ,
  \end{equation*}
  where
  \begin{align*}
    k_{\zeta,r}^{(1)}(W,W') &= q \E (P_W \nabla \bar{H}_r)^\T \Rho e^{-KM} W_*^\T P_W^{\perp} P_{W'}^{\perp} W_* e^{-KM'} (\Rho')^\T (P_{W'} \nabla \bar{H}_r) \\
    k_{\zeta,r}^{(2)}(W,W')&= q \E \Tr W (P_W \nabla^2 \bar{H}_r) {W^\T W_* e^{-KM} W_*^\T P_W^{\perp}}
    W' (P_{W'} \nabla^2 \bar{H}_r)(W')^\T W_* e^{-KM'} W_*^\T P_{W'}^{\perp}, \\
    q &= \zk{(\Tr \exp(-K M)) (\Tr \exp(-K M'))}^{-1}
  \end{align*}
  For the case $W=W'$, we have
  \begin{align*}
    k_{\zeta,r}(W,W') &= \frac{2r}{p} (\Tr \exp(-K M))^{-2} \Tr \exp(-2K M) M(I-M) \\
    &=  \frac{2r}{p} (\Tr \exp(-K \Sigma^2))^{-2} \Tr \exp(-2K \Sigma^2) \Sigma^2 (I - \Sigma^2) \\
    & \leq \frac{2r}{p} (\Tr \exp(-K \Sigma^2))^{-2}\Tr \exp(-2K \Sigma^2) \\
    &\leq 2r.
  \end{align*}
  For the derivative, we can use the same argument as in the previous proof.
  The only difference is that we have to consider additionally the matrix derivative of
  $f(M) = (\Tr e^{-KM})^{-1} e^{-KM} $.
  We have
  \begin{equation*}
    \dv{t} f(M) =  (\Tr e^{-KM})^{-1} \dv{t} e^{-KM} - (\Tr e^{-KM})^{-2} (\dv{t} \Tr  e^{-KM}) e^{-KM}
    = I + II.
  \end{equation*}
  The second term is easy to bound.
  Using the property of matrix derivative and trace, we have
  \begin{equation*}
    \dv{t} \Tr e^{-KM}= -K \Tr e^{-KM} \dot{M},
  \end{equation*}
  so
  \begin{align*}
    \norm{(\Tr e^{-KM})^{-2} (\dv{t} \Tr  e^{-KM}) e^{-KM}}
    &= K (\Tr e^{-KM})^{-2} \abs{\Tr e^{-KM} \dot{M}} \norm{e^{-KM}}\\
    &\leq K (\Tr e^{-KM})^{-2} \Tr e^{-KM} \norm{\dot{M}} \Tr e^{-KM}\\
    &= K \norm{\dot{M}},
  \end{align*}
  where we use the fact that $\abs{\Tr A B} \leq \norm{AB}_1 \leq \Tr A \norm{B}$ for positive semi-definite matrices $A$,
  where $\norm{\cdot}_1$ is the trace norm.
  For the first term, we use the matrix derivative formula that
  \begin{equation*}
    \dv{t} \phi(M) = Q (R\odot (Q^\T \dot{M} Q)) Q^\T,
  \end{equation*}
  where $M = Q \Lambda Q^\T$ is the spectral decomposition of $M$,
  \begin{equation*}
    R_{ij} =
    \begin{cases}
      \frac{\phi(\lambda_i) - \phi(\lambda_j)}{\lambda_i - \lambda_j} & \text{if } \lambda_i \neq \lambda_j \\
      \phi'(\lambda_i) & \text{if } \lambda_i = \lambda_j
    \end{cases}
  \end{equation*}
  In our case, we note that $\phi(t) = e^{-Kt}$, $\phi'(t) = -K e^{-Kt}$, so using the mean value theorem, we have
  \begin{equation*}
    R_{ij} = \phi'(\xi_{ij}) = -K e^{-K \xi_{ij}},\qquad
    \abs{R_{ij}}  \leq K e^{-K \lambda_{\min}},
  \end{equation*}
  where $\xi_{ij}$ is between $\lambda_i$ and $\lambda_j$ and $\lambda_{\min} = \min_{i} \lambda_i$.
  Consequently,
  \begin{align*}
    \norm{\dv{t} e^{-KM}}_2
    &= \norm{Q (R\odot (Q^\T \dot{M} Q)) Q^\T}_2
    = \norm{R\odot (Q^\T \dot{M} Q)}_2 \\
    &\leq K e^{-K \lambda_{\min}} \norm{Q^\T \dot{M} Q}_2
    \leq K e^{-K \lambda_{\min}} \norm{\dot{M}}_2,
  \end{align*}
  and hence
  \begin{equation*}
    \norm{I} \leq \norm{I}_2 \leq K \frac{e^{-K \lambda_{\min}} }{\Tr e^{-KM}} \norm{\dot{M}}_2 \leq K \norm{\dot{M}}_2.
  \end{equation*}
  In summary, we have
  \begin{equation*}
    \norm{\dv{t} f(M)} \leq 2 K \norm{\dot{M}}_2.
  \end{equation*}
  Consequently, we can conclude that the derivative of $k_{\zeta,r}(W,W')$ is bounded by a polynomial in $p,r$.
  Applying \cref{lem:GaussianProcess_UniformBound} together with a union bound, we obtain the desired bound.

\end{proof}

\subsubsection{Training dynamics}
\begin{proposition}
  \label{prop:Multi_Seq_Init}
  Under \cref{assu:Multi_InformationIndex}, assume $n \gtrsim d^{2m_0+1 + s}$ for some $s > 0$.
  Let \cref{eq:Multi_Seq_Error_e} and \cref{eq:Multi_Seq_Error_Zeta} hold.
  Let $\delta > 0$ be fixed.
  Then, with probability at least $1 - \delta$, when $n,d$ is large enough,
  we have
  \begin{equation}
    \label{eq:Multi_Seq_Init}
    \min_{1 \leq i \leq p} \sigma_i^2(t) \geq 1/2,\quad
    \forall t \in [T^{\mr{app}}, T^{\mr{fin}}],
  \end{equation}
  where $T^{\mr{app}} \lesssim \log d + d^{m_0 - 1}$ and $T^{\mr{fin}}$ can be taken as any fixed polynomial in $d$,
  and the constant in the $\lesssim$ notation can depend on $\delta$.
\end{proposition}
\begin{proof}
  The proof follows similar strategies as in \cref{prop:SIM_Seq_Init} and \cref{prop:Multi_Pop_AppTime}.
  Without loss of generality, we can assume that $h_{r} \geq 0$ for all $r$.

  First, according to \cref{prop:Multi_RandomInit}, there is some $c = c(\delta)$ such that with probability at least $1 - \delta$,
  \begin{equation*}
    \min_{1 \leq i \leq p} \sigma_i^2 \geq c d^{-1}.
  \end{equation*}
  Let us take $\omega$ and $\phi_1$ as in \cref{eq:Multi_Proof_PhiOmega}.
  According to \cref{prop:Aux__Softmin}, by taking $K \asymp C d \log p$, we have
  \begin{equation*}
    \omega(0) \geq c d^{-1} - \frac{1}{K}\log p \geq c d^{-1},
  \end{equation*}
  Moreover, we claim that we will have $\omega(t) \geq \rho_0 \coloneqq \frac{1}{2} c d^{-1}$ for the range of $t$ we are interested in.
  We will prove this claim later.

  Recall the dynamics in \cref{prop:Multi_Seq_Equations}.
  Let us take $\rz = m_0/2$.
  Noticing that when $\omega(t) \geq \rho_0$, we have
  \begin{equation*}
    \phi_{\rz}(t) \geq \omega(t)^{\rz} \gtrsim d^{-\rz} \gtrsim \sqrt {\frac{d}{n}},
  \end{equation*}
  since we have $n \gtrsim d^{2m_0+1 + s}$.
  Consequently, we have
  \begin{equation}
    \label{eq:Proof__Multi_Seq_hrz}
    \dot{h}_{\rz} = \mu_{2\rz} (\phi_{\rz} h_{\rz}^* - h_{\rz} + e_{\rz})
    \geq \mu_{2\rz}( \phi_{\rz} h_{\rz}^*/2 - h_{\rz}),
  \end{equation}
  and thus
  \begin{equation*}
    T_0^g \coloneqq \inf \dk{t \geq 0 : h_{\rz}(t) \geq \frac{1}{4} \rho_0^{\rz} h_{\rz}^*} \lesssim 1.
  \end{equation*}

  Now let us prove the claim for $t \leq T_0^g$.
  We recall the dynamics of $\omega$ in \cref{prop:Multi_Seq_Equations}:
  \begin{equation*}
    \dot{\omega} = \frac{1}{\sum_{j=1}^p e^{-K \sigma_j^2}} \sum_{i=1}^p e^{-K \sigma_i^2} (1 - \sigma_i^2) A_i + \zeta,
  \end{equation*}
  where $A_i = 4 \sum_{r \geq 1} h_r h^*_r \sum_{\abs{\rr} = r} r_i \nu_{\rr}^2 \sigma^{2 \rr}$.
  Similar to \cref{prop:SIM_Seq_Control_g}, we can show that
  \begin{equation*}
    h_r(t) \geq -C \min(1,\mu_{2r} t) \abs{e_r} \geq -C \min(1,\mu_{2r} t) \sqrt {\frac{d \log r}{n}}.
  \end{equation*}
  Also, similar to \cref{eq:SIM_Seq_Control_tau}, we have
  \begin{equation}
    \label{eq:Proof__Multi_Seq_Bound_Series}
    \sum_{r \geq 1} (r \log r)^{\hf} \abs{h_r} \lesssim (1 + \log^+ t)^{2},
  \end{equation}
  so the error term \cref{eq:Multi_Seq_Error_Zeta} is further bounded by
  \begin{equation}
    \abs{\zeta(W,g)} \lesssim \sqrt {\frac{pd}{n}} \sqrt {\log (pK)} \sum_{r \geq 1} (r \log r)^{\hf} \abs{h_r}
    \lesssim \sqrt {\frac{d \log d}{n}}(1 + \log^+ t)^2
  \end{equation}
  Therefore,
  \begin{align}
    \notag
    A_i & \gtrsim - \sum_{r \geq 1} h^*_r \min(1,\mu_{2r} t) \sqrt {\frac{d \log r}{n}} \sum_{\abs{\rr} = r} r_i \nu_{\rr}^2 \sigma^{2 \rr}\\
    \notag
    & \gtrsim - \sqrt {d/n} \sum_{r \geq 1} h^*_r r (\log r)^{\hf} \min(1,\mu_{2r} t) \\
    \label{eq:Proof__Multi_Seq_Bound_Ai}
    & \gtrsim - \sqrt {d/n}  (1 + \log^+ t)^2
  \end{align}
  where we use $\mu_r \lesssim e^{-\gamma r}$ and \cref{prop:Series_2} in the last step.
  Plugging these two bound into the dynamics of $\dot{\omega}$, we find that
  \begin{align*}
    \dot{\omega}
    &\gtrsim -\frac{1}{\sum_{j=1}^p e^{-K \sigma_j^2}} \sum_{i=1}^p e^{-K \sigma_i^2} \sqrt {d/n} (1 + \log^+ t)^2
    -\sqrt {\frac{d \log d}{n}}(1 + (\log^+ t)^2) \\
    &\gtrsim -\sqrt {\frac{d \log d}{n}} (1 + \log^+ t)^2
  \end{align*}
  Consequently, since $T_0^g \lesssim 1$ and $n \gtrsim d^{2m_0 + 1+ s}$, we have
  \begin{equation*}
    \omega(T_0^g) \geq \omega(0) - \sqrt {\frac{d \log d}{n}} T_0^g \geq c d^{-1} - \sqrt {\frac{d \log d}{n}}
    \geq \rho_0,
  \end{equation*}
  which proves the claim until $T_0^g$.

  After that, we have
  \begin{align*}
    h_{\rz} h_{\rz}^* \sum_{\abs{\rr} = \rz} r_i \nu_{\rr}^2 \sigma^{2 \rr}
    & \geq \rho_0^{\rz} (h_{\rz}^*)^2  \sum_{\abs{\rr} = \rz} r_i \nu_{\rz}^2 \rho_0^{\rz} \\
    &=  c  (h_{\rz}^*)^2  \rho_0^{2\rz}  \gtrsim d^{-2\rz},
  \end{align*}
  so since  $n \gtrsim d^{2m_0 + 1+ s}$, we have
  \begin{align}
    \notag
    A_i &= 4 h_{\rz} h_{\rz}^* \sum_{\abs{\rr} = \rz} r_i \nu_{\rr}^2 \sigma^{2 \rr} + 4 \sum_{r \neq \rz} h_r h^*_r \sum_{\abs{\rr} = r} r_i \nu_{\rr}^2 \sigma^{2 \rr}\\
    \notag
    &\geq 4 h_{\rz} h_{\rz}^* \sum_{\abs{\rr} = \rz} r_i \nu_{\rr}^2 \sigma^{2 \rr} - \sqrt {d/n}  (1 + \log^+ t)^2 \\
    \label{eq:Proof__Multi_Seq_Bound_Ai2}
    & \geq 2 h_{\rz} h_{\rz}^* \sum_{\abs{\rr} = \rz} r_i \nu_{\rr}^2 \sigma^{2 \rr},
  \end{align}
  provided that $t$ is at most polynomially large in $d$ and $n$.
  Thus,
  \begin{equation}
    \label{eq:Proof__Multi_Seq_omega}
    \dot{\omega} \geq \frac{2}{\sum_{j=1}^p e^{-K \sigma_j^2}} \sum_{i=1}^p e^{-K \sigma_i^2} (1 - \sigma_i^2) h_{\rz} h_{\rz}^* \sum_{\abs{\rr} = \rz} r_i \nu_{\rr}^2 \sigma^{2 \rr}
    \geq 0.
  \end{equation}
  Consequently, with \cref{eq:Proof__Multi_Seq_hrz} and \cref{eq:Proof__Multi_Seq_omega},
  we can apply the same argument as in \cref{prop:Multi_Pop_AppTime}.
  Defining similar times $T_k^g$ and $T_k^\rho$, we have $T^{\mr{app}} = T_L^{\rho} \lesssim \log d + d^{m_0-1}$,
  where $L = 1 + \cek{\log_2 \rho_0^{-1}}$.
\end{proof}

\begin{proposition}
  \label{prop:Multi_Seq_Conv}
  Under \cref{assu:Multi_InformationIndex}, assume $n \gtrsim d^{2m_0+1 + s}$ for some $s > 0$.
  Let \cref{eq:Multi_Seq_Error_e}, \cref{eq:Multi_Seq_Error_Xi} hold with $\nu = 1/n$.
  Suppose \cref{eq:Multi_Seq_Init} holds.
  Then, there is some $t_1 \leq T^{\mr{app}} + C$ such that
  \begin{equation*}
    1 - \phi_1(t_1 + s) \lesssim \exp(- c s),
  \end{equation*}
  provided that
  \begin{equation*}
    1 - \phi_1(t_1+s) \gtrsim \frac{p d \cdot \polylog(n,d,p)}{n}.
  \end{equation*}
\end{proposition}
\begin{proof}
  First, using the dynamics \cref{eq:Proof__Multi_Seq_hrz}, we find that we will have $h_{\rz} \geq c h^*_{\rz}$ after $t_1 = T^{\mr{app}} + C$.
  Then, using \cref{eq:Proof__Multi_Seq_Bound_Ai2} and \cref{eq:Multi_Seq_Init}, we have
  \begin{equation*}
    A_i \gtrsim h_{\rz} h_{\rz}^* \sum_{\abs{\rr} = \rz} r_i \nu_{\rr}^2 \sigma^{2 \rr}
    \gtrsim 1.
  \end{equation*}
  Consequently, the dynamics of $\phi_1$ in \cref{prop:Multi_Seq_Equations} gives
  \begin{equation*}
    \dot{\phi}_1 \geq c \sum_{i=1}^p (1 - \sigma_i^2) - \abs{\xi} = c(1 - \phi_1) - \abs{\xi}.
  \end{equation*}
  On the other hand, combining \cref{eq:Multi_Seq_Error_Xi} with \cref{eq:Proof__Multi_Seq_Bound_Series} and that the time is at most polynomially large in $d$,
  we have
  \begin{align*}
    \abs{\xi} \lesssim (1 - \phi_1)^{\hf} \sqrt {\frac{dp \polylog(n,d)}{n}}.
  \end{align*}
  Therefore, as long as $1 - \phi_1 \gtrsim \frac{dp \polylog(n,d,p)}{n}$, we have $\abs{\xi} \leq c(1 - \phi_1)/2$, so
  $\dot{\phi}_1 \geq c(1 - \phi_1)/2$,
  yielding the desired result.
\end{proof}

\begin{proof}[Proof of \cref{thm:Multi_Seq}]
  The initial feature error measure is an easy consequence of \cref{prop:Multi_RandomInit}.
  For the decay of the feature error measure, se apply \cref{prop:Multi_Seq_Init} and \cref{prop:Multi_Seq_Conv} with error bounds \cref{prop:Multi_Seq_Error_e}, \cref{prop:Multi_Seq_Error_Xi} and \cref{prop:Multi_Seq_Error_Zeta}.
  For the final feature error measure, we can use \cref{prop:Multi_FEM_ProjBound} and the fact that
  $1 - \phi_1 \leq \nu$ implies $1 - \min_j \sigma_j^2 \leq p \nu$.
\end{proof}


\section{Auxiliary results}

\subsection{Random process}

\begin{definition}
  Let $X(\cdot)$ be a random process on a metric space $(T,d)$.
  We say that $X(\cdot)$ is $\sigma^2$-sub-Gaussian if
  \begin{align*}
    \norm{X(t) - X(s)}_{\psi_2}^2 \leq \sigma^2 d(s,t)^2,\quad \text{for all $s,t \in T$}.
  \end{align*}
\end{definition}

The following theorem is standard result on the supremum of sub-Gaussian processes; see, e.g., \citet[Section 8.1]{vershynin2018_HighdimensionalProbability}.

\begin{theorem}[Dudley's entropy integral]
  Let $X$ be a $\sigma^2$-sub-Gaussian random process on a metric space $(T,d)$ and $\E X(t) = 0$.
  Let $\caN(\ep,T,d)$ be the covering number of $T$ with respect to $d$.
  Take $D = \mr{diam}(T,d)$ the diameter of $T$ with respect to $d$
  and define the Dudley integral as
  \begin{equation}
    I = \int_0^{D} \sqrt{\log \caN(\ep,T,d)} \dd \ep.
  \end{equation}
  Then, we have
  \begin{align}
    \E \sup_{t \in T} X(t) \leq C \sigma I.
  \end{align}
  Also, for any $u \geq 0$, with probability at least $1 - 2\exp(-u^2)$, we have
  \begin{align}
    \sup_{s,t \in T} \abs{X(s) - X(t)} \leq C \sigma \xk{I + D u}.
  \end{align}
\end{theorem}

Let us now consider a mean-zero Gaussian process $X(t), t \in T$ for $T \subseteq \R^p$.
The covariance function of $X(\cdot)$ is given by $k(s,t) = \E X(s) X(t)$.
The induced metric is given by
\begin{equation*}
  d(s,t) = \sqrt{k(s,s) + k(t,t) - 2 k(s,t)}.
\end{equation*}
We have the following result on the covering number of $T$ with respect to $d$.

\begin{lemma}
  \label{lem:GaussianProcess_UniformBound}
  Let $X(t), t \in T \subseteq \R^p$ be a mean-zero Gaussian process with covariance function $k(s,t)$.
  Suppose that $k(t,t) \leq 1$ for all $t \in T$ and $k$ is Hölder continuous with exponent $\alpha$ on the diagonal:
  \begin{align*}
    \abs{k(t,t) - k(s,t)} \leq L \norm{t-s}^{\alpha},\quad \forall s,t \in T.
  \end{align*}
  Let $R = \mr{diam}(T, \norm{\cdot})$ be the diameter of $T$ with respect to the Euclidean norm.
  Then,
  \begin{equation*}
    \E \sup_{t \in T} X(t) \leq C \alpha^{-1/2} (\sqrt{\log R + \log L + C}) \sqrt{p}.
  \end{equation*}
  Also, for any $u \geq 0$, with probability at least $1 - 4\exp(-u^2)$, we have
  \begin{equation*}
    \sup_{t \in T} \abs{X(t)} \leq C \alpha^{-1/2} (\sqrt{\log R + \log L + C}) \sqrt{p} + C u.
  \end{equation*}
\end{lemma}
\begin{proof}
  Let us denote by \(d(s,t)\) the induced metric introduced by the Gaussian process.
  Since $k(t,t) \leq 1$, basic property of the covariance function gives
  \( \abs{k(s,t)} \leq \sqrt{k(s,s) k(t,t)} \leq 1 \),
  so that
  \( d(s,t) \leq 2 \),
  which implies that the diameter \( \mr{diam}(T,d) \leq 2 \).
  On the other hand, by the Hölder continuity assumption, we have
  \begin{align*}
    d(s,t) &= \sqrt{k(s,s) + k(t,t) - 2 k(s,t)} \leq \sqrt{\abs{k(s,s) - k(s,t)}+ \abs{k(t,t) - k(s,t)}} \\
    & \leq \sqrt{2L \norm{t - s}^{\alpha}}
    = \sqrt{2L} \norm{t - s}^{\alpha/2}.
  \end{align*}
  This shows that
  \begin{equation*}
    \caN(\ep, T, d) \leq  \caN( (\ep/\sqrt{2L})^{2/\alpha}, T, \norm{\cdot}).
  \end{equation*}
  Now, the standard result on the covering number of $\R^p$ gives
  $\caN(\delta, T, \norm{\cdot}) \leq (C R/\delta)^p$,
  so
  \begin{align*}
    \caN( (\ep/\sqrt{2L})^{2/\alpha}, T, \norm{\cdot}) \leq (C R/(\ep/\sqrt{2L})^{2/\alpha})^p = \xk{C R (2L)^{1/\alpha}}^p \ep^{-2p/\alpha}.
  \end{align*}
  Plugging this into Dudley's entropy integral, we find that
  \begin{align*}
    \int_0^{2} \sqrt{\log \caN(\ep,T,d)} \dd \ep
    &\leq \int_0^{2} \sqrt{\log \xk{C R L^{1/\alpha}}^p (\ep/\sqrt{2})^{-2p/\alpha}} \dd \ep \\
    & = \int_0^{2} \sqrt{p (\log (CR)+ \alpha^{-1} \log (2L)) + \frac{2p}{\alpha} \log (\sqrt{2}/\ep)} \dd \ep \\
    & \leq \int_0^{2} \zk{\sqrt{p (\log (CR)+ \alpha^{-1} \log (2L))} + \sqrt{\frac{2p}{\alpha} \log (\sqrt{2}/\ep)} } \dd \ep \\
    & \lesssim  \sqrt{p (\log (CR)+ \alpha^{-1} \log (2L))} + \sqrt{2p/\alpha} \\
    & \lesssim \alpha^{-1/2} (\sqrt{\log R + \log L + C}) \sqrt{p}.
  \end{align*}

\end{proof}

\subsection{Sequence model}

\begin{lemma}
  \label{lem:SeqModel_Grad_Diff}
  Consider the sequence model $z_j = f^*_j + \ep_j,~ j \in N$ induced by an orthogonal basis.
  For a function $f$ and its coefficients $f_j$ under the basis, define the population loss and the empirical loss as
  $\caL = \hf \sum_{j \in N} (f^*_j - f_j)^2 = \hf \norm{f^* - f}_{L^2}^2$ and $\hat{\caL} = \hf \sum_{j\in N} (z_j - f_j)^2$.
  Then, we have
  \begin{equation}
    \label{eq:SeqModel_Grad_Pop}
    -\nabla \hat{\caL}
    = -\nabla \caL + \sum_{j\in N} \ep_j \nabla f_j.
  \end{equation}
\end{lemma}
\begin{proof}
  It is direct from the following computation:
  \begin{equation*}
    -\nabla \hat{\caL}
    = \sum_{j\in N} (z_j - f_j) \nabla f_j
    = \sum_{j\in N} (f^*_j - f_j) \nabla f_j
    + \sum_{j\in N} \ep_j \nabla f_j
    = -\nabla\caL + \sum_{j} \ep_j \nabla f_j.
  \end{equation*}
\end{proof}

\subsection{Series}

\begin{proposition}
  \label{prop:Series_1}
  Fix $\alpha \in \R$.
  Let $\lambda_r \asymp r^{-\gamma}$ for $\gamma > \max(0,\alpha+1)$.
  Then, for any fixed $s > 0$, we have
  \begin{equation}
    \sum_{r \geq 0} r^{\alpha} (\log r)^{q} \min(1,\lambda_r t)
    \lesssim 1 + t^{\frac{1}{\gamma}(\alpha+1+s)},
  \end{equation}
  where the implicit constant depends on $\alpha, q, s,\gamma$.
\end{proposition}
\begin{proof}
  Let $L = \inf \dk{r \geq 0 : \lambda_r t \leq 1} \asymp t^{1/\gamma}$.
  Then, we have
  \begin{equation*}
    I = \sum_{r \geq 0} r^{\alpha} (\log r)^{q} \min(1,\lambda_r t)
    = \sum_{r \leq L} r^{\alpha} (\log r)^{q} + \sum_{r > L} r^\alpha (\log r)^{q} \lambda_r t
    = I_1 + I_2.
  \end{equation*}
  For $I_1$, we have
  \begin{equation*}
    I_1 = \sum_{r \leq L} r^\alpha (\log r)^{q}
    \lesssim
    \begin{cases}
      1, &           \alpha < -1, \\
      L^{\alpha+1+s} \lesssim t^{\frac{1}{\gamma}(\alpha+1+s)}, & \alpha \geq -1.
    \end{cases}
  \end{equation*}
  For $I_2$, we have
  \begin{equation*}
    I_2 = \sum_{r > L} r^\alpha (\log r)^{q} \lambda_r t
    \lesssim t \sum_{r > L} r^\alpha (\log r)^{q} r^{-\gamma}
    = t \sum_{r > L} r^{-\gamma + \alpha} (\log r)^{q}
    \lesssim t L^{-\gamma+\alpha+1+s} \lesssim t^{\frac{1}{\gamma}(\alpha+1+s)}.
  \end{equation*}
  The result follows from the above two inequalities.
%
\end{proof}

\begin{proposition}
  \label{prop:Series_2}
  Fix $p,q \in \R$.
  Let $\lambda_r \lesssim e^{-\gamma r}$ for $\gamma > 0$.
  Then, for any fixed $s > 0$, we have
  \begin{equation}
    \sum_{r \geq 0} r^{p} (\log r)^{q} \min(1,\lambda_r t)
    \lesssim 1 + (\log^+ t)^{p+1+s}
  \end{equation}
  where the implicit constant depends on $p, q, s,\gamma$.
\end{proposition}
\begin{proof}
  Without loss of generality, we assume that $p,q \geq 0$ and $t \geq 1$.
  Let us define $L = \inf \dk{r \geq 0 : \lambda_r t \leq 1} \lesssim \gamma^{-1} \log t$.
  Then, we have
  \begin{equation*}
    I = \sum_{r \geq 0} r^{p} (\log r)^{q} \min(1,\lambda_r t)
    = \sum_{r \leq L} r^{p} (\log r)^{q} + \sum_{r > L} r^p (\log r)^{q} \lambda_r t
    = I_1 + I_2.
  \end{equation*}
  For $I_1$, we have
  \begin{equation*}
    I_1 = \sum_{r \leq L} r^p (\log r)^{q}
    \lesssim
    \begin{cases}
      1, &           p < -1, \\
      L^{p+1+s}, & p \geq -1.
    \end{cases}
  \end{equation*}
  For $I_2$, we have
  \begin{equation*}
    I_2 = \sum_{r > L} r^p (\log r)^{q} \lambda_r t
    \lesssim t \sum_{r > L} r^p (\log r)^{q} e^{-\gamma r}
    \lesssim t L^{p+s} \exp(-\gamma L)
    \lesssim L^{p+s}.
  \end{equation*}
  Combining the above two inequalities with $L \lesssim \log t$, we prove the result.
\end{proof}

\subsection{Some Elementary Functions}

\begin{proposition}
  \label{prop:Aux__Softmin}
  Let $x_i \in [0,1]$, $i = 1,\dots,p$.
  Let $K > 0$.
  Define the function
  \begin{equation*}
    \omega(x_1,\dots,x_p) = -\frac{1}{K} \log \xk{\sum_{i=1}^p e^{-K x_i}}.
  \end{equation*}
  Then,
  \begin{equation*}
    \min(x_1,\dots,x_p) \geq \omega(x_1,\dots,x_p) \geq \min(x_1,\dots,x_p) - \frac{1}{K} \log p.
  \end{equation*}
  Moreover,
  \begin{equation*}
    \pdv{x_i} \omega = \frac{e^{-K x_i}}{\sum_{j=1}^p e^{-K x_j}}.
  \end{equation*}
\end{proposition}

\begin{proposition}
  \label{prop:Aux__Softmin_2}
  The following inequality holds for any $x_i \in [0,1]$, $i = 1,\dots,p$ and $K > 0$:
  \begin{equation*}
    \frac{1}{\sum_{j=1}^p e^{-K x_j}}\sum_{i=1}^p e^{-K x_i} (1-x_i)
    \geq 1 - \frac{1}{p} \sum_{i=1}^p x_i.
  \end{equation*}
\end{proposition}

\end{document}